\documentclass{article} 
\usepackage{iclr2026_conference,times}

\usepackage{amsmath,amsfonts,bm}

\def\eqref#1{equation~\ref{#1}}

\def\1{\bm{1}}

\DeclareMathAlphabet{\mathsfit}{\encodingdefault}{\sfdefault}{m}{sl}
\SetMathAlphabet{\mathsfit}{bold}{\encodingdefault}{\sfdefault}{bx}{n}

\DeclareMathOperator*{\argmin}{arg\,min}

\usepackage{booktabs}       
\usepackage{amsfonts}       
\usepackage{nicefrac}       
\usepackage{microtype}      
\usepackage{tikz}           
\usepackage{xcolor}         
\usepackage{hyperref}
\usepackage{url}
\usepackage{soul}
\usepackage{algorithm}
\usepackage{algpseudocode}
\usepackage{mathtools}
\usepackage{multirow}
\usepackage{float}
\usepackage{subcaption}
\usepackage{graphicx}
\usepackage{wrapfig}
\usepackage{amsmath}
\usepackage{amssymb}
\usepackage{bm}
\usepackage{bbm}
\usepackage{lineno}
\usepackage{enumitem}
\usepackage{comment}
\usepackage[color=pink!20,textsize=scriptsize]{todonotes}
\usepackage{macros}
\usepackage{makecell}

\usepackage{minitoc}
\usepackage[toc,page,header]{appendix}

\usepackage[export]{adjustbox}

\usepackage[nameinlink]{cleveref}

\usepackage{pifont}

\hypersetup{
	colorlinks=true,
	linkcolor=myblue,
	citecolor=myblue,
	urlcolor=mygray,
}

\definecolor{myblue}{RGB}{0,0, 150}
\definecolor{mygray}{RGB}{90,90,110}

\usepackage{amsthm}
\usepackage{mathtools}
\usepackage{thmtools, thm-restate}

\newtheorem{proposition}{Proposition}[section]

\crefname{theorem}{theorem}{theorems}
\Crefname{theorem}{Theorem}{Theorems}
\crefname{proposition}{proposition}{propositions}
\Crefname{proposition}{Proposition}{Propositions}
\newtheorem{corollary}{Corollary}[proposition]
\crefname{definition}{Definition}{Definitions}
\Crefname{definition}{Definition}{Definitions}
\crefname{assumption}{Assumption}{Assumptions}
\Crefname{assumption}{Assumption}{Assumptions}

\makeatletter
\renewcommand{\ALG@beginalgorithmic}{\small}
\makeatother

\newcounter{tfaenum}
\renewcommand{\thetfaenum}{\arabic{tfaenum}}
\newenvironment{tfaitem}[4][]
{%
  \refstepcounter{tfaenum}%
  \setlength\intextsep{2pt}%
  \begin{wrapfigure}{r}{0.45\textwidth}
    \centering
    \includegraphics[width=0.95\linewidth]{#2}
    \vspace{-8pt}
    \caption{#3}
    \label{#4}
  \end{wrapfigure}
  \noindent\textbf{\thetfaenum.\ #1}%
}
{\par\vspace{0.75\baselineskip}}

\title{Priors in Time: Missing Inductive Biases for Language Model Interpretability}

\author{Ekdeep Singh Lubana$^{1*}$, Can Rager$^{2*}$, Sai Sumedh R.\ Hindupur$^{3*}$,\\ 
\textbf{Valérie Costa$^{4}$,\, Greta Tuckute$^{5}$,\, Oam Patel$^{3}$,\, Sonia Krishna Murthy$^{5}$,\, Thomas Fel$^{5}$,}\\
\textbf{Daniel Wurgaft$^{1,6}$,\, Eric J.\ Bigelow$^{1,7}$,\, Johnny Lin$^8$,\, Demba Ba$^{3,5}$, }\\
\textbf{Martin Wattenberg$^{3}$, Fernanda Viegas$^{3}$, Melanie Weber$^{3}$, Aaron Mueller$^{9}$}\vspace{6pt}\\
$^1$Goodfire AI, $^2$Independent, $^3$SEAS, Harvard University, $^4$EPFL\\
$^5$Kempner Institute at Harvard University, $^6$Department of Psychology, Stanford University,\\
$^7$Department of Psychology, Harvard University, $^8$Decode Research, $^{9}$Boston University,\\ $^*$Co-first authors
}

\iclrfinalcopy 
\begin{document}

\doparttoc 
\faketableofcontents 

\maketitle

\vspace{-14pt}
\begin{abstract}

Recovering meaningful concepts from language model activations is a central aim of interpretability. While existing feature extraction methods aim to identify concepts that are independent directions, it is unclear if this assumption can capture the rich temporal structure of language. Specifically, via a Bayesian lens, we demonstrate that Sparse Autoencoders (SAEs) impose priors that assume independence of concepts across time, implying stationarity. Meanwhile, language model representations exhibit rich temporal dynamics, including systematic growth in conceptual dimensionality, context-dependent correlations, and pronounced non-stationarity, in direct conflict with the priors of SAEs. Taking inspiration from computational neuroscience, we introduce a new interpretability objective---Temporal Feature Analysis---which possesses a temporal inductive bias to decompose representations at a given time into two parts: a predictable component, which can be inferred from the context, and a residual component, which captures novel information unexplained by the context. Temporal Feature Analyzers correctly parse garden path sentences, identify event boundaries, and more broadly delineate abstract, slow-moving information from novel, fast-moving information, while existing SAEs show significant pitfalls in all the above tasks. Overall, our results underscore the need for inductive biases that match the data in designing robust interpretability tools. Code available at \url{https://github.com/eslubana/TemporalFeatureAnalysis}.

\vspace{-5pt}
\end{abstract}

\section{Introduction}
\vspace{-5pt}

Given the success of Language Models (LMs)~\citep{bubeck2023sparks, geminigold}, there is growing interest in understanding how such models incrementally update over sequences of tokens to exhibit complex behaviors~\citep{murthy2025inside, lindsey2025emergent, lindsey2025biology, lepori2025just, bigelow2025belief, tuckute2024driving, klindt2025superposition}. 
Interpretability research aims to make such analyses tractable, offering tools for hypothesis design, testing, and intervention based on evaluation of intermediate activations~\citep{geiger2025abstraction, sharkey2025open, bereska2024mechanistic}.
Often, such work builds on hypothesized computational models of how concepts are encoded in a neural network's representations, e.g., the linear representation hypothesis (LRH)~\citep{elhage2022superposition, arora2018linear}, correspondingly motivating tools such as sparse autoencoders (SAEs)~\citep{gao2024scaling, cunningham2023sparse} for unsupervised extraction of a dictionary of vectors that (ideally) mediate human-interpretable concepts~\citep{mueller2025questrightmediatorsurveying}.

A central challenge in ``bottom-up'' approaches to interpretability, like SAEs, is the mismatch between the assumptions of their underlying implementational account and the precise behavior or computation they intend to explain~\citep{jonas2017could, geiger2025abstraction, costa2025flat} (see Fig.~\ref{fig:intro}).
For instance, since LRH posits that different concepts correspond to directions in activation space that can be independently manipulated, it implicitly claims the data distribution can be factorized into independently varying latent variables~\citep{allen2024unpicking}.
This mismatch between the structure of the data distribution and strong priors codified in LRH can lead to misleading or pathological explanations when using SAEs to understand neural networks~\citep{chanin2025absorptionstudyingfeaturesplitting, bricken2023monosemanticity, hindupur2025projecting}. 
This raises a set of critical questions for using SAEs to interpret models trained on sequential data like language.
Specifically, since language exhibits rich temporal structure at \textit{multiple scales}~\citep{marsenwilson-1980, thompson1999temporal}---e.g., sentences contain dependencies that link words across time \citep{gibson2000dependency, mcelree2003memory}, upcoming words can be anticipated from context \citep{hale2001probabilistic, levy2008expectation}, and discourse imposes structure over longer timescales through phenomena like event boundaries \citep{zacks2007event, baldassano2017discovering}---one can ask \textit{what assumptions about temporal structures do SAEs make? How do these assumptions align with the actual temporal structure present in a LM's activations?}  

\begin{figure}
  \begin{center}
    \includegraphics[width=\linewidth]{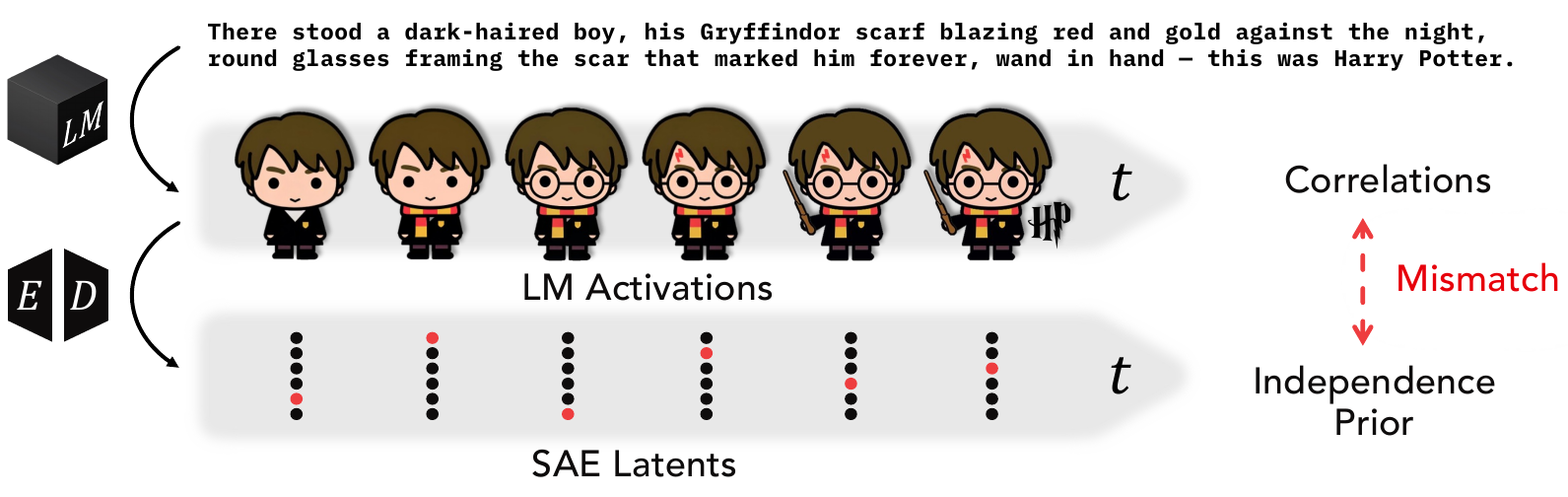}
  \end{center}
\vspace{-10pt}
\caption{\label{fig:intro}
\textbf{The mismatch between SAE assumptions and temporal structure of language}. An illustrative sentence describing attributes of Harry Potter is shown. When passed into a language model (LM), it leads to activations $\x_t$ that include concepts within them (possibly entangled): note the presence of large numbers of shared attributes over time, which manifest as correlations across time of activations. Sparse Autoencoders (SAEs) implicitly have an independence (i.i.d.) prior across time $t$ over their latents and thereby over concepts, which clashes with the true structure of language.}
  \label{fig:priordatamismatch-and-harrypotter}
  \vspace{-10pt}
\end{figure}

\paragraph{This work.}
Building on the Bayesian interpretation of sparse coding~\citep{olshausen1996emergence, olshausen1997sparse}---the framework that motivates SAEs---we rephrase the optimization objective of SAEs as a MAP (maximum a posteriori) estimation problem. 
This allows us to make explicit prior assumptions about temporal structure embedded in SAEs, showing they implicitly assume concepts are \textit{uncorrelated across time} and the number of concepts necessary to explain an activation is \textit{time-invariant}---that is, the information present at each token position is independent of information at other positions and uniformly distributed.
As we empirically show, these assumptions stand in stark contrast to the actual temporal structure present in language and language model representations, and can result in empirically observed pathologies in SAEs, such as feature splitting \citep{bricken2023monosemanticity, chanin2025absorptionstudyingfeaturesplitting, bussmann2025learning}.

These results then motivate us to draw a broader parallel between SAEs and computational neuroscience approaches for understanding neural data.
Specifically, population-level analyses of neural recordings have revealed that representations often lie on structured manifolds~\citep{khona2022attractor, nogueira2023geometry, sohn2019bayesian}, challenging the reductionist assumption in sparse coding that computations occur via independently firing, monosemantic features~\citep{eichenbaum2018barlow, saxena2019towards, barack2021two}. 
This motivated a paradigm shift towards more structured analysis protocols---methods designed around the generative process of the behavior one aims to explain \citep{schneider2023learnable, chen2018sparse}. 
Motivated by this and similar findings of intricate geometrical structure in neural network representations~\citep{fel2025into, gurnee2025when, modell2025origins}, we propose \textbf{Temporal Feature Analysis}, a new protocol for interpreting language model activations that incorporates explicit inductive biases about temporal structure. 
Our approach decomposes activations at each timestep into two orthogonal components: a \textit{predictable component}, obtained by projecting current representations onto past context using a learned attention mechanism, and a \textit{novel component}, representing residual information orthogonal to the predictable component. 
That is, we assume the \emph{novel} component---not the total representation---is uncorrelated over time. 
This allows correlations between total codes and hence enables our method to capture the temporal structure of LM activations.
Overall, we make the following contributions in this paper.

\begin{itemize}
    
    \item \textbf{Highlighting temporal structure in language model (LM) representations (Sec.~\ref{sec:llmtemporalstructure}).} We demonstrate rich temporal dynamics in LM representations, which show nonstationarity through systematic growth in conceptual dimensions with time, time-varying activation correlations, and strong correlations with recent context, that are at odds with the implicit priors of SAEs.
    
    \item \textbf{Precise characterization of prior assumptions of existing SAEs (Sec.~\ref{sec:saepriors}).} Adopting a Bayesian perspective on the SAE objective, we show that SAEs implicitly assume an independent and identically distributed (i.i.d.) prior across time on their latents, and thereby on concepts. As a consequence, SAEs fail to capture local correlations in LM activations.

    \item \textbf{Incorporating data-informed temporal inductive biases into a novel SAE architecture (Sec.~\ref{sec:temporalsaearch}).} Using empirically observed temporal structure in LM activations, we design a novel interpretability protocol---called \textit{Temporal Feature Analysis}---which decomposes the model activation into two components: a predictive component, which captures slow-moving contextual information, and a novel component, which captures fast-moving stimulus-driven information.
    
    \item \textbf{Using Temporal Feature Analysis to reveal temporal structures in LM representations (Sec.~\ref{sec:temporalresults}).} We show Temporal Feature Analysis's ability to extract temporal structures by studying linguistic input with temporal structure at different scales: garden path sentences, stories, and in-context dialogue. Specifically, we show the extracted predictive codes: (1) correctly parse garden path sentences, (2) decompose stories into events, (3) capture precise structure of in-context representations, and (4) delineate user-assistant interactions through smooth trajectories. Meanwhile, the novel codes allow automatic interpretability pipelines, akin to SAEs.
\end{itemize}


\section{Preliminaries}
\label{sec:prelims}

\paragraph{Notations.} Let bold, lowercase letters represent vectors (e.g., $\z$). Subscripts on vectors denote different samples (e.g., $\z_i$), while superscripts denote the index within the vector, leading to a scalar (e.g., $z^k$). We denote model activations by $\x \in \mathbb{R}^n$, SAE latents (sparse code) by $\z \in \mathbb{R}^M$, and the dictionary by $\D \in \mathbb{R}^{n \times M}$ ($M$ is the dictionary size). 

\paragraph{Sparse Coding.} Sparse dictionary learning \citep{olshausen1996emergence, olshausen1997sparse} expresses data as a sparse linear combination of dictionary elements, where both the weights and dictionary are learned from data. Intuitively, the dictionary behaves as a data-adaptive overcomplete basis; i.e., it typically has more elements that the dimension of ambient space. The optimization problem involved in this framework is $\argmin_{\D, \z} \frac{1}{N} \sum_{i=1}^{N} \|\x_{i} - \D \z_{i}\|_2^2 + \lambda \regsparse(\z_{i})$, where $\regsparse(\cdot)$, typically chosen to be the $\ell_1$-norm, is a sparsity-inducing regularizer. Sparsity assists in picking the fewest most relevant dictionary atoms to explain a given data point. 

\paragraph{Sparse Autoencoders (SAEs).} SAEs aim to disentangle~\citep{bengio2013representation, higgins2018towards, olahreps} neural network activations into human-interpretable concepts~\citep{cunningham2023sparse, bricken2023monosemanticity}. Specifically, SAEs transform their inputs (i.e., neural network activations) into a latent representation which is encouraged to be sparse. As shown by \citet{hindupur2025projecting}, this is achieved by solving the sparse coding problem using a specific parametric form for the sparse codes: 
\begin{equation}
\label{eq:sae-sparsecoding}
\begin{split}
    \argmin_{\D, \z} \frac{1}{T} \sum_{i=1}^{T} \|\x_{i} - \D \z_{i}\|_2^2 + \lambda \regsparse(\z_{i}),\quad
    \text{s.t. }\quad \z_{k} = f_{\mathtt{SAE}}(\x_{k})\; \forall k, \; \Tilde{g}(\z_1, \dots, \z_T)=0,
\end{split}
\end{equation}
where $\regsparse(\cdot)$ is the regularizer (typically the $L_1$ norm), $f_{\mathtt{SAE}}$ is the SAE encoder architecture, and $\Tilde{g}(\cdot)$ captures SAE-specific sparsity constraints on $\z$. 
$f_{\mathtt{SAE}}$ is typically a single hidden layer as in the ReLU SAE \citep{bricken2023monosemanticity, cunningham2023sparse}, TopK SAE \citep{gao2024scaling, makhzani2013k}, JumpReLU SAE \citep{rajamanoharan2024jumpingaheadimprovingreconstruction} and BatchTopK SAE \citep{bussmann2024batchtopksparseautoencoders}, 
though recent work has also explored alternative architectures inspired by sparse coding algorithms to capture specific structures, e.g., hierarchies \citep{muchane2025incorporatinghierarchicalsemanticssparse, costa2025flat}.

\section{Temporal Structure in Language Model Activations}
\label{sec:llmtemporalstructure}

\begin{figure}
    \centering
    \vspace{-5pt}
    \includegraphics[width=1\linewidth]{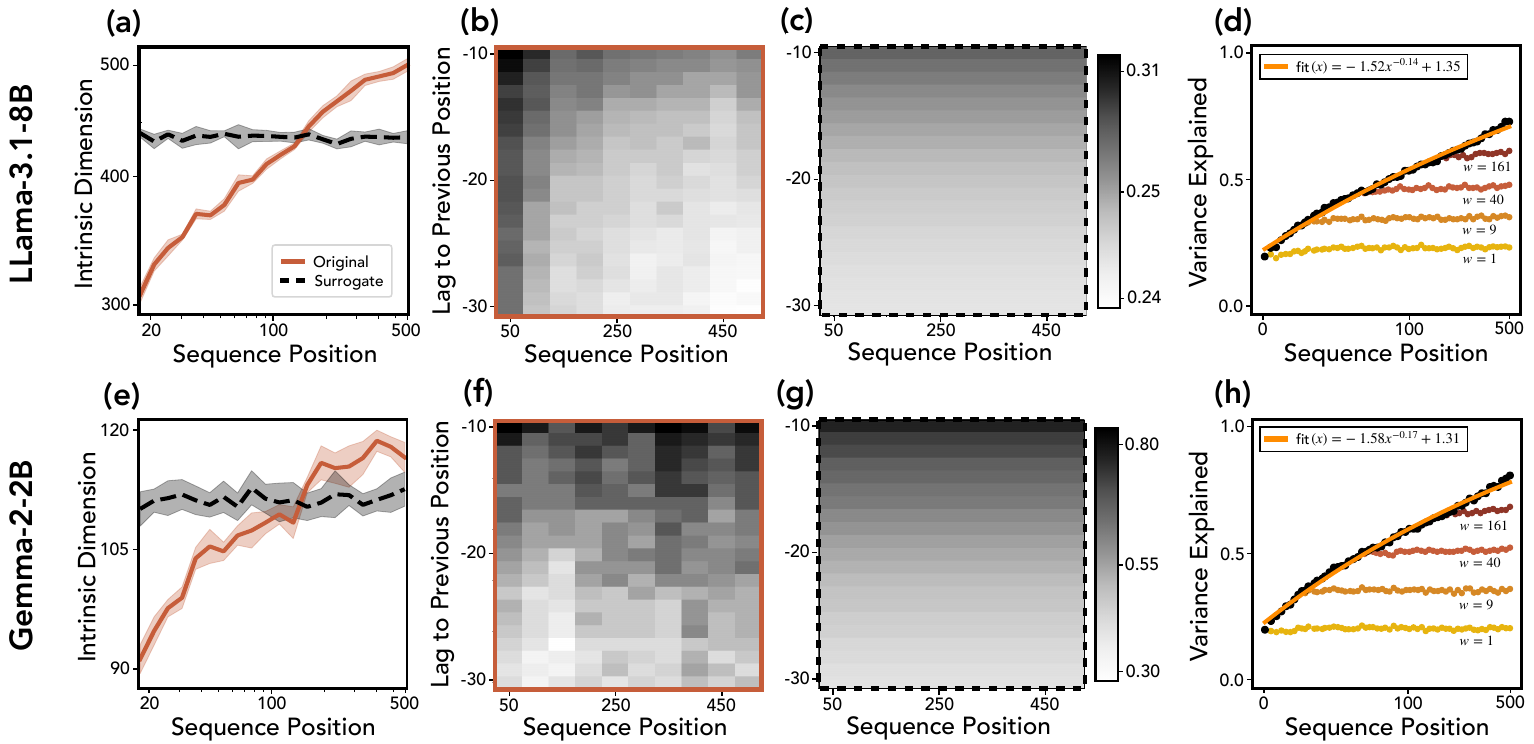}
    \vspace{-15pt}
        \caption{\textbf{Temporal structure of LLM activations reveals nonstationarity.} We use Pile samples~\citep{monology2021pile-uncopyrighted} to analyze temporal structure from activations of two pretrained LMs, comparing it to a surrogate signal that is stationary in nature (see App.~\ref{app:surrogate}). (\textbf{a, e}) Intrinsic dimension of model activations and stationary surrogate. (\textbf{b, f}) Autocorrelations $A(\x_t, \x_{t-\tau})$ as a function of sequence position ($t$) and lag ($\tau$). (\textbf{c, g}) Autocorrelation of the stationary surrogate. (\textbf{d, h}) Variance explained by projecting current representation $\x_t$ onto past context window $\{\x_{t-1}, \dots, \x_{t-w}\}$ with different sizes $w$, along with a baseline. Results consistently show representations getting `denser' over time and being significantly more structured than a stationary surrogate.
    \vspace{-10pt}
    }
    \label{fig:llm_temporal_structure}
\end{figure}

To contextualize the prior assumptions made by SAEs about temporal structure in LMs' activations, we first perform an empirical characterization of such temporal structure in pretrained LMs.
Specifically, since LMs are trained to generate coherent text by learning the distribution of natural language, one can expect their representations capture the rich phenomenology of its sequential structure~\citep{elman1990finding}; indeed, recent work has in fact found LM representations to be predictive of human neural recordings during language comprehension~\citep{hosseini2024universality, schrimpf2021neural, tuckute2024driving, hong2024scale, georgiou2023using}.
Motivated by this, we perform two experiments relevant to our discussion (see Fig.~\ref{fig:llm_temporal_structure}): (i) \textit{measuring intrinsic dimensionality}---an approximation of the number of concepts necessary to explain the data, which can be expected to increase in a monotonic manner with time~\citep{zhong2024random, can2025statistical, barak2014working,meister2021revisiting}---and (ii) \textit{signal nonstationarity}---which assesses whether model activations reflect the contextual relations between phrases of a passage ~\citep{zacks2007event}.

\paragraph{Increasing intrinsic dimensionality.} Fig.~\ref{fig:llm_temporal_structure}~(a,e) show the dimensionality of the underlying manifold structure (intrinsic dimension) in model activations. We estimate the intrinsic dimensionality at a fixed position across a set of sequences with the U-statistic (App.\ Sec.\ \ref{appendixsubsection:ustat}). For language model activations, this metric increases steadily with sequence position. On the other hand, a stationary surrogate of the data (see App.\ Sec.\ \ref{app:surrogate}) shows nearly constant intrinsic dimension over time. This indicates that model activations get `denser', i.e., they possess more information over time. Correspondingly, the number of concepts needed to explain them varies with context.

\paragraph{Non-stationarity: Context explains bulk of signal variance.} Subplots (b), (f) show the autocorrelation of model activations (App.\ Sec.\ \ref{app:autocorr}), which is noticeably different at different sequence positions (x-axis), while the stationary surrogate, as expected, shows nearly position-invariant autocorrelation values (subplots (c), (g)). This finding is a clear signature of time-dependent correlation structure, and therefore of non-stationarity. We quantify the similarity of a representation with its context in subplots (d), (h). Specifically, we project representations of token $\x_t$ at a given time $t$ onto the subspace spanned by preceeding representations in the context $\{\x_{<t}\}$. These subplots show that up to $80\%$ variance of $\x_t$ is explained by a context of 500 tokens, further highlighting strong cross-temporal correlations. Significant variance in the representation at time $t$, $\x_t$, can be predicted (expressed) using representations from the past context.

\section{Temporal Priors of Sparse Autoencoders}
\label{sec:saepriors}
We now state the prior assumptions made by existing SAEs regards sequential structure in an input, contrasting these assumptions with the empirical results shown in Sec.~\ref{sec:llmtemporalstructure}.
Specifically, building on the arguments used by \citet{olshausen1997sparse} to formalize the problem of sparse coding, we note that the SAE training objective (Eq.~\ref{eq:sae-sparsecoding}) can be interpreted from a Bayesian lens: minimize the negative log posterior $\argmin_{\{\z_t\}} -\log P(\z_1, \dots, \z_T \mid \x_1, \dots, \x_T)$ of the data, which, by Bayes' rule, can be written as the sum of log likelihood (MSE) and log prior (the regularizer $\regsparse$).
From this lens, SAEs' prior assumptions on sequential structure in LM activations can be described as follows. 

\begin{proposition}[Independence prior over time] Consider the SAE maximum aposteriori (MAP) objective from Eq.~\ref{eq:sae-sparsecoding}. Since the sparsity constraints are additive over time, this objective has an independent and identically distributed (i.i.d.) prior over time:\vspace{-2pt}
\begin{equation}    
    P(\z_1, \dots, \z_T) \propto \prod_{t=1}^T \exp \left( - \lambda \regsparse(\z_i) -\Tilde{\lambda}\Tilde{g}(\z_i) \right) =  \prod_i P(\z_i).
\end{equation}
\label{thm:priors-over-time}
\end{proposition}
\vspace{-3pt}

A more precise version of the claim for specific SAE architectures is provided in Appendix~\ref{appsec:saepriors-derived}.
Intuitively, the claim above says that SAEs assume an independence of latents, and hence the concepts underlying the generative process of language, over time. 
This directly conflicts with the rich contextual structure of LM activations we empirically observed in Fig.~\ref{fig:llm_temporal_structure}b--d, f--h.
Crucially, this also implies that SAEs assume the sparsity of latent codes necessary to explain model activations to be time-invariant, as stated formally in the corollary below. 

\begin{corollary}[Assumptions of time-invariant sparsity] As a consequence of the i.i.d. priors over time from Prop.~\ref{thm:priors-over-time}, standard SAEs assume that sparsity of representations emerges from a fixed distribution independently over time (i.i.d.), i.e., $P(\|\z_1\|_0, \dots, \|\z_T\|_0) = \prod_t P(\|\z_t\|_0).$
\label{corr:time-invariant-sparsity}
\end{corollary}

SAEs thus assume that sparsity---which, in their underlying generative model of activations corresponds to the number of concepts necessary for explaining the data~\citep{bricken2023monosemanticity, elhage2022superposition}---remains approximately constant over time. 
This again does not align with the increasing dimensionality of representations observed in LM activations (see Fig.~\ref{fig:llm_temporal_structure}a,e). 
Correspondingly, if enough concepts aggregate over context such that a model's activations become `denser' than the assumed sparsity budget, the assumption of time-invariant sparsity implies SAEs can fail to capture the temporal structure inherent in language, as formally stated below.

\begin{proposition}[Restrictive Sparsity Budget Leads to Support Switching in SAEs] Suppose data $\x$ lies on a $C^1$ manifold $\mathcal{M}\subset \mathbb{R}^d$. For $\x \in \mathcal{M}$, let $T_{\x}\mathcal{M}$ denote the tangent space at $\x$, and $m(\x)=\dim T_{\x}\mathcal{M}$ its dimension. Suppose an SAE $S$ with latent code $S(x)$ has a sparsity budget $|S(\x)|=K$. If $K \le m(\x)$, under the assumption of low error (Eq.~\ref{eq:sae-sparsecoding}), in some neighborhood $\mathcal{N}_{\x}$ of $\x$, $\exists\; \x_1, \x_2 \in \mathcal{N}_{\x}$ s.t. $S(\x_1)\neq S(\x_2)$, i.e., support switching occurs in the SAE latents. 
\label{prop:supportswitching}
\end{proposition}

\setlength\intextsep{1pt}
\begin{wrapfigure}{}{0.46\textwidth}
    \vspace{-2pt}
    \centering
    \includegraphics[width=\linewidth]{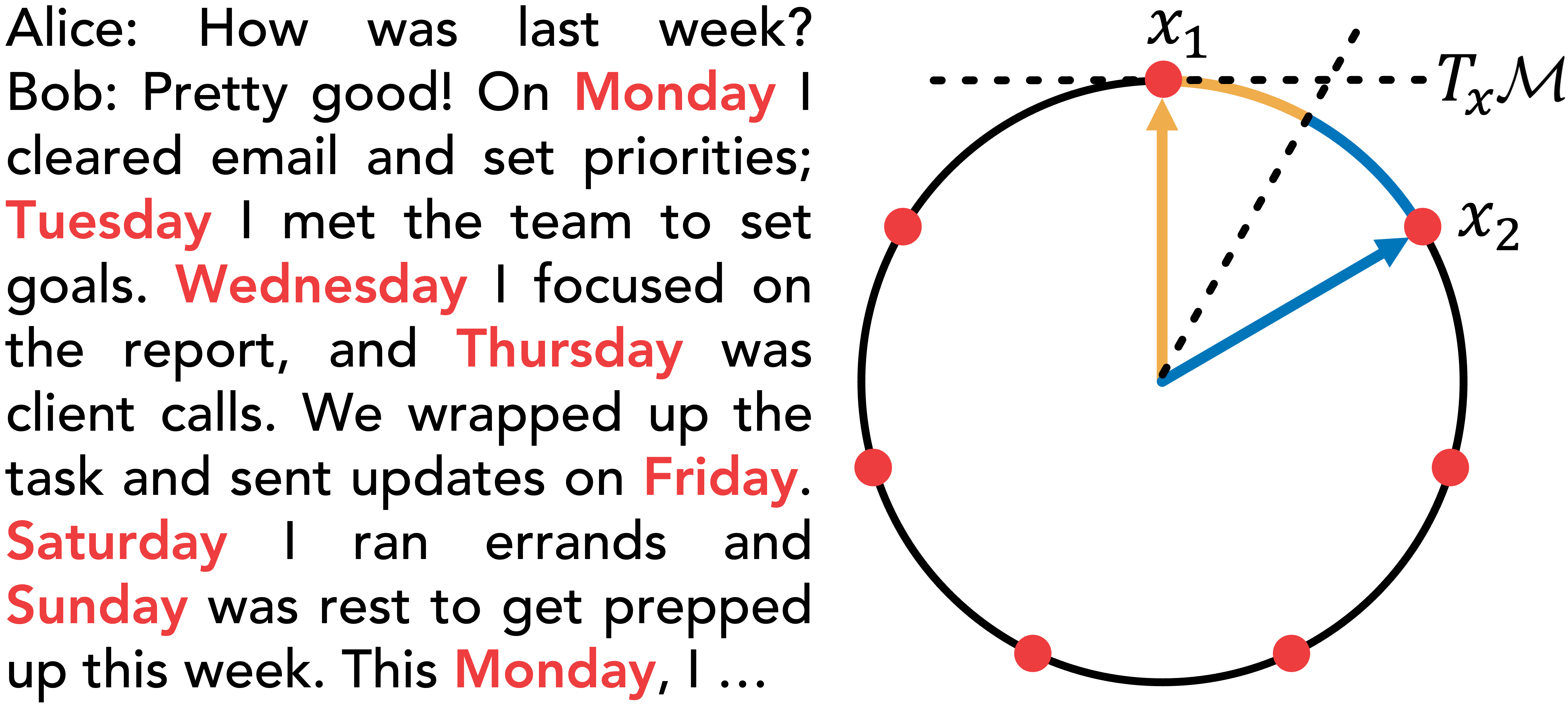}
    \vspace{-18pt}
    \caption{\textbf{Sparsity splits local structure.} (left) Temporally correlated inputs can yield geometrically structured activations. (right) If the sparsity budget is lower than the intrinsic dimensionality of the activation geometry, an SAE is incentivized to partition the manifold into local regions such that even nearby points map to disjoint codes and local structure is lost.
    }
    \label{fig:theory}
\end{wrapfigure}

Intuitively, the claim above says that even if an SAE achieves low reconstruction error, its specific geometric assumptions---use of independently firing directions and restricted intrinsic dimensionality of data---will lead to a ``splitting'' of the input data into several regions such that different sets of latents will activate for even locally nearby points (see Fig.~\ref{fig:theory}).
Consequently, if the sparsity budget is too little or if model activations are more dense than the assumed budget, then local geometry present in the activations will be lost in the SAE latent codes---as is arguably already observed empirically with phenomena like feature splitting~\citep{chanin2025absorptionstudyingfeaturesplitting, bussmann2025learning, bricken2023monosemanticity}.
For the context of our paper, as per our results in Fig.~\ref{fig:llm_temporal_structure}~(a,e), temporal correlations can increase activations' intrinsic dimensionality and, when viewed from the lens of the claim above, we arrive at the conclusion that SAEs can fail to capture temporal structure present in activations.
We empirically validate this claim in Sec.~\ref{sec:temporalresults}. 
More broadly, we note there is growing evidence that neural network representations possess rich geometric structures across a broad set of modalities~\citep{gurnee2025when, fel2025into, park2025iclrincontextlearningrepresentations, engels2024not, costa2025flat, kantamneni2025language, pearce2025tree}. 
As per Prop.~\ref{prop:supportswitching}, the use of SAEs for interpreting behaviors that invoke such structures can be misleading and should be approached carefully (cf.~\citep{sengupta2018manifold}).

\section{Temporal Feature Analysis: Towards Predictive Approaches for Interpreting Neural Representations}
\label{sec:temporalsaearch}

\begin{figure}
    \centering
    \includegraphics[width=\linewidth]{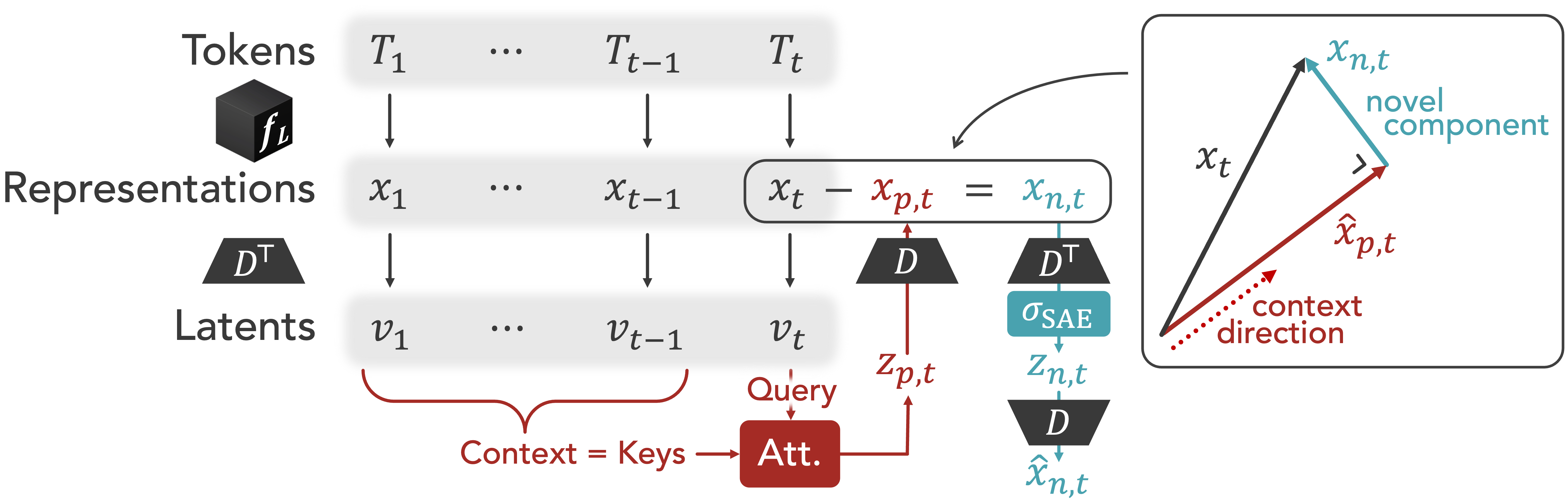}
    \vspace{-15pt}
    \caption{
\textbf{Schematic of Temporal Feature Analysis.} Temporal Feature Analyzers decompose activations $\x_t$ into two components: a predictable component, obtained by projecting $\x_t$ onto a context direction (derived from the past $\x_{<t}$ using attention), and a sparse, novel component orthogonal to the predictable component that captures new information seen at time $t$. 
}
\label{fig:architecture}
\end{figure}

As stated in Sec.~\ref{sec:prelims}, sparse coding, a framework designed in computational neuroscience to understand neural representations in biological brains~\citep{olshausen1996emergence, olshausen1997sparse}, inspired SAEs as a framework for interpreting artificial neural networks~\citep{bricken2023monosemanticity, cunningham2023sparse}.
In fact, the parallels between these communities can be made deeper: motivated by observations of intricate geometry of neural representations derived out of multi-dimensional population analyses~\citep{khona2022attractor, nogueira2023geometry, sohn2019bayesian}, there were calls in computational neuroscience to discard the limiting reduction assumed in sparse coding that computations occur via a set of independently firing, monosemantic features~\citep{eichenbaum2018barlow, saxena2019towards, barack2021two, seung1996brain, chung2021neural}---similar to our arguments in Sec.~\ref{sec:llmtemporalstructure},~\ref{sec:saepriors} (and results that follow in Sec.~\ref{sec:temporalresults}).
Correspondingly, a need for more structured protocols was suggested~\citep{eichenbaum2018barlow, barack2021two, chen2019sparse, sengupta2018manifold}, leading to methods that were motivated by the generative process of the behavior one is trying to explain~\citep{schneider2023learnable, chen2018sparse, berkes2005slow, wiskott2002slow, linderman2017bayesian, yu2008gaussian}.
We argue a similar paradigm shift is needed in language model interpretability: given that we train models to learn the distribution of highly structured data, we ought to embrace the fact that neural network activations can exhibit intricate geometrical organization.
These geometrical objects may in fact be \textit{the} units of computation that are necessary for describing model computation~\citep{perich2025neural, eichenbaum2018barlow, modell2025origins}, and hence we must design interpretability approaches to identify them.
In what follows, as an attempt to qualify our arguments, we propose \textit{one} such approach, titled \textit{Temporal Feature Analysis}, that focuses on the temporal structure of LM activations.
We also point the reader to relevant related work from identifiability literature~\citep{joshi2025identifiable, klindt2020towards}.

\paragraph{Temporal Feature Analysis.} In computational neuroscience, when analyzing data from dynamical domains (e.g., audio, language, or video), a commonly made assumption is that there is contextual information present in the recent history that informs the next state---this part of the signal is deemed \textit{predictable}~\citep{chen2019sparse, millidge2024predictive}, \textit{slow-changing}~\citep{berkes2005slow, klindt2020towards}, \textit{invariant}~\citep{olshausen2007learning, hyvarinen2003bubbles}, or \textit{dense}~\citep{tasissa2022improvingdiscriminativereconstructionsimultaneous}.
Meanwhile, the remaining signal corresponds to new bits of information added by the observed state at the next timestep---this part can be deemed \textit{novel}, \textit{fast-changing}, \textit{variant}, or \textit{sparse} with respect to the context.
We argue LM activations are amenable to a similar generative model.
Specifically, our observations in Sec.~\ref{sec:llmtemporalstructure} show that activations $\x_t$ at time $t$ are correlated with the context and can be decomposed into two such parts as well.
Motivated by this, we propose the following model of LM activations:
\begin{align}
\label{eq:tfa_model}
    \x_t = \x_{p, t} + \x_{n, t}, \quad \text{s.t. } \quad \x_{p,t}= \D\z_{p, t}\,\, \text{and}\,\, \x_{n,t}= \D\z_{n, t},
\end{align}
where $\x_{p,t}$ denotes a \textit{predictable} component of the signal that captures the correlations of $\x_t$ with past data $\left\{\x_{<t}\right\}$, while $\x_{n,t}$ denotes a \textit{novel} component that represents new, i.e., uncorrelated with the past, information added by the current token $\x_t$. 
To obtain $\z_{p, t}$, we project $\x_t$ onto $\{\x_{<t}\}$ to explain the predictable variance in $\x_t$ as a convex combination of past data. 
Specifically, we apply an encoder (ReLU plus a linear map) to take the inputs to a latent space $\v$ and perform an attention operation $f$ in that space, yielding $\z_{p,t} = f( \ \{ \phi(D^{T}\x_1), \dots, \phi(D^{T}\x_{t-1})\}, \phi(D^{T}\x_{t}))$; here $\x_{t}$ defines the query and the remaining context serves as keys.
Meanwhile, $\z_{n,t}$ is defined using a standard SAE: $\z_{n,t} = \Tilde{f}_{\mathtt{SAE}}(\x_t, \z_{p,t}) = \sigma(\D^T(\x_t - \D \z_{p, t}))$; we use a standard SAE nonlinearity (either TopK or BatchTopK) to instantiate $\sigma$, applying it to $\x_t - \D\z_{p,t}$ to derive $\z_{n,t}$. 
See Fig.~\ref{fig:architecture} for an overall schematic of the encoding process.
The learning objective in Temporal Feature Analysis follows.
\vspace{-5pt}
\begin{equation}
\label{eq:tfa_loss}
\begin{split}
    \argmin_{\D, \z} \frac{1}{T} &\sum_{i=1}^{T} \|\x_{i} - \D(\z_{p,i} + \z_{n,i})\|_2^2 + \lambda \regsparse(\z_{n,i}),\\
    \text{s.t. } \,\,\, \z_{p,k} &= f_{\mathtt{SAE}}(\{\x_{<  k} \} \x_k)\,\,\, \text{and}\,\, \z_{n,k} = \Tilde{f}_{\mathtt{SAE}}(\x_k, \z_{p,k})\,\, \forall k. 
\end{split}
\end{equation}

Relating to Sec.~\ref{sec:saepriors}, we note the prior assumption in Temporal Feature Analysis is that the residual $\z_{n,t} = \z_t - \z_{p,t}$, which captures the novel information in $\x_t$ remaining after removing the projections onto the past context, is \textit{i.i.d.}\ over time. 
This prior allows correlations between the total codes, and thereby between concepts, across time, instead of assuming that all concepts are time-independent. 
However, we do note this prior is still an assumption: one can reasonably argue the residuals need not be \textit{i.i.d.}. 
We nevertheless state this point to be explicit about assumptions involved in our work.

\setlength\intextsep{1pt}
\begin{wrapfigure}{}{0.5\textwidth}
  \centering
  \begin{minipage}[t]{\linewidth}
    \centering
    \footnotesize 
    \setlength{\tabcolsep}{3pt}
    \renewcommand{\arraystretch}{1.05}
    \captionof{table}{\label{tab:comparing_saes_nmse}Temporal Feature Analysis and SAEs achieves similar NMSE across domains (Simple Stories, Webtext, Code).
    \vspace{-5pt}
    }
\begin{tabular}{@{}l|ccc|c|c@{}}
  \toprule
  & {ReLU} & {TopK} & {BTopK} & {Pred. Only} & {Temporal} \\
  \midrule
  Story & 0.20 & 0.155 & 0.152 & 0.34 & 0.139 \\
  Web & 0.19 & 0.144 & 0.139 & 0.36 & 0.139 \\
  Code & 0.20 & 0.154 & 0.149 & 0.38 & 0.152 \\
  \bottomrule
\end{tabular}  
  \end{minipage}
\vfill
\vspace{5pt}
  \begin{minipage}[t]{\linewidth}
    \centering
    \footnotesize 
    \setlength{\tabcolsep}{3pt}
    \renewcommand{\arraystretch}{1.05}
    \captionof{table}{\label{tab:comparing_saes_var}Temporal Feature Analysis explains similar amount of signal variance as SAEs.
    \vspace{-5pt}
    }
\begin{tabular}{@{}l|ccc|c|c@{}}
  \toprule
  & {ReLU} & {TopK} & {BTopK} & {Pred. Only} & {Temporal} \\
  \midrule
  Story & 0.60 & 0.71 & 0.72 & 0.29 & 0.73 \\
  Web & 0.69 & 0.78 & 0.79 & 0.40 & 0.79 \\
  Code & 0.65 & 0.75 & 0.75 & 0.33 & 0.75 \\
  \bottomrule
\end{tabular}  
  \end{minipage}
  \vfill
  \vspace{5pt}
  \begin{minipage}[t]{\linewidth}
    \centering
    \footnotesize 
    \setlength{\tabcolsep}{2pt}
    \renewcommand{\arraystretch}{1.05}
    \captionof{table}{\label{tab:comparing_pred_novel}Predictive and novel codes explain different parts of the input. See main text for details.
    \vspace{-5pt}
    }
\begin{tabular}{@{}l|c|cc|cc|cc@{}}
  \toprule
  & $\langle \x_p, \x_n \rangle$ & \multicolumn{2}{c|}{{\% Norm}} & \multicolumn{2}{c|}{{NMSE}} & \multicolumn{2}{c}{{Var. Expl.}}\\
  \cmidrule(lr){3-4} \cmidrule(lr){5-6} \cmidrule(lr){7-8}
  & & {Pred.} & {Novel} & {Pred.} & {Novel} & {Pred.} & {Novel} \\
  \midrule
  Story & -0.02 & 76.2 & 23.5 & 0.53 & 4.03 & 0.11 & 0.64 \\
  Web & -0.02 & 80.5 & 19.5 & 0.49 & 4.28 & 0.17 & 0.66 \\
  Code & -0.02 & 74.2 & 26.0 & 0.57 & 3.84 & 0.14 & 0.65 \\
  \bottomrule
\end{tabular}
\end{minipage}
\end{wrapfigure}

\paragraph{Sanity Checking Temporal Feature Analysis.}
Before analyzing how different approaches represent the temporal structure of language, we demonstrate that Temporal Feature Analysis performs on par with SAEs on standard metrics such as reconstruction error.
Specifically, we train a Temporal Feature Analyzer and standard SAEs (ReLU, TopK, BatchTopK) on $1$B token activations extracted from Gemma-2-2B~\citep{team2024gemma} from the Pile-Uncopyrighted dataset~\citep{monology2021pile-uncopyrighted}.
We also analyze a baseline of the prediction only module from Temporal Feature Analysis, reported as `Pred. only', which can be expected to underperform since predicting the next-token representation is likely to be more difficult than reconstructing it.
Results are provided in Tab.~\ref{tab:comparing_saes_nmse},~\ref{tab:comparing_saes_var} and show competitive performance between all protocols, except Pred. only.
One can also assess which part of a fully trained Temporal Feature Analyzer is more salient in defining its performance, i.e., does the estimated predictive part $\hat{\x}_{p,t}$ contribute more to the reconstruction $\hat{\x}$ or does the estimated novel part $\hat{\x}_{n,t}$.
Results are reported in Tab.~\ref{tab:comparing_pred_novel}.
We see that the error vectors, i.e., $\x - \hat{\x}_p$ and $\x-\hat{\x}_n$, are approximately orthogonal, suggesting the modules computing predictive and novel codes capture separate bits of information from the input signal.
Furthermore, we find that a bulk of the reconstructed signal $\hat{\x}_t$ (in the sense of norm) is captured by the predictive code---in fact, the percentage contribution of the predictive code is $\sim$80\%, in line with numbers observed in Fig.~\ref{fig:llm_temporal_structure}d.
However, analyzing the reconstruction performance, we see the predictive component primarily contributes to achieving a good NMSE, while the novel component is more responsible for explaining the input signal variance.
These results align with the generative model assumed in Temporal Feature Analysis (Eq.~\ref{eq:tfa_model}).
Specifically, NMSE captures the average reconstruction, and hence a slower moving, contextual signal can expect to dominate its calculation; in fact, we see the attention layer used for defining the prediction module produces block structured attention maps related to sub-event chunks (see Fig.~\ref{fig:attn_pattern_pred}), suggesting more granular temporal dynamics are being captured by predictive part (we elaborate on this last point in the next section). 
Meanwhile, variance assesses changes per dimension and timestamp in the signal, which better matches the inductive bias imposed on the novel part.

\section{Capturing dynamic structure with Temporal Feature Analysis}
\label{sec:temporalresults}

\begin{figure}
\vspace{-15pt}
    \centering
    \includegraphics[width=\linewidth]{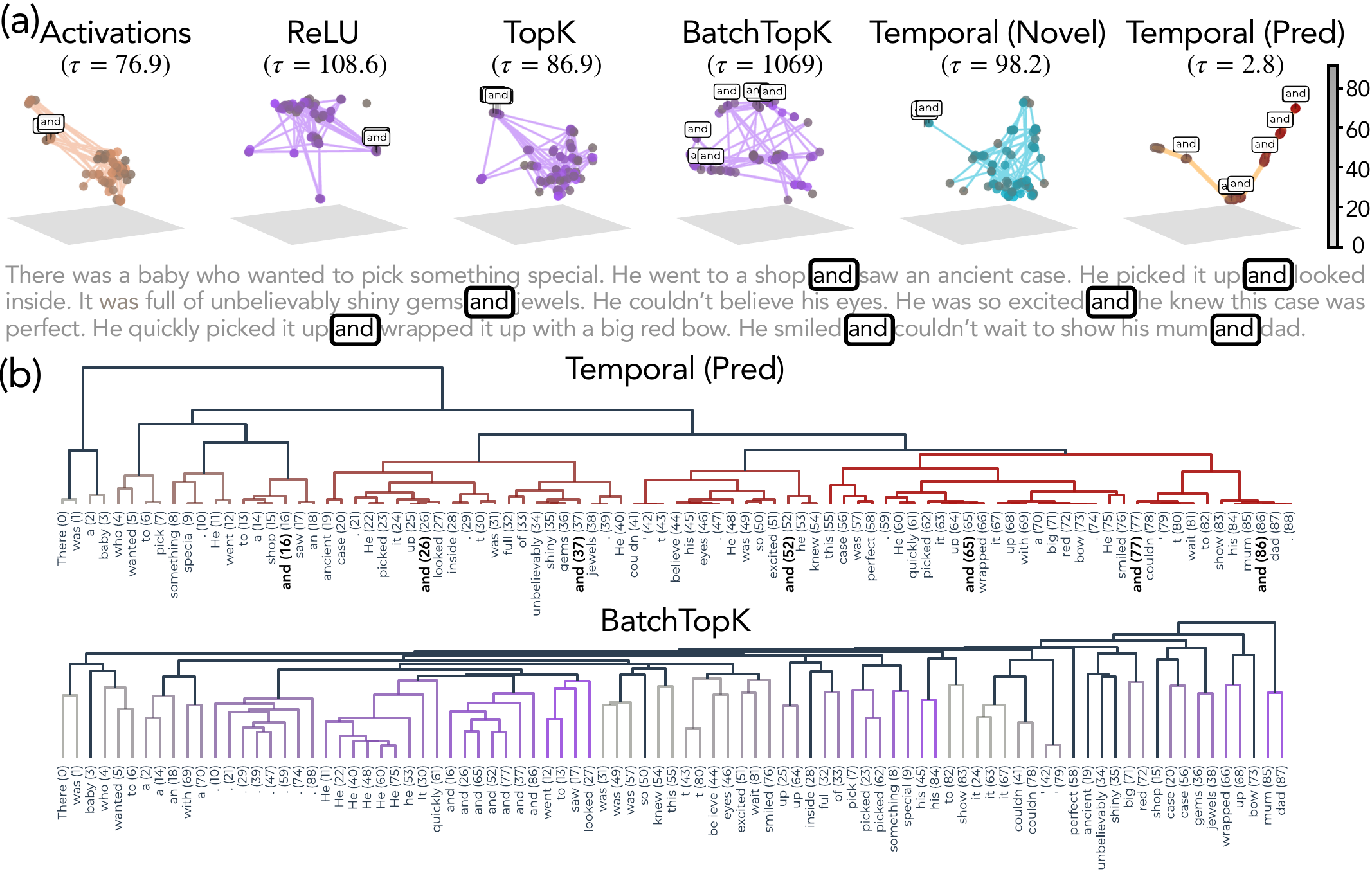}
    \vspace{-10pt}
    \caption{
\textbf{Temporal Feature Analyzers unroll stories, decomposing into events.} We consider model activations from a story and compute pairwise similarity of codes extracted from different interpretability protocols. 
(a) We see predictive codes from Temporal Feature Analysis organizes in hierarchical block structures that seem to align with (sub)event boundaries in the analyzed story, while the novel code primarily emphasizes sudden changes in the narrative; meanwhile, standard SAEs show a mixture of the two structures, with a stronger similarity to the structure exhibited by the novel codes.
(b) We confirm the alignment of predictive codes with event boundaries by running an off-the-shelf hierarchical clustering algorithm, finding the token clusters indeed correspond to (sub)events occurring in the story as the narrative proceeds. Running this process on SAEs, we find this process yields temporally incoherent clusters that are primarily defined by lexical information.
\vspace{-10pt}
}
\label{fig:ind_story_geometry}
\end{figure}

We now assess the ability of different SAEs towards capturing temporal structure in language model representations, offering empirical evidence for our arguments in Sec.~\ref{sec:saepriors}.
We contextualize these results with respect to our proposed protocol of Temporal Feature Analysis, hence assessing what new information about model activations a temporally-informed interpretability protocol buys us.

\paragraph{Experimental Setup.} We analyze standard SAEs (ReLU, TopK, and BatchTopK) and Temporal Feature Analyzers trained on 1B model activations from Layer 15 of a Gemma-2-2B model extracted from the Pile Uncopyrighted dataset~\citep{monology2021pile-uncopyrighted}. 
For evaluation domains, we again seek inspiration from cognitive and computational neuroscience literature on assessing humans' abilities to process temporal structure in language.
These settings capture a spectrum of interaction between a system's prior knowledge versus temporally offered information about a concept, i.e., how the concept gets used within the given context.
Specifically, we first use \textbf{stories} as a testbed for evaluation.
Motivated by prior work on understanding how narrative structures are parsed by humans~\citep{rumelhart1975notes, thorndyke1977cognitive, baldassano2017discovering, nastase2021narratives}, we first analyze whether local event structures from stories are reflected in latent codes derived out of SAEs and Temporal Feature Analysis.
This help us evaluate how different protocols represent local vs.\ global (over time) semantic information. 
To investigate a more syntactic variant of this experiment, we analyze latent codes extracted from \textbf{garden path sentences}, i.e., grammatically valid but ambiguous phrases in which the constituents, when parsed according to their typical syntactic role, lead to an incorrect overall sentence parse (e.g., ``The old man the road'').
Thus, this setting captures a domain wherein the context requires interpreting constituent words in a manner that goes against the most likely interpretation from a model's pretraining prior, i.e., the temporal information dominates the interpretation of a constituent.
While humans can resolve these sentences by backtracking~\citep{hale2001probabilistic, hanna2024incremental, levy2008expectation}, LMs operate in an autoregressive manner and hence must capture the correct sentence parse in the forward direction. 
Given that LMs possess the ability to do so~\citep{li2024incremental, hanna2024incremental}, we compare whether latent codes extracted using SAEs and Temporal Feature Analysis relate sentence constituents according to the valid parse.
Finally, motivated by literature on cognitive maps~\citep{behrens2018cognitive}, wherein a subject has to infer a latent in-context structure that relates two observations (tokens)---a capability LMs have been show to possess~\citep{park2025iclrincontextlearningrepresentations}---we assess whether different interpretability protocols' latent codes capture these \textbf{in-context representations}. 
In this last setting, a model cannot rely on its prior knowledge at all, since the behavior is entirely guided by the in-context, temporal structure.

\subsection{The Geometry of Stories: A Narrative-Driven Domain} 
\label{sec:stories}

\paragraph{UMAP of Latent Codes Suggests Models Temporally Straighten Activations.} We consider the TinyStories datasets~\citep{eldan2023tinystories} for its relatively straightforward narrative structures, and qualitatively analyze the geometry of latent codes extracted from model activations when processing these stories.
Visualizing the latent codes in a low-dimensional basis via a 3D UMAP projection~\citep{mcinnes2018umap}, we see SAEs yield a highly irregular and unstructured geometry (see Fig.~\ref{fig:ind_story_geometry}a).
Calculating Tortuosity~\citep{bullitt2003measuring}, a measure of how aligned local arcs are with respect to the global structure of a curve, we see very high values emerge for SAEs' latent codes geometry, suggesting sudden changes in the local similarity as a story unravels. To further understand the results above, we highlight a specific token (`\texttt{and}') from the story, the UMAP analysis shows that standard SAEs generally just cluster tokens by lexical identity.
This is further corroborated by running a hierarchical clustering algorithm on the latent codes~\citep{dendrogram}, finding temporally incoherent, but lexically related clusters (see Fig.~\ref{fig:ind_story_geometry}b).

\setlength\intextsep{2pt}
\begin{wrapfigure}{}{0.47\textwidth}
  \centering
  \begin{minipage}[t]{\linewidth}
    \centering
    \footnotesize 
    \setlength{\tabcolsep}{3pt}
    \renewcommand{\arraystretch}{1.1}
    \captionof{table}{\textbf{Kernel similarity (CKA) between latent codes and model activations.} Predictive codes from Temporal Feature Analyzers show strong similarity to slow-changing part of activations, while novel codes and SAEs primarily capture the fast-changing part.
    \vspace{-10pt}
    }
    \label{tab:side-table}
    \begin{tabular}{lccccc}
      \toprule
      & \multicolumn{3}{c}{SAEs} & \multicolumn{2}{c}{Temporal} \\
      \cmidrule(lr){2-4}\cmidrule(lr){5-6}
      & ReLU & TopK & BatchTopK & Novel & Pred \\
      \midrule
      Slow & 0.37 & 0.35 & 0.35 & 0.19 & 0.75 \\
      Fast & 0.54 & 0.54 & 0.54 & 0.75 & 0.18 \\
      \bottomrule
    \end{tabular}
  \end{minipage}
  \vfill
  \vspace{1pt}
  \begin{minipage}[t]{\linewidth}
    \centering
    \includegraphics[width=0.98\linewidth]{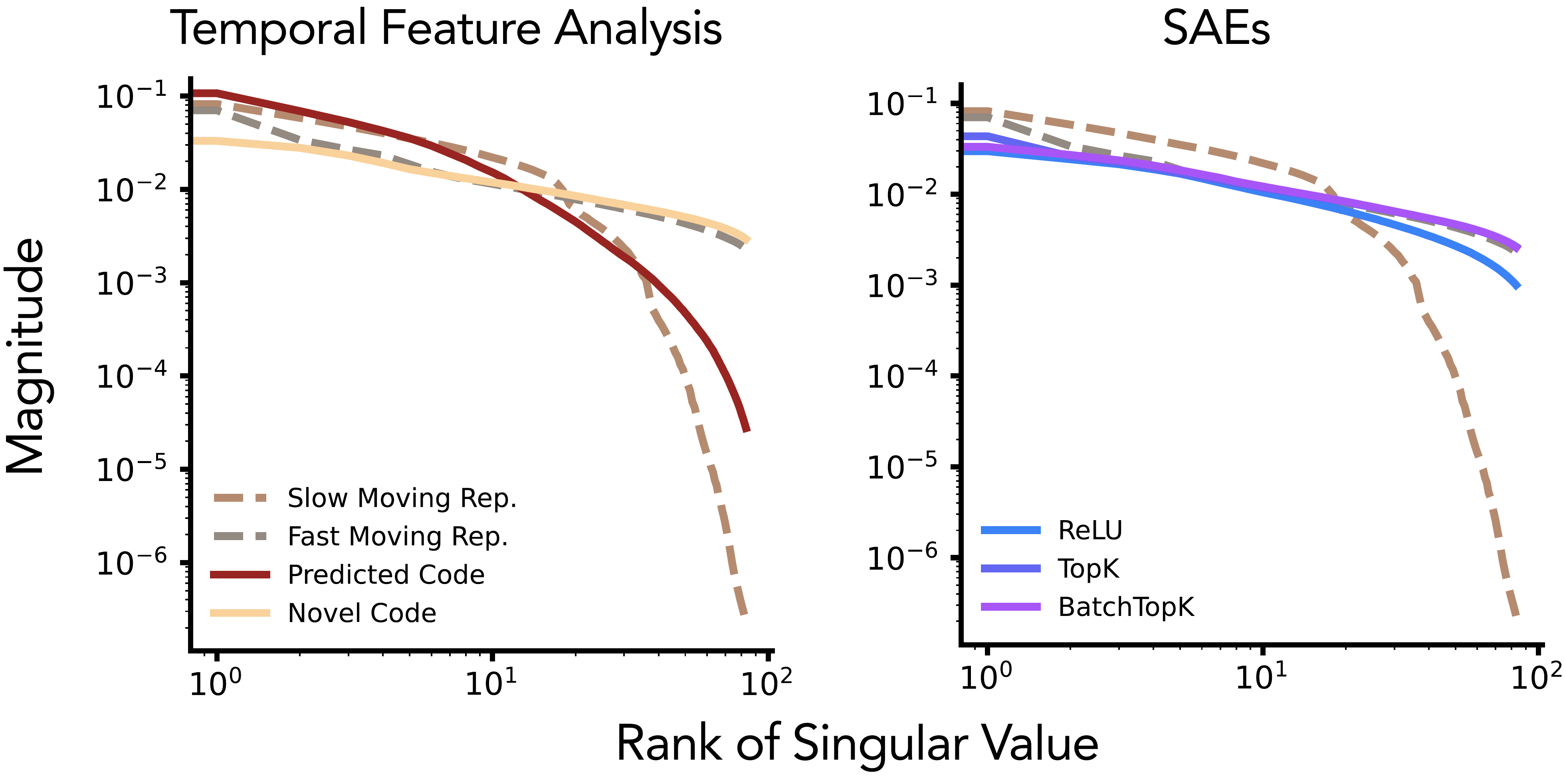}
    \vspace{-8pt}
    \caption{\textbf{Kernel spectrum for latent codes and model representations}. Kernels defined using novel code from Temporal Feature Analysis and standard SAEs both align well with the fast-changing part of model representations; meanwhile, only the predictive code shows strong similarity to the slow changing part.
    }
    \label{fig:fourier_spectra}
  \end{minipage}
\end{wrapfigure}

Performing the experiments above on latent codes extracted from Temporal Feature Analyzers, we see that while the novel component exhibits behavior that is similar to standard SAEs', the \textit{predictive component shows a smooth, regular curve} that clusters together several tokens from the story around a similar point in space.
Again running hierarchical clustering, the extracted dendrograms suggest these clusters qualitatively correspond to events occurring in the story.
We emphasize this finding is very similar to the phenomenon of \textit{temporal straightening} observed in neural recordings of human subjects when processing stories~\citep{xutemporal} and sequential visual data~\citep{henaff2019perceptual}!
Straightening simplifies future state prediction and recent findings show that such straightening is exhibited by language models~\citep{hosseini2023large}; however, Temporal Feature Analysis recovers this phenomenon in an unsupervised manner.

\textbf{Quantifying Similarity to Slow vs.\ Fast Moving Signals.} To further quantify the straightening claim, we compute the temporal Fourier transform of the model activations and divide the frequency spectrum into two halves at a critical frequency $f_c$ such that the energy (i.e., sum of squared value of phase information) in the frequencies below $f_c$ equals that of the remaining ones.
We call the first split ``slow part'' of a sequence, and latter the ``fast part''.
We then compute the correlation matrix defined by the slow and fast parts, compute their spectrum, and analyze how similar these spectrums are to the ones defined using different SAEs' codes.
Results are shown in Fig.~\ref{fig:fourier_spectra}.
We clearly see the spectrum extracted from the predictive component of Temporal Feature Analysis approximates that of the slow part of the representation, while the novel component's spectrum is similar to that of the fast moving part. 
Meanwhile, spectra of SAEs is primarily similar to the fast part, suggesting lack of ability to capture longer range dependencies necessary for interpreting narrative-driven texts like most language domains.
To quantify this further, we also measure kernel similarities (CKA) between the kernels respectively defined by the slow moving and fast moving signal to the different SAE codes (see Tab.~\ref{tab:side-table}).
Results clearly show Standard SAEs and novel codes are primarily similar to fast moving part of the representation, while predictive codes are similar to the slow moving part.

\begin{figure}
\vspace{-12pt}
    \centering
    \includegraphics[width=\linewidth]{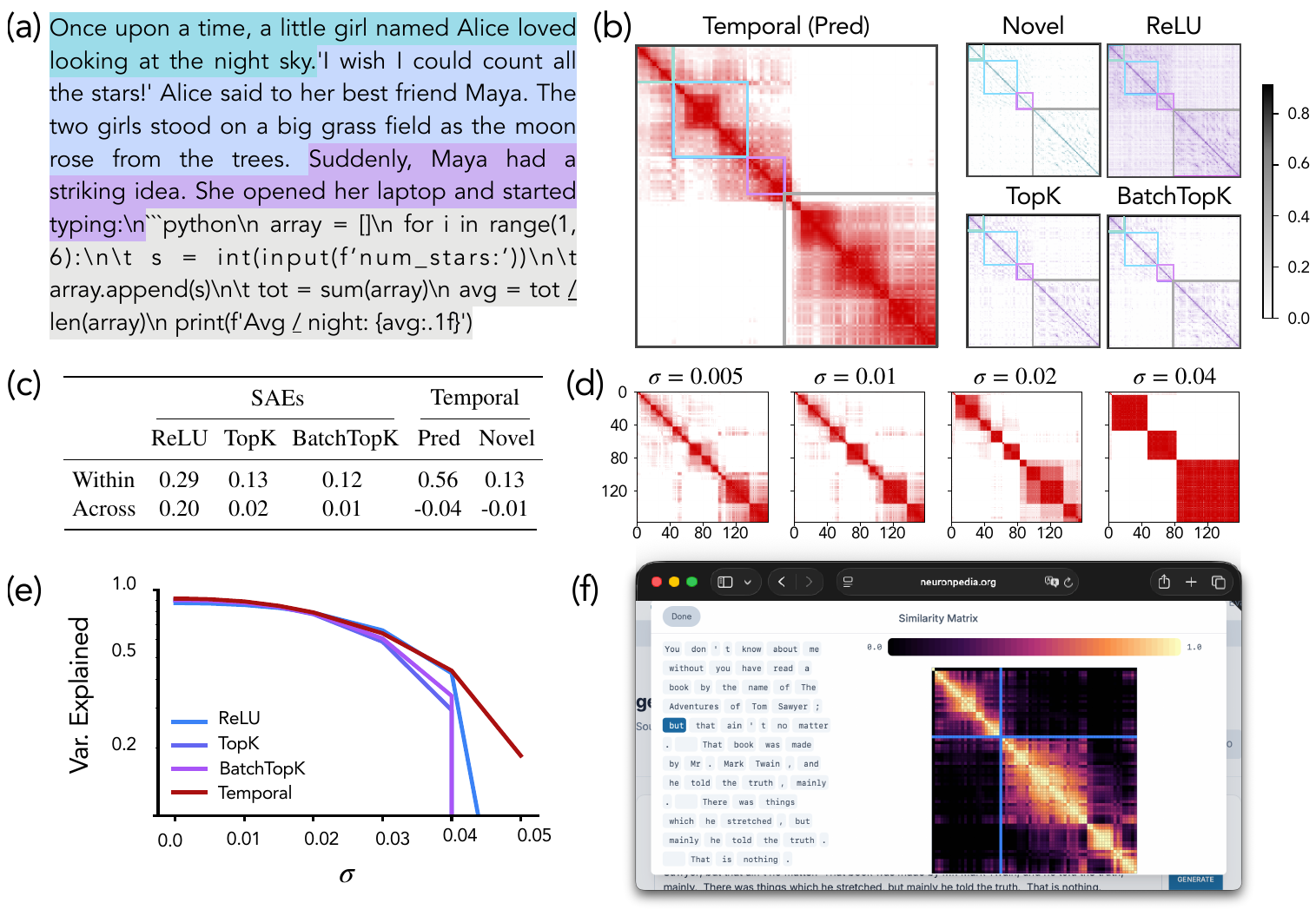}
    \vspace{-15pt}
    \caption{
\textbf{Predictive codes decompose stories into events.} (a) We consider model representations from a synthetic story with well-defined event boundaries.
(b) Computing pairwise cosine similarity of latent codes extracted using different protocols, we see the predictive code of Temporal Feature Analysis organizes in hierarchical block structures that seem to align with (sub)event boundaries in the analyzed story.
(c) We confirm the alignment of predictive codes with event boundaries by computing average pairwise similarity of token latent codes for tokens that span the same event (`within') versus not (`across'). Results clearly show high within-event similarity scores for the predictive code.
(d) and (e) The results above are further corroborated by running a noising process on the latent codes: we add Gaussian noise of scale $\sigma$ to the input before computing latent codes, defining the similarity maps and computing explained variance of un-noised data. This process elicits coarser grained clusters from the similarity maps for the predictive code, suggesting the multi-scale temporal structure of stories is reflected in predictive codes (see App.~\ref{appsection:gemma_results} for other latent codes).
(f) Our interactive interface generates the similarity matrix for predictive codes live. Detect event boundaries in your own stories: \url{https://www.neuronpedia.org/gemma-2-2b/12-temporal-res}.
\vspace{-10pt}}
\label{fig:stories}
\end{figure}

\paragraph{Predictive Component Captures Local Event Boundaries.} The results above demonstrate Temporal Feature Analyzers' predictive component qualitatively align with event boundaries in a story.
To investigate this result more quantitatively, we use GPT-5 to create a synthetic dataset of 50 stories with well-defined event boundaries (see Fig.~\ref{fig:stories}a for an example). 
We extract latent codes for these stories' tokens, center them by subtracting the mean to remove any globally shared information, and compute the cosine similarity of token to token latent codes.
If the latent codes reflect local event structure of a story, the cosine similarity (on average) will be high between token pairs that come from the same event and low (if not zero) between pairs sampled across events; see Fig.~\ref{fig:stories}b for an example similarity map corresponding to the story shown in Fig.~\ref{fig:stories}a.
Results are shown in Fig.~\ref{fig:stories}~(c) and corroborate our qualitative findings: we see predictive components from Temporal Feature Analyzers show substantially higher similarity of codes if tokens are sampled from within an event, while the novel component and SAEs generally show low similarity between two tokens. 
These results are further supported by the robustness of Temporal Feature Analysis to noise.
Specifically, we see that when we add noise to the input data, which, on average, will lead to turning off of latents with small magnitudes (due to the encoding nonlinearity), the temporal structure of the data, if it is present, will be amplified. 
We see precisely this effect in Fig.~\ref{fig:stories}d: the predictive components' cosine similarity map under Gaussian noised input maps elicits coarser block structures with increasing noise scale; this is reminiscent of percolation or heat diffusion perspectives on graph clustering, wherein noise diffuses only within a connected component and hence community structure is elicited~\citep{von2007tutorial}.
Correspondingly, we see Temporal Feature Analyzers respond most gracefully to noise: variance explained reduces slower than SAEs', which in fact drops to $\sim$0 at some critical noise scale.

\begin{figure}
    \centering
    \includegraphics[width=\linewidth]{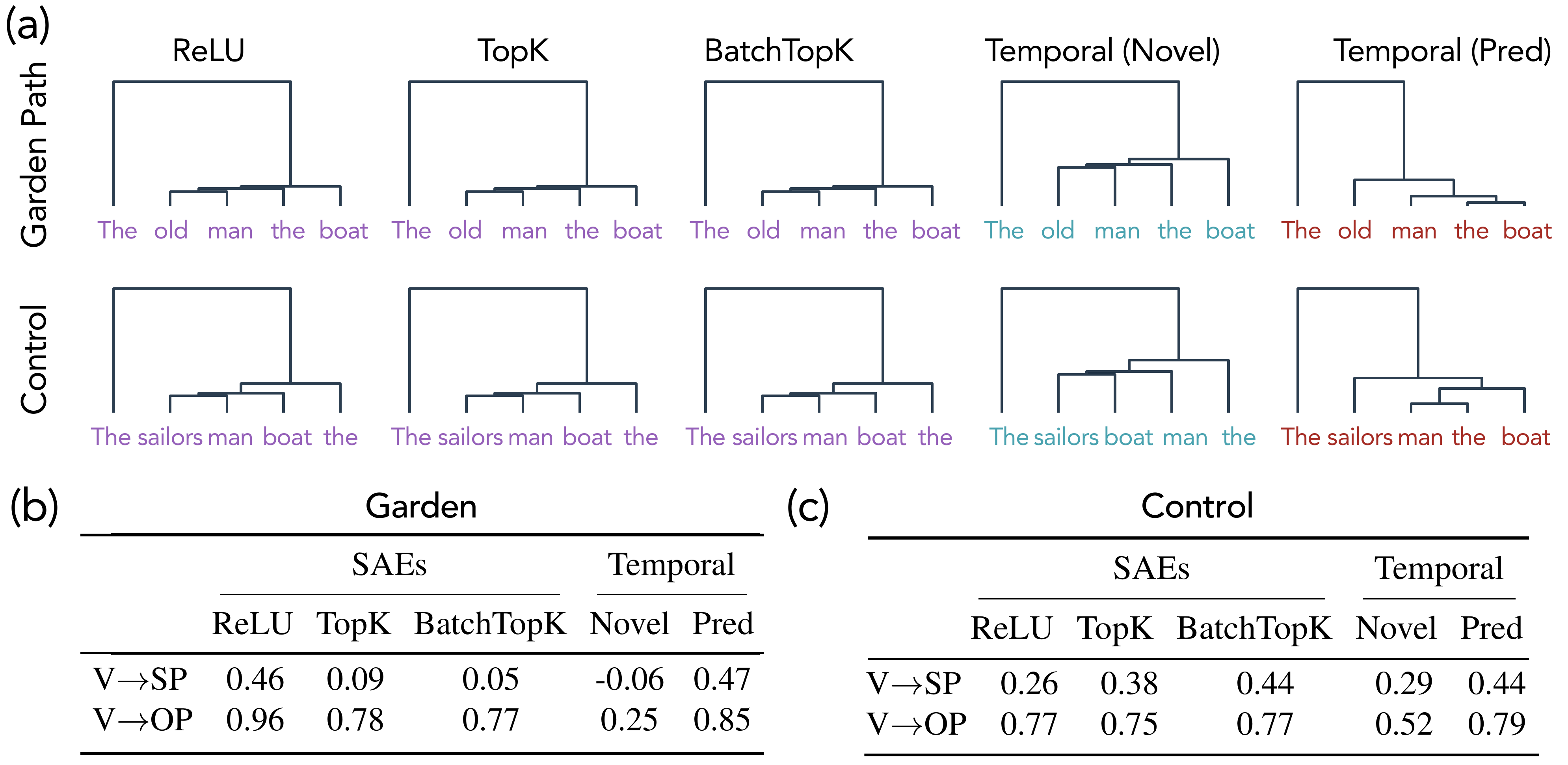}
    \caption{\textbf{Hierarchically clustering SAE codes for garden path sentences.} (a) Pairwise similarity maps of the predictive code from Temporal Feature Analysis link long-distance heads and dependents that define the ultimately correct parse in garden path sentences, while standard SAEs (e.g., BatchTopK) emphasize only local, transient relations, falling for the misleading cues.
    (b, c) Comparing the cosine similarity of average latent code extracted from the subject phrase (SP), verb phrase (V), and object phrase (OP), we see across ambiguous garden path sentences and unambiguous control variants thereof, only the predictive component of Temporal Features Analyzers shows consistent similarity scores (as expected if the SP and V ambiguity is reflect in the latent codes).
    }
\label{fig:garden_path}
\end{figure}

\subsection{Garden Path Sentences: Ambiguities Resolved via Temporal Structure} 
\label{sec:gp}

Garden-path sentences---e.g., \texttt{The old man the boat}---initially cue an incorrect local parse before a later token forces reanalysis. 
Language models have been shown to be able to correctly parse such ambiguous sentences, offering in fact a predictive account of human per-token surprisals~\citep{li2024incremental, hanna2024incremental, oh2023transformer}. 
Interestingly, when using SAE codes to assess whether LM activations offer a valid parse of the sentence, we find hierarchical clustering of SAE codes yields a parse that is suggestive of the misleading cue; meanwhile, Temporal Feature Analysis recovers the correct parse by separating the predictable, slow-moving component of the representation from the novel, fast-changing residual. 
Specifically, in Fig.~\ref{fig:garden_path}, we see the predictive codes link long-distance heads and dependents that reflect the correct parse (e.g., \texttt{man} as verb), producing coherent similarity structure over the full span, whereas standard SAEs emphasize only transient, local changes and miss these cross-temporal constraints (see App.~\ref{app:gemma_gp} for more examples). 

The results above suggest Temporal Feature Analyzers encode syntactic structure that unfolds over time when evidence to collapse the correct constituent parse emerges. 
To make these results more quantitative in nature, we use GPT-5 to synthetically generate a set of 50 garden path sentences where the subject is ambiguous. 
We create 50 control variants of these sentences such that the controls do not possess ambiguity with respect to typical parse of the sentence constituents: e.g., changing the subject in the sentence~\texttt{The old man the boat} from \texttt{old} to \texttt{sailors}, yielding~\texttt{The sailors man the boat}.
We then divide all sentences into their respective subject phrase (SP), verb phrase (V), and object phrase (OP): e.g., \texttt{The old} (SP), \texttt{man} (V), and \texttt{the road} (OP).
We compute the average latent code for tokens from these three constituent phrases, under the hypothesis that if the sentence ambiguity is reflected in the latent code, then until the OP shows up, the correct parse of prior words cannot be identified. Correspondingly, all valid parses must be stored in the same representation.
This suggests the cosine similarity of latent codes of V and SP tokens should be of a similar order in both the garden path and control sentences; meanwhile, the similarity between V and OP should be much higher than V and SP.
Results are reported in Fig.~\ref{fig:garden_path}~(b,c).
We clearly see extreme sensitivity in similarity values of SAE latent codes and the novel component of Temporal Feature Analyzers, but the predictive component is essentially invariant across sentence type, suggesting it captures the temporal dynamics likely relevant for a LM to parse garden path sentences.

\begin{figure}
    \centering
    \includegraphics[width=\linewidth]{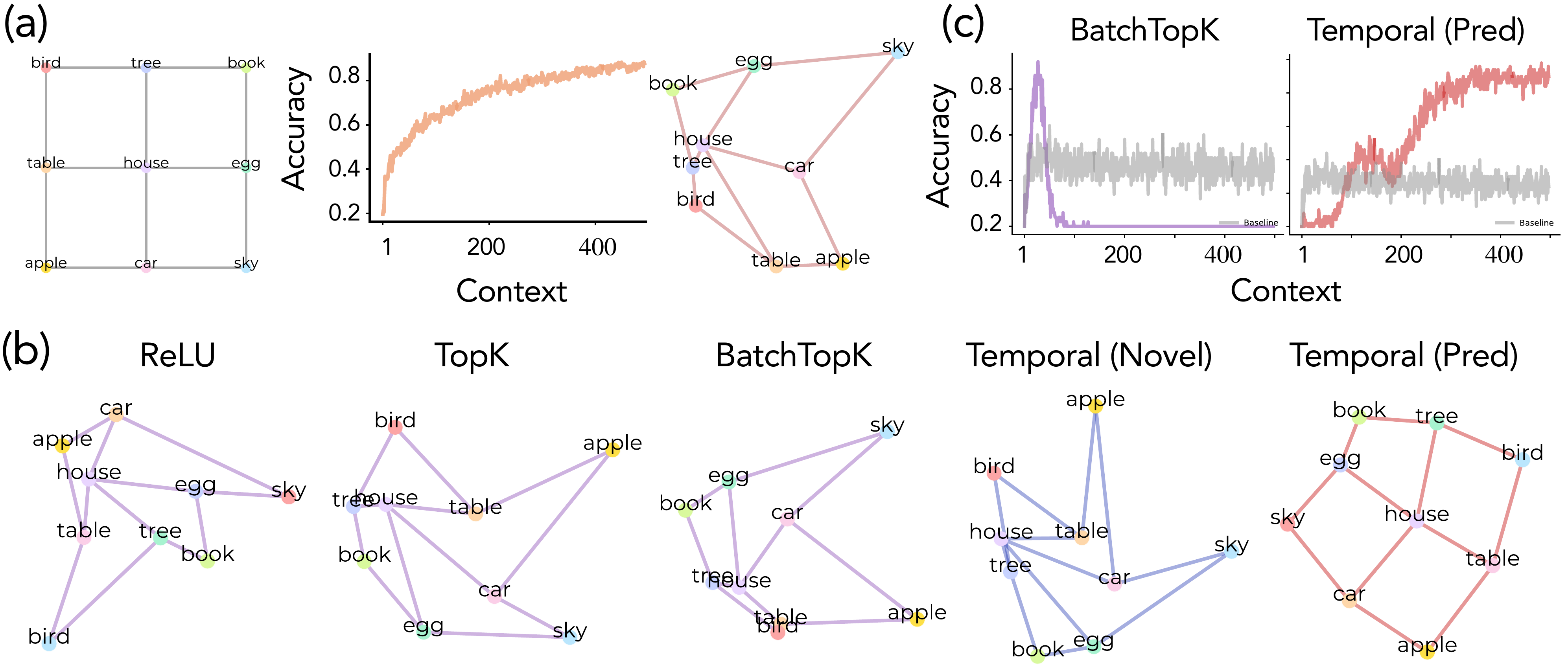}
    \vspace{-5pt}
    \caption{\textbf{In-context representations defined by temporal relations.} \textbf{(a)} Schematic of the synthetic ``cognitive map'' task defined by \citet{park2025iclrincontextlearningrepresentations}: a random walk over a small graph is concatenated into a string of tokens, which when fed into a Gemma-2-2B model leads to model next-token predictions that align with the graph structure. This behavior is exhibited in the model's activations, since a mere 2D PCA shows the graph structure. \textbf{(b)} Using latent codes for low-dimension projection, we see only the \emph{predictive} component of Temporal Feature Analysis recovers the graph’s geometry, whereas standard SAEs and the novel component only partially reflect it. \textbf{(c)} Counterfactual editing: projecting the final-token representation onto the PCA basis and replacing it with a target node’s code reliably shifts next-token predictions in a graph-consistent manner.
    \vspace{-5pt}
    }
\label{fig:iclr_iclr}
\end{figure}

\subsection{In-Context Representations defined by entirely temporal relations} 
\label{sec:iclr}

We finally use the recently described phenomenon of in-context representations by \citet{park2025iclrincontextlearningrepresentations} as a domain to compare SAEs and Temporal Feature Analyzers.
In particular, inspired by experiments on cognitive maps~\citep{behrens2018cognitive}, \citet{park2025iclrincontextlearningrepresentations} propose to take a structured graph whose nodes are assigned entities a pretrained LM is likely to have seen before during training (e.g., \texttt{apple}, \texttt{car}, etc.); see Fig.~\ref{fig:iclr_iclr}a for an example.
A random walk of length 500 tokens is performed on this graph and the sampled nodes during the walk are concatenated together into a string.
When this is inputted into a pretrained language model, provided enough context, one finds (see Fig.~\ref{fig:iclr_iclr}a): (i) the model begins to produce next-token predictions inline with the graph structure and (ii) a 2D PCA projection of the model activations reflects this structure.
That is, the graph structure, which is randomly defined and hence impossible for the model to have seen during pretraining, is an entirely temporal construct that is inferred by the model in-context.

Can latent codes from SAEs or Temporal Feature Analyzers capture the above phenomenon?
To analyze this, we repeat the low-dimensional projection experiment with model activations, but use the latent codes instead. 
We find that SAEs' and the novel component from Temporal Feature Analyzers partially (but not entirely) reflects the geometry of the graph.
Meanwhile, the geometrical structure of the graph is vividly observed in the predictive component of the Temporal Feature Analyzers.
To further confirm of the utility of this representation, we perform a simple counterfactual analysis: we sample walks of different lengths, project the model activations for the last token onto the PCA basis defined in Fig.~\ref{fig:iclr_iclr}b, and finally replace it with the representation of a `target' node.
Ideally, this intervention changes the model's next-token predictions as expected by the neighborhood of the target node.
Results show for $\nicefrac{4}{9}$ nodes, the predictive component yields a counterfactual result as expected by the graph structure; meanwhile, no other latent representation offers the ability to counterfactually edit model beliefs (see Fig.~\ref{fig:iclr_iclr}c for an example).

\section{Further Investigation of Temporal Feature Analysis}
In the results above, we showed across three domains that Temporal Feature Analysis elicits intricate in-context geometry inline with our expectations based on experiments with LM activations in Sec.~\ref{sec:llmtemporalstructure}.
Next, to make a case for following such more structured approaches for interpretability, we add to the results above to show that Temporal Feature Analysis (i) offers a new way of interpreting temporally structured domains (e.g., user--model chats) via the context-dependent predictive component, and (ii) information that can be identified via SAEs continues to remain available within the novel component of Temporal Feature Analysis.

\begin{figure}
    \centering
    \includegraphics[width=\linewidth]{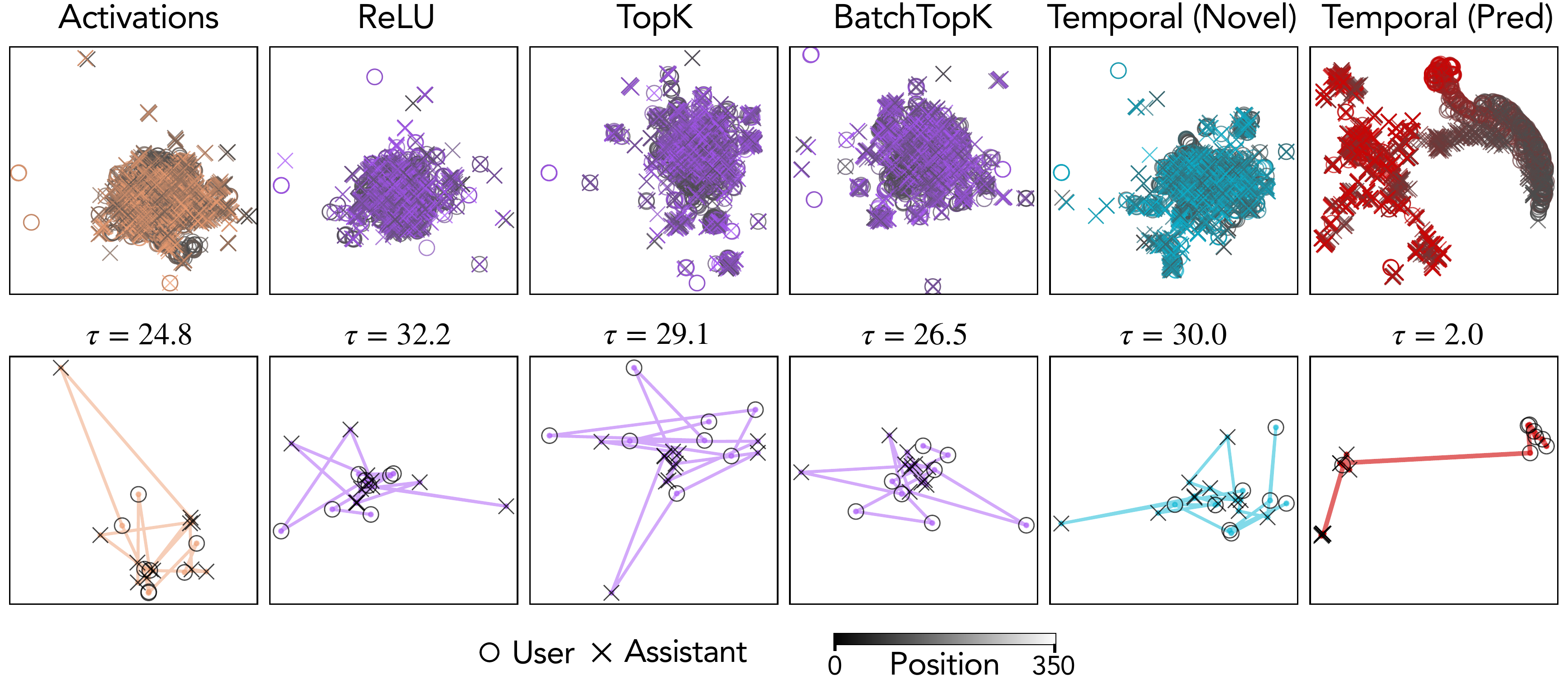}
    \caption{\textbf{Population and trajectory geometry in OOD dialogue.} UMAPs of raw activations and SAE codes on Ultrachat (first 350 tokens per conversation). \textbf{Top:} population embeddings colored by sequence position (gray$\to$color) and marked by role (user=\(\circ\), assistant=\(\times\)); only the \emph{predictive} code cleanly separates speaker roles while preserving a smooth temporal gradient. \textbf{Bottom:} a single conversation overlaid on the same UMAP; the predictive code follows a smooth near-geodesic with low tortuosity \(\tau\), in contrast to jagged paths from standard SAEs/novel codes.}
\label{fig:pop_geometry}
\end{figure}

\subsection{Geometry of latent codes in an OOD dialogue domain}
We analyze a multi-turn chat domain to test whether temporal feature analysis continues to reveal slow-moving structure. 
We emphasize the activations in this experiment are sampled from a Gemma-2-2B-IT model, although all analyzed interpretability protocols were trained on the base (non-instruction tuned) model.
Thus, in a sense, the experiments in this section also gauge the ability of different protocols to generalize out-of-distribution (OOD)---a property SAEs have been argued to partially possess when transferring between base and the corresponding instruction-tuned models~\citep{saetransfer}.
Specifically, we use the Ultrachat dataset, where we take the first 350 tokens from 1{,}000 conversations (mostly single-turn within this window, with up to four turns). 
For each token, we compute a shared 2D UMAP embedding from (i) raw model activations and (ii) latent codes from ReLU, TopK, and BatchTopK SAEs, as well as the novel and predictive components of Temporal Feature Analyzers. 
To visualize population structure clearly, we display only 100 conversations at a time to avoid clutter (Fig.~\ref{fig:pop_geometry}), though fitting is applied to all 1{,}000. 

At the population level, the predictive component of Temporal Feature Analysis is the only representation that cleanly organizes the data along \emph{two} interpretable axes: (1) a \emph{speaker-role separation}, with user and assistant tokens forming distinct yet smoothly adjoining manifolds, and (2) a \emph{temporal gradient}, with positions flowing continuously from early to late tokens along the manifold. By contrast, standard SAEs and the novel code show diffuse clouds with weaker role separability and no coherent positional gradient (Fig.~10, top row).
To assess within-sequence geometry, we overlay a single conversation's trajectory onto the population embedding (Fig.~10, bottom row). Predictive codes trace a \emph{smooth trajectory} that remains locally coherent within each role segment and transitions gracefully at speaker switches. In contrast, standard SAEs and the novel code produce jagged, zig--zagging paths. We again quantify this with \emph{tortuosity} $\tau$~\citep{bullitt2003measuring} (lower is straighter): predictive codes achieve $\tau \approx 2.0$, whereas SAEs and the novel code yield $\tau \approx 25$--$30$, indicating substantially more local turning (Fig.~10, panel headers). This straightening mirrors the story-domain results (Sec.~\ref{sec:stories}), now in an OOD conversational chat regime.

\begin{figure}
    \centering
    \includegraphics[width=\linewidth]{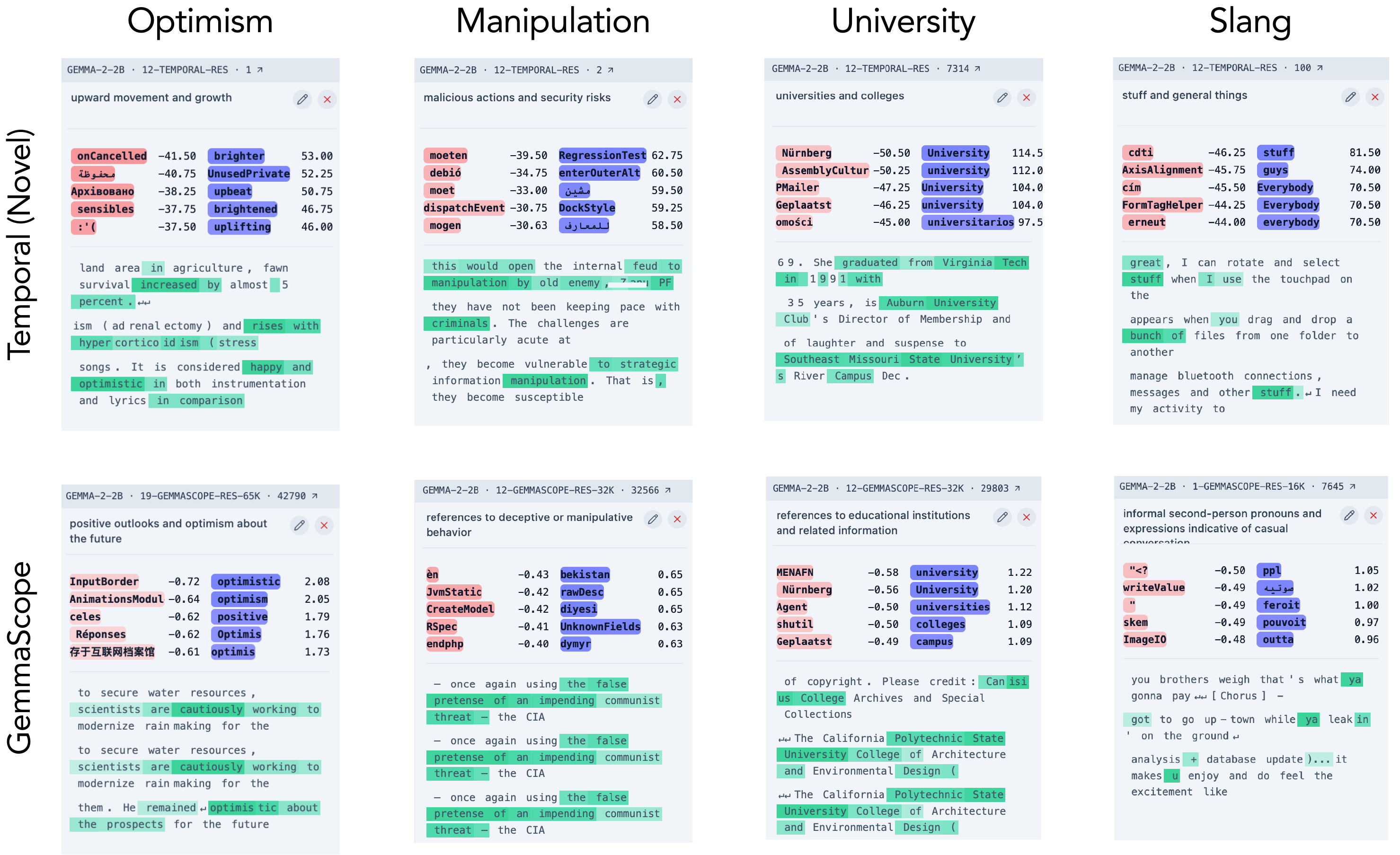}
    \caption{\textbf{Interpretability of novel codes.} 
    We qualitatively compare the interpretability of novel codes extracted using Temporal Feature Analysis with latent codes extracted using GemmaScope SAEs across four diverse categories. Example feature cards are shown across four categories (Optimism, Manipulation, University, Slang), produced with the same activation-triggered examples plus direct logit attribution workflow (GemmaScope / Neuronpedia) used for standard SAEs. Browse the full list of features  \href{https://www.neuronpedia.org/list/cmh9dj4cq0001i9h}{on Neuronpedia}.}
\label{fig:novel_code_autointerp}
\end{figure}

\subsection{Automatic Interpretability of Novel Codes}

A central benefit of SAEs is that their latents admit automatic interpretability pipelines (``autointerp'') that attach human-readable semantics to features via activation-triggered examples, token/phrase saliency, and lightweight dashboards (e.g., GemmaScope / Neuronpedia). Our Temporal Feature Analyzer adds a new, temporally aware predictive axis, but crucially \textit{it does not sacrifice} these established interpretability affordances: the novel code remains compatible with standard SAE autointerp workflows. In other words, Temporal Feature Analysis augments the toolbox rather than replacing it.
To support this point, we apply a conventional autointerp loop to the novel codes: (i) collect high-activation snippets and nearest-neighbor contexts for each latent; (ii) compute direct logit attribution to inspect which vocabulary items a latent tends to support or suppress; and (iii) render compact, browsable ``feature cards'' that aggregate examples and scores (as in GemmaScope / Neuronpedia). While we do not perform an exhaustive causal-intervention study here, nothing in the pipeline is specific to standard SAEs; the same probes and dashboards operate unchanged on the novel codes, and we leave a systematic intervention suite to future work.
Qualitative inspection confirms that the novel codes recover the kinds of monosemantic, language-level features typically reported for SAEs. Figure 11 (p. 13) shows representative examples across diverse categories---Optimism, Manipulation, University, and Slang---where the novel latents consistently activate on intuitively relevant cue phrases and contexts, and where their attribution profiles align with the expected lexical sets. 
Overall, when standard SAEs produce readable feature cards, the Temporal Feature Analysis's novel stream does as well. This complements our earlier results: the predictive stream captures slow, contextual structure (events, roles, long-range constraints), whereas the novel stream concentrates the fast, stimulus-driven information that matches inferences possible via existing SAEs, including autointerp pipelines.

\section{Discussion}
\label{sec:discussion}

A longstanding view in computational and cognitive neuroscience treats perception as predictive and time-coupled: slowly varying, context-invariant structure coexists with fast-changing, surprising residuals~\citep{levy2008expectation, millidge2024predictive, shain2020fmri}. 
We find that this structure is reflected in LM activations (Sec.~\ref{sec:llmtemporalstructure}): activations are strongly correlated with recent context and can be decomposed accordingly. 
In contrast, SAEs assume independence over time; this is misaligned with the identified non-stationary structure in language model activations. 
Our protocol---Temporal Feature Analysis---relaxes this assumption, incorporating empirically observed temporal correlations to surface structure that i.i.d.\ priors obscure. 
We showed that Temporal Feature Analysis is able to highlight temporal structures in its predictable codes that SAEs are unable to capture due to their prior-data mismatch. 
Overall, our work adds to recent findings that inductive biases about the specific application under study are necessary to build improved interpretability tools~\citep{hindupur2025projecting, costa2025flat, sutter2025the}.

\subsection{Future Work}
Our findings provide indirect support for alternative perspectives on the structure of concepts in language models. Here, we highlight these perspectives, and emphasize directions that we believe will have significant impact.

\paragraph{Features as Low-Dimensional Manifolds.} We believe our work brings interpretability of language models one step closer to recent advances in interpretability of neural networks designed to simulate computational neuroscience tasks~\citep{sussillo2013opening, cotler2023analyzing, mante2013context, maheswaranathan2019universality, maheswaranathan2020recurrent}.
Specifically, rather than isolated directions, as is assumed in both sparse coding~\citep{olshausen1996emergence, olshausen1997sparse} and SAEs~\citep{elhage2022superposition, bricken2023monosemanticity}, ``features'' are better viewed as low-dimensional manifolds that evolve over time---as has also been stated in other recent work~\citep{modell2025origins, fel2025into, gurnee2025when}. 
In particular, we see in our results that predictive codes trace smooth, stable trajectories (events, roles, long-horizon constraints), whereas novel codes capture transient excursions. 
Under this framing, high reconstruction with poor temporal structure in standard SAEs reflects a prior that fragments local manifold geometry (as formalized in Sec.~\ref{sec:saepriors}); adding the temporal prior helps preserve the geometry that yields temporally coherent representations~\citep{berkes2005slow,olshausen2007learning}. 

\paragraph{Temporal Straightening in Language Model Representations.} Across narrative and dialogue settings, the Temporal Feature Analyzer's predictable code organizes representations into hierarchical blocks and smooth paths---an effect reminiscent of ``temporal straightening'' ~\citep{henaff2019perceptual,hosseini2023large,xutemporal}, where the representational trajectory of an input sequence becomes linearized to support extrapolation toward future states of the sequence. In stories, predictive codes align with (sub)event boundaries and show robustness under noise; in OOD dialogue, they cleanly separate roles and produce near-geodesic single-conversation paths. Kernel spectra and CKA further support the slow/fast split: predictive codes align with slow components; novel codes and standard SAEs align with fast ones. We believe this dynamic warrants further investigation, and may serve fertile ground to develop computational hypotheses about temporal straightening itself.

\paragraph{Priors in Temporal Feature Analysis.} The encoding process in Temporal Feature Analysis realizes the predictable component with a single-layer self-attention module, which may bias the kinds of dependencies it can capture. Future work could place the inductive bias inside the sparsity regularizer (retaining SAE encoders) or adopt richer temporal encoders. More broadly, we assume i.i.d.\ residuals in the novel stream; relaxing this to allow controlled temporal correlations is an important direction. Finally, scaling causal evaluations---intervening on codes during generation---will help move from qualitative alignment to mechanistic claims.

\paragraph{Training Dynamics and Possible Variants of Temporal Feature Analysis}
Since the goal of proposing Temporal Feature Analysis was to identify and highlight the benefits of accommodating the temporal structure of model activations, we did not perform significant tuning of the proposed protocol. That is, we strongly emphasize that our proposed protocol is just an example and we hope much better variants will likely follow as the community explores this research direction---as, e.g., happened for SAEs, which started with vanilla ReLU variants~\citep{cunningham2023sparse, bricken2023monosemanticity}, eventually leading to much more sophisticated approaches eventually~\citep{rajamanoharan2024jumpingaheadimprovingreconstruction, bussmann2024batchtopksparseautoencoders, costa2025flat}. To assist this, we highlight the following.

\begin{tfaitem}[Competition dynamic in the learning process.]%
  {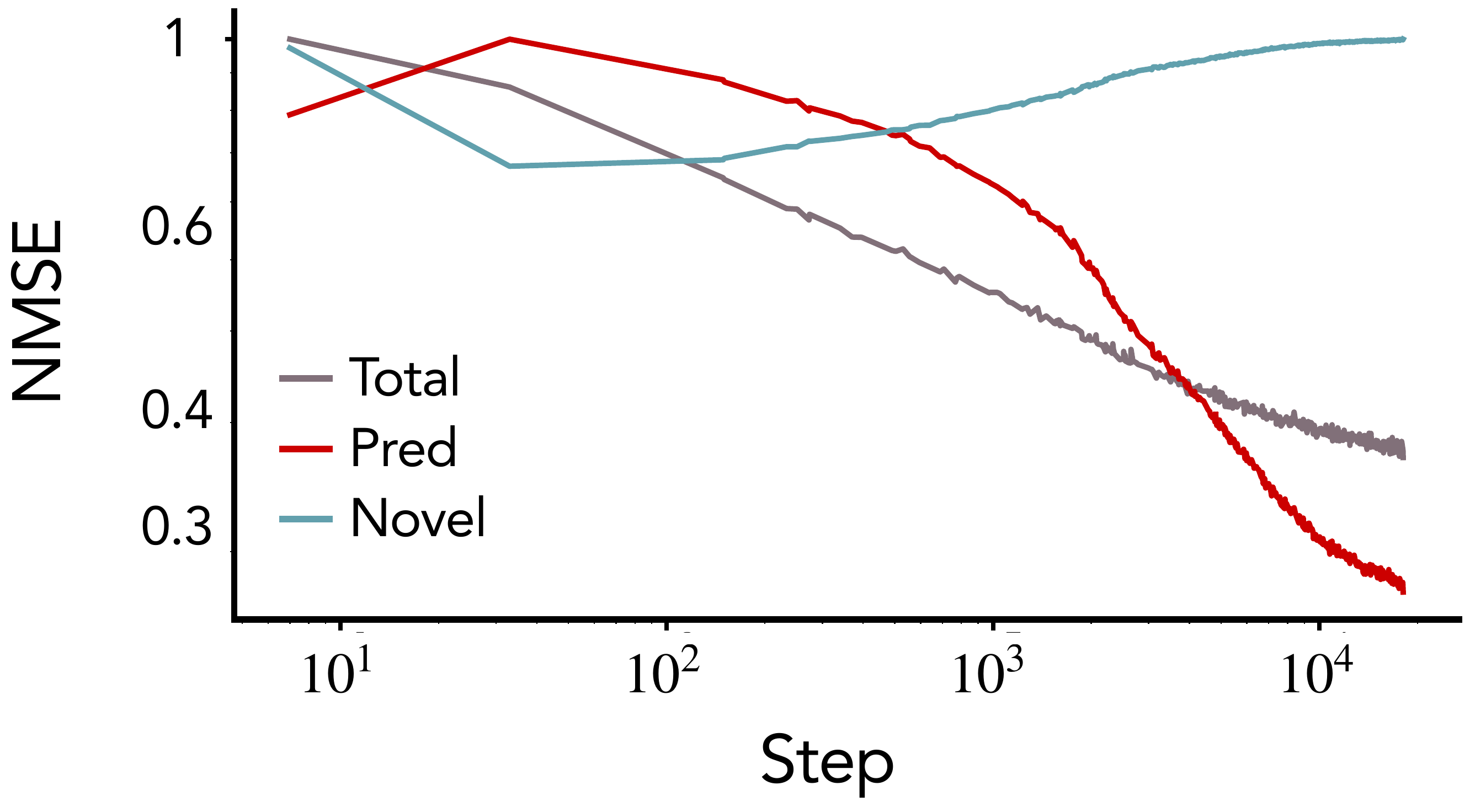}%
  {\textbf{Competition Dynamics.} We see while the overall learning curve looks smooth, the predictive and novel code in fact compete with each other to explain the input signal.}%
  {fig:gemma_curves}
The precise way the training objective is defined in the current version of Temporal Feature Analysis (see Eq.~\ref{eq:tfa_loss}), there is a competition dynamic imposed between the predictive and novel component for explaining model representations. Specifically, we found that learning dynamics underwent three phases (see Fig.~\ref{fig:gemma_curves}). First, since the objective of \textit{predicting} a future direction is more challenging than reconstruction to begin with, we found that the novel part, which is defined using a standard SAE and can in fact reconstruct a bulk of the training signal (since SAEs are trainable in the first place), will dominate the learning dynamics. Then, once the dictionary quality improves, the ability to predict starts to improve and the learning dynamics goes through a chaotic phase with successful training resulting in the predictive component taking over the learning dynamics. Finally, we saw that the novel component keeps conceding its dominance over the reconstructed signal, such that its NMSE finally becomes very high and the predictive component's NMSE becomes very low. Unsuccessful training, however, can also happen here: at times we saw a fourth phase occurs where the novel component, relatively suddenly, takes over the learning dynamics. This would make the predictive component futile. This was rarely seen when working with Gemma models, but a frequent occurrence in Llama models; we resolved this problem by switching from using TopK to BatchTopk activation for the novel component, which resulted in perfectly stable training for Llama. We believe there are principled mechanisms to address this challenge, however. For example, Temporal Feature Analysis, in its structure, is very similar to Kalman filters~\citep{welch1995introduction}---latter generally involves Markov assumption, but Temporal Feature Analysis uses the entire context. Seeking inspiration from training objectives for Kalman filters, it can be useful to decompose the loss into two halves: (i) a term that uses only the predictive component to reconstruct the signal, and (ii) a term that uses the novel component to reconstruct just the residual. This breaks the competition dynamic by introducing an asymmetry in the training objective. 
\end{tfaitem}

\begin{tfaitem}[Rank sparsity.]%
  {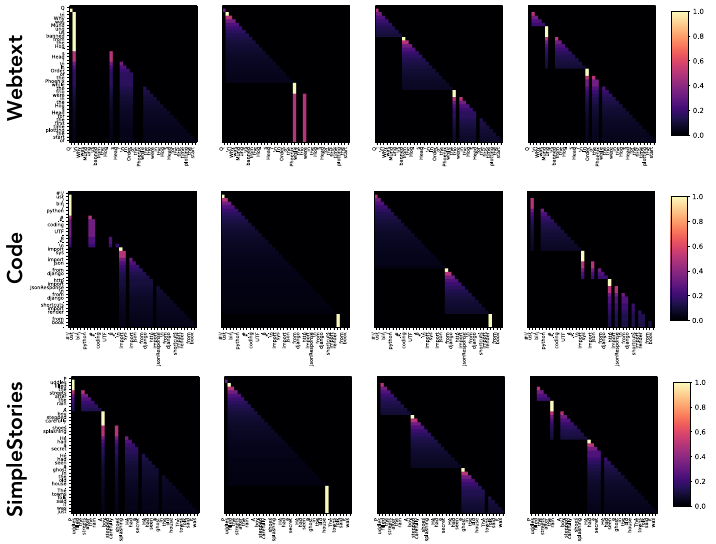}%
  {\textbf{Attention Patterns.} Attention weights for single sequences across three domains show low-rank structure in similarities.}%
  {fig:attn_pattern_pred}
While the predictive component is dense in the sense used in interpretability literature, i.e., its $L_0$ norm is not constrained and can be high, in practice we found the learned attention maps and cosine similarity matrices (e.g., in Sec.~\ref{sec:stories}) to be very low-rank (e.g., see Fig.~\ref{fig:attn_pattern_pred}). In fact, we note in the broader field of dictionary learning and sparse coding~\citep{elad2010sparse}, an alternative notion of sparsity is \textit{rank-sparsity}. Specifically, in this paradigm, one constrains the rank of a matrix to be low: in the Euclidean sense, the latent code to reconstruct the target signal will still be dense, but the number of directions used to achieve that, in an alternative space, will be small. While one can optimize for this property using a nuclear norm like regularization objective, this can be expensive in practice (due to the need of an SVD). An alternative is thus to constrain the dimensions of the factors used for a signal's reconstruction, i.e., what is often called imposing a ``rank bottleneck''. When using this methodology on the predictive component, i.e., reducing the dimensions in the K, Q, V matrices of the attention layer, we find the resulting Temporal Feature Analyzers to achieve similar reconstruction as the unconstrained ones. While we did not explore this thread further, we believe it can be promising and is related to recent work on interpreting attention layers~\citep{he2025towards, kamath2025tracing}.
\end{tfaitem}

\begin{tfaitem}[Layer choice.]%
  {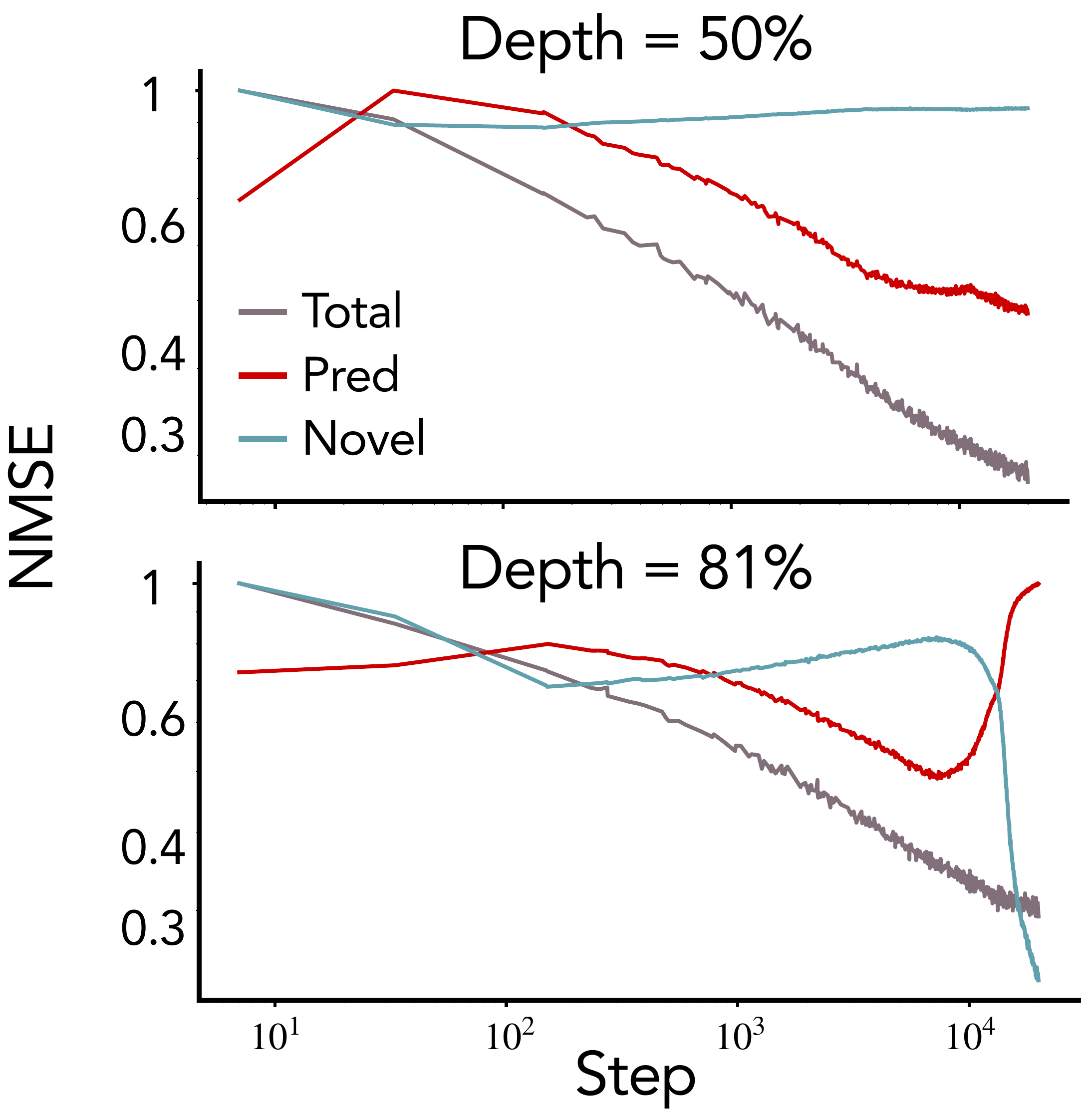}%
  {\textbf{Llama Learning Curves.} Temporal Feature Analysis works best in middle layers of the model, with training on deeper layers often showing a learning dynamic whereby the novel part eventually takes over and starts to explain a bulk of the signal.}%
  {fig:llama_curves}
For all our experiments, we trained SAEs and Temporal Feature Analyzers at layers found around 50\% model depth (Layer 12 for Gemma and Layer 15 for Llama; both 0-indexed). When we tried to train Temporal Feature Analyzers on a deeper layer---specifically, layer 26 ($\sim$81\% depth) of Llama---we found the predictive component would not be dominate the norm of the latent code (see Fig.~\ref{fig:llama_curves}). \textit{We believe this is a feature, not a bug}. In particular, deeper layers have to unembed the extracted contextual signal to make a next-token prediction, and correspondingly a model is likely to discard unnecessary information in the deeper layer that can interfere with its ability to confidently produce next-token predictions. This claim is corroborated by block structured similarities seen in CKA and other representation similarity analysis literature~\citep{kornblith2019similarity, raghu2017svcca}. We believe this thread warrants further investigation, as it may hint at a deeper result on temporal dynamics be varied across different layers, such that one could argue SAEs have the right inductive biases for interpreting deeper layers, but not middle or earlier layers.
\end{tfaitem}

\newpage
\section*{Acknowledgments}
ESL thanks members of the Mechanisms team at Goodfire AI, particularly Tom McGrath, Owen Lewis, Jack Merullo, and Atticus Geiger for helpful discussions. ESL additionally thanks Jack Lindsey for useful comments that helped formalize the intrinsic dimensionality arguments in Sec.~\ref{sec:saepriors}, and Raphael Sarfati for suggesting the Tortuosity metric used in Fig.~\ref{fig:ind_story_geometry}.
The authors also thank David Bau for a helpful discussion about the generalization ability of SAEs, and David Klindt for several useful comments on an earlier draft of the paper.
ESL and SSRH further thank Noor Sajid and members of the CRISP lab at Harvard for helpful conversations.
CR is supported by a MATS extension grant. 
AM's work is partially supported by the National Science Foundation (Grant No.\ 2530728) and Binational Science Foundation.
This work has been made possible in part by a gift from the Chan Zuckerberg Initiative Foundation to establish the Kempner Institute at Harvard University.

\bibliographystyle{plainnat}
\bibliography{ref}

\appendix
\newpage
\addcontentsline{toc}{section}{Appendix}
{%
  \hypersetup{linkcolor=blue}%
  \part{Appendix}%
  \parttoc%
}

\newpage

\section{Experimental details}
Below, we discuss broad details of the experimental setup and measures used for corresponding evaluations.
For experiments with LM representations, we primarily analyze Gemma-2-2B and Llama-3.1-8B models in this paper. 
Specifically, we precache 10K activations for the analyzed datasets for our experiments on dynamics on language model representations.

\subsection{Autocorrelation Between Activations}
\label{app:autocorr}

We compute autocorrelation by selecting evenly spaced tokens across the sequence and measuring the cosine similarity between each token and tokens at various lags in the past. Specifically, for tokens at position $t$, we compute similarities to tokens at $t-w$ where lag $w$ ranges from 5 to 20. This creates a heatmap where rows represent lag offsets and columns represent token positions.

For a stationary process, we expect the autocorrelation pattern to remain consistent across time---that is, the relationship between a token and its historical context should be similar regardless of position in the sequence. This would manifest as similar autocorrelation patterns repeating horizontally across token positions. In contrast, for a non-stationary process where representations evolve over time, we expect the autocorrelation patterns to vary systematically across positions, with columns showing different temporal dependency structures as the sequence progresses.

\subsection{Projection Analysis}
\label{appendixsubsection:context_projection_exp_var}

This analysis quantifies how much of the representation $\x_t \in \mathbb{R}^D$ at token position $t$ can be reconstructed from its preceding context $\{\x_i\}_{i \in W}$, where the context window $W = [t-w, t-1]$ contains the previous $w$ tokens.

For a population of $B$ sentence samples, we compute the projection of each target representation onto the subspace spanned by its context. Let $\mathbf{X}_{W}^{(b)} = [\x_{t-w}^{(b)}, \ldots, \x_{t-1}^{(b)}]^T \in \mathbb{R}^{w \times D}$ denote the matrix of context representations for sample $b$, where each row is a context vector. 
Prior to projection, we center all representations by subtracting the mean computed over the target positions across samples. The projection of $\x_t^{(b)}$ onto the span of the context is:
\begin{align}
    \mathbf{c}_{t,w}^{(b)} = {\text{span}(\mathbf{X}_W^{(b)})} \cdot \x_t^{(b)}
\end{align}
where $\text{span}(\mathbf{X}_W^{(b)})$ is the row space of $\mathbf{X}_W^{(b)}$, that can be computed via SVD.
The variance explained by the context is:
\begin{align}
    \text{expvar}(t,w) = \frac{\sum_{d=1}^D \text{var}(\mathbf{c}_{t,w,d})}{\sum_{d=1}^D \text{var}(\mathbf{x}_{t,d})}
\end{align}
where $\text{var}(\mathbf{c}_{t,w,d})$ and $\text{var}(\x_{t,d})$ denote the variance of the $d$-th dimension across the $B$ samples. This ratio measures the fraction of total representational variance that can be linearly reconstructed from the preceding context. Values approaching 1 indicate high predictability from context; values near 0 indicate representations largely orthogonal to their context subspace.

\subsection{U-statistic}
\label{appendixsubsection:ustat}

We initially experimented with more well-known measures of intrinsic dimensionality, but found them to either be too computationally expensive~\citep{costa2005estimating} or saturate due to inappropriate assumptions (e.g., data satisfying a Euclidean geometry)~\citep{carter2009local}.
We thus defined an intrinsic dimensionality measure for a specific generative process for activations that is motivated by the linear representation hypothesis~\citep{park2023linear, elhage2022superposition, costa2025flat, arora2018linear}.
Specifically, assume that each observation (activation) $x_t \in \mathbb{R}^D$ is a unit-norm vector generated from an (unknown) overcomplete dictionary $V = [v_1,\dots,v_N] \in \mathbb{R}^{D\times N},$ where $D < N$, whose columns are approximately orthonormal, in the sense that $\langle v_i,v_j\rangle$ is small for $i\neq j$ and $\|v_i\|_2=1$.
At time $t$, the observation is $x_t = V a_t$ and $\|x_t\|_2 = 1$, where $a_t \in \mathbb{R}^N$ is a random coefficient vector with a time-dependent distribution $P_t(a)$.
Thus the process is nonstationary through the evolution of $P_t$ (and hence $a_t$), while the dictionary $V$ is fixed, matching the phenomenology observed in Sec.~\ref{sec:llmtemporalstructure}.

Under the assumptions above, we measure the intrinsic dimensionality of LM activations based on pairwise cosine similarities. Specifically, let $\mathbf{X}_t =[\x_t^{(1)}, \dots, \x_t^{(M)}]$ be $M$ samples of normalized activations at time $t$ (from different timeseries). Define the gram matrix $\mathbf{G}_t = \mathbf{X}_t^T \mathbf{X}_t$. Then, a U-statistic of intrinsic dimension can be defined as follows: 
\begin{align}
    \text{U-stat}(t) = \frac{M^2 - M}{\|\mathbf{G}_t\|_F^2 - M},
\end{align}
where $\|\mathbf{G}_t\|_F^2$ is the squared Frobenius norm of the Gram matrix. 
As we show in the following subsection, this quantity is an estimator of the effective rank $\nicefrac{1}{\text{tr}(\mathbf{C}_t^2)}$, where $\mathbf{C}_t = \mathbb{E}[\x_t \x_t^T]$ is the second moment matrix and $\x_t$ is the activation vector at time $t$. 
Under stationarity, U-stat remains constant; when representations evolve over time, U-stat increases systematically as more orthogonal directions become active.

\subsubsection{Why is U-stat an intrinsic dimensionality measure}

Let $P_t$ denote the distribution of $x_t$ at time $t$ and define the (time-local) second-moment matrix $C_t \;:=\; \mathbb{E}_{x \sim P_t}[x x^\top] \in \mathbb{R}^{D\times D}$.
Assuming $\|x\|_2=1$, we have $\operatorname{tr}(C_t) = 1$.
Given $M$ samples $\{x_{t,1},\dots,x_{t,M}\}$ drawn from $P_t$ (e.g., from a time window or an ensemble of trajectories at time $t$), we form the Gram matrix $G_t \in \mathbb{R}^{M\times M}$ with entries $(G_t)_{ij} = x_{t,i}^\top x_{t,j}$.
The corresponding U-statistic
\[
\hat U_t \;:=\; \frac{1}{M(M-1)} \sum_{i\neq j} (x_{t,i}^\top x_{t,j})^2 \;=\; \frac{\|G_t\|_F^2 - M}{M^2 - M}
\]
is a consistent estimator of 
\[
U_t \;:=\; \mathbb{E}_{x,x' \sim P_t}[(x^\top x')^2] \;=\; \operatorname{tr}(C_t^2),
\]
where $x$ and $x'$ are independent draws from $P_t$.
Thus $\hat U_t$ provides a basis-free estimate of $\operatorname{tr}(C_t^2)$.

Now, a natural, rotation-invariant notion of intrinsic dimensionality at time $t$ is the (order-2) effective rank of $C_t$, $d_{\mathrm{eff}}(t)\;:=\; \frac{(\operatorname{tr} C_t)^2}{\operatorname{tr}(C_t^2)}$.
For unit-norm data $\operatorname{tr}(C_t)=1$, so
\[
d_{\mathrm{eff}}(t) \;=\; \frac{1}{\operatorname{tr}(C_t^2)} \;=\; \frac{1}{U_t},
\qquad
\hat d_{\mathrm{eff}}(t) \;:=\; \frac{1}{\hat U_t}
\;=\;
\frac{M^2 - M}{\|G_t\|_F^2 - M}.
\]

Let $\lambda_1(t),\dots,\lambda_D(t)$ be the eigenvalues of $C_t$.
Then $\operatorname{tr}(C_t^2) = \sum_{i=1}^D \lambda_i(t)^2$. This yields the following.
\[
d_{\mathrm{eff}}(t) = \frac{1}{\sum_{i=1}^D \lambda_i(t)^2}.
\]
Now if the process at time $t$ spreads its energy equally over exactly $k$ orthogonal directions, then $\lambda_1(t)=\dots=\lambda_k(t)=1/k$ and $\lambda_{k+1}(t)=\dots=0$, and, hence, $d_{\mathrm{eff}}(t) = k$.
However, if most of the energy lies in a single direction, then $\sum_i \lambda_i(t)^2$ is close to $1$ and $d_{\mathrm{eff}}(t)$ is close to $1$.
Thus, $d_{\mathrm{eff}}(t)$ behaves like the ``number of active directions'' in which the process varies at time $t$.
This is especially appropriate as a measure for our assumed generative process, i.e., $x_t = V a_t$, which yields
\[
C_t = \mathbb{E}[x_t x_t^\top]
= \mathbb{E}[V a_t a_t^\top V^\top]
= V S_t V^\top,
\qquad
S_t := \mathbb{E}[a_t a_t^\top].
\]
If $V$ is close to orthonormal and $S_t$ is close to diagonal, with
\[
e_i(t) := \mathbb{E}[a_{t,i}^2]
\]
denoting the average energy of coefficient $i$, then
\[
\operatorname{tr}(C_t^2) \approx \sum_{i=1}^N e_i(t)^2,
\qquad
\sum_{i=1}^N e_i(t) \approx 1.
\]
Hence
\[
d_{\mathrm{eff}}(t)
= \frac{1}{\operatorname{tr}(C_t^2)}
\approx \frac{1}{\sum_{i=1}^N e_i(t)^2},
\]
which is the usual ``effective number'' of coefficients that carry energy
at time $t$.
In particular, if the energy is spread evenly across $k$ dictionary atoms,
then $d_{\mathrm{eff}}(t) \approx k$.
Therefore $d_{\mathrm{eff}}(t)$, estimated from the Gram U-statistic,
is a simple and natural measure of intrinsic dimensionality for this
generative process.

\subsection{Surrogate}
\label{app:surrogate}

For the U-statistic and Autocorrelation metrics, we compare LLM activations to surrogate distributions that preserve certain statistical properties while removing temporal structure. We operate on representation vectors $\mathbf{X} \in \mathbb{R}^{B \times T \times d}$, where $B$ denotes batch size, $T$ denotes sequence length, and $d$ denotes the model dimension.

\textbf{U-statistic surrogate (Fig.~\ref{fig:llm_temporal_structure} a, e):} For each sequence $i \in \{1, \ldots, B\}$, we construct the surrogate $\tilde{\mathbf{X}}_i$ by applying a random permutation $\pi_i: \{1, \ldots, T\} \to \{1, \ldots, T\}$ to the temporal positions:
$$\tilde{\mathbf{X}}_{i,t,:} = \mathbf{X}_{i,\pi_i(t),:} \quad \forall t \in \{1, \ldots, T\}$$
This preserves the marginal distribution of activations within each sequence while destroying temporal dependencies.

\textbf{Autocorrelation surrogate (Fig.~\ref{fig:llm_temporal_structure} c, g):} Given the similarity matrix $\mathbf{S} \in \mathbb{R}^{T \times T}$ where $S_{ij} = \text{sim}(\mathbf{X}_{\cdot,i,\cdot}, \mathbf{X}_{\cdot,j,\cdot})$, we construct the surrogate similarity matrix $\tilde{\mathbf{S}}$ by replacing each diagonal with its mean:
$$\tilde{S}_{ij} = \bar{S}_k \quad \text{where } k = |i-j|, \quad \bar{S}_k = \frac{1}{T-k}\sum_{t=1}^{T-k} S_{t,t+k}$$
This preserves the average correlation structure at each lag while removing position-specific temporal patterns.

\subsection{Analysis of Stories: UMAP Projections, Dendrograms, Similarity Heatmaps, and Event Boundary Detection}

\paragraph{Data.} For experiments in Sec.~\ref{sec:stories}, we used either stories sampled from the TinyStories dataset~\citep{eldan2023tinystories} (for Fig.~\ref{fig:ind_story_geometry},~\ref{fig:fourier_spectra}) or synthetically sampled stories using GPT-5 (for Fig.~\ref{fig:stories}).
For the former, we sampled stories that were up to a 100 tokens long, allowing us to characterize compute UMAP for both Gemma and Llama models without hitting memory bottlenecks. 
For consistency, we then used these stories for all experiments.
For synthetically defined stories, we defined 3 in-context examples (one of which is shown in Fig.~\ref{fig:stories}a) and prompted GPT-5 to produce 50 stories with similar such suddenly changing events.
Stories varied in size, but were generally less than 150 tokens long.
Following \citet{georgiou2023using}, the generated stories' events were labeled using GPT-5. 
These labels served as the `ground truth' event boundaries.
We qualitatively analyzed and confirmed the event boundaries overlap with our intuitively expected event structure in the stories.

\paragraph{Analysis Details.} To isolate the low-dimensional geometry of latent codes, we perform a 3D UMAP analysis using the open-source package~\citep{umap}.
Similarly, Dendrograms in Fig.~\ref{fig:ind_story_geometry} are computed using the hierarchical clustering package in SciPy~\cite{dendrogram}.
Branches are colored based on proximity, which in our case was between $0$--$1$ (since we use cosine similarity as the clustering measure).
We put a proximity bound of $0.2$ as the distance under which branches are colored within a cluster.
For Fourier Analysis, as mentioned in Sec.~\ref{sec:stories}, we performed a Fourier transform of the model activations for a given story, isolated its lower frequency components---defined as $0.1 \times$ the Nyquist rate, which turns out to possess $\sim$50\% of the signal norm---and compute the cosine similarity kernel of the low / high frequency spectra. These kernels are then compared to  kernels of latent codes derived using Temporal Feature Analysis or SAEs.

\subsection{Analyzing Garden Path Sentences: Dendrograms and Phrase Similarity}
\paragraph{Data.} We used GPT-5 to sample 50 garden path sentences and control variants thereof (1 corresponding to each sentence).
Sentences were 20--40 tokens long and had a structure such that the observation of the object phrase resolved ambiguity.
The precise ambiguity structure was were varied in type, i.e., the ambiguity could be resolved by altering the subject (e.g., \texttt{old} $\to$ \texttt{sailors} in \texttt{The old man the road}) or by other mechanisms such as adding punctuations to elicit a pause (e.g., adding a comma, such as \texttt{The old train the young fight} $\to$ \texttt{The old train, the young fight}).
Control variants had a mixture of such resolutions applied.

\paragraph{Analysis Details.} For assessing whether tokens in verb phrase relate more with the subject phrase, as would be expect by the typical parse in our used garden path sentences, or with the object phrase, as would be necessary for the correct parse, we computed Dendrograms using the hierarchical clustering package in SciPy~\citep{dendrogram}.
The similarity between subject phrase (SP), verb phrase (V), and object phrase (OP) was computed by taking the set of tokens $T_{P}$ that belong to a phrase $P \in \{$SP, V, OP$\}$, and computing the average latent code: $\bar{c}(P) = \frac{1}{|P|} \sum_{t \in T_P} c_t$, where $c_t$ denotes the latent code extracted for token $t$ using either SAEs or the predictive or novel component of Temporal Feature Analysis.
This protocol is similar to popularly used strategies for computing sentence embeddings (see, e.g., work using BERT embeddings~\citep{koroteev2021bert}).
Cosine similarity between these phrase-averaged garden path / control sentences yields the tables shown in Fig.~\ref{fig:garden_path}.

\subsection{Analyzing In-Context Representations: Visualization and Counterfactual Experiments}
\paragraph{Data.} We followed the pipeline of \citet{park2025iclrincontextlearningrepresentations} and defined grid graphs where random 1-token words are associated to the nodes of a 3$\times$3 grid. 
A random walk of length 500 tokens is performed on the grid, with the visited nodes stringed together with a whitespace (e.g., stringing \texttt{apple}, \texttt{car}, \texttt{house} yields the string `\texttt{apple car house}'). 
These walks were inputted into our analyzed, pretrained language models (Gemma-2-2B in Fig.~\ref{fig:iclr_iclr} and Llama-3.1-8B in Fig.~\ref{fig:llama_iclr_iclr}) to extract corresponding activations.

\paragraph{Analysis Details.} A 3D PCA of activations was performed to check which two subsequent dimensions (i.e., PC1/2 or PC2/3) best reflect the graph structure, which was then used to plot the figures.
These activations were then passed through both Temporal Feature Analyzers and SAEs, resulting in latent codes that were again visualized using PCA.

\clearpage
\section{Details on Training Temporal Feature Analyzers and SAEs}

\subsection{Training, Harvesting, and Inference for Temporal Feature Analysis and SAEs}
\paragraph{Compute.} All experiments were performed using 1 H100, both for extracting representations, analyzing them, and training SAEs and Temporal Feature Analyzers.
To save training costs, we precache activations on a local server and train in an ``online manner'', i.e., sampling a new batch of activations every training iteration.

\paragraph{SAEs' training.} All analyzed SAEs were trained from scratch on 1B precached activations from the Pile-Uncopyrighted dataset~\citep{monology2021pile-uncopyrighted}.
While we did try to use existing, off-the-shelf SAEs, we found the results to be wildly inconsistent depending on where we borrowed the SAE from. To enable a consistent and fair evaluation, we thus preferred to train all SAEs from scratch.
For Gemma models, activations were extracted from Layer 12; for Llama models, from layer 15 (all 0-indexed).

\paragraph{Training procedure.} We use Adam optimizer with standard hyperparameters. We initialize training with a warmup of 200 steps to a learning rate of $10^{-3}$, following from thereon to a minimum of $9\times 10^{-4}$, i.e., the learning rate remains essentially constant throughout training.

\paragraph{Activations normalization.} Following best practices, we normalize activations to unit \textit{expected} norm~\citep{bricken2023monosemanticity, costa2025flat}. This helps put different SAEs on the same training scale, and especially for ReLU SAEs, makes training substantially easier by reducing the need for tuning regularization strength. On a related note, we emphasize our ReLU SAEs have a slightly higher $L_0$, i.e., are less sparse, than almost all other SAEs analyzed in this work.

\clearpage
\section{Further Results: Language Model Representations (U-Statistic and Autocorrelation)}
\vspace{10pt}

\subsection{U-Statistic across domains for Gemma-2-2B}

\begin{figure}[htb!]
    \centering
    \includegraphics[width=1\linewidth]{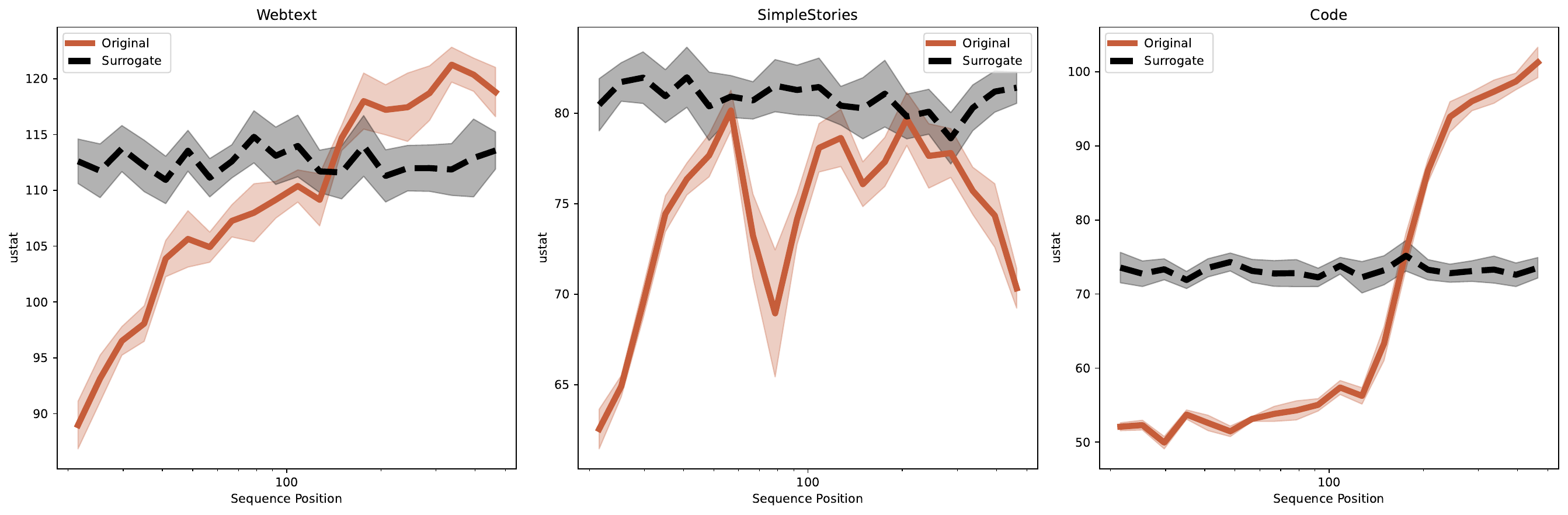}
    \caption{U-Statistic across domains for LLM activations, surrogate}
    \label{fig:ustat_gemma_alldata}
    \vspace{5pt}
\end{figure}

\subsection{Autocorrelation across domains for Gemma-2-2B}

\begin{figure}[htb!]
    \centering
    \includegraphics[width=1\linewidth]{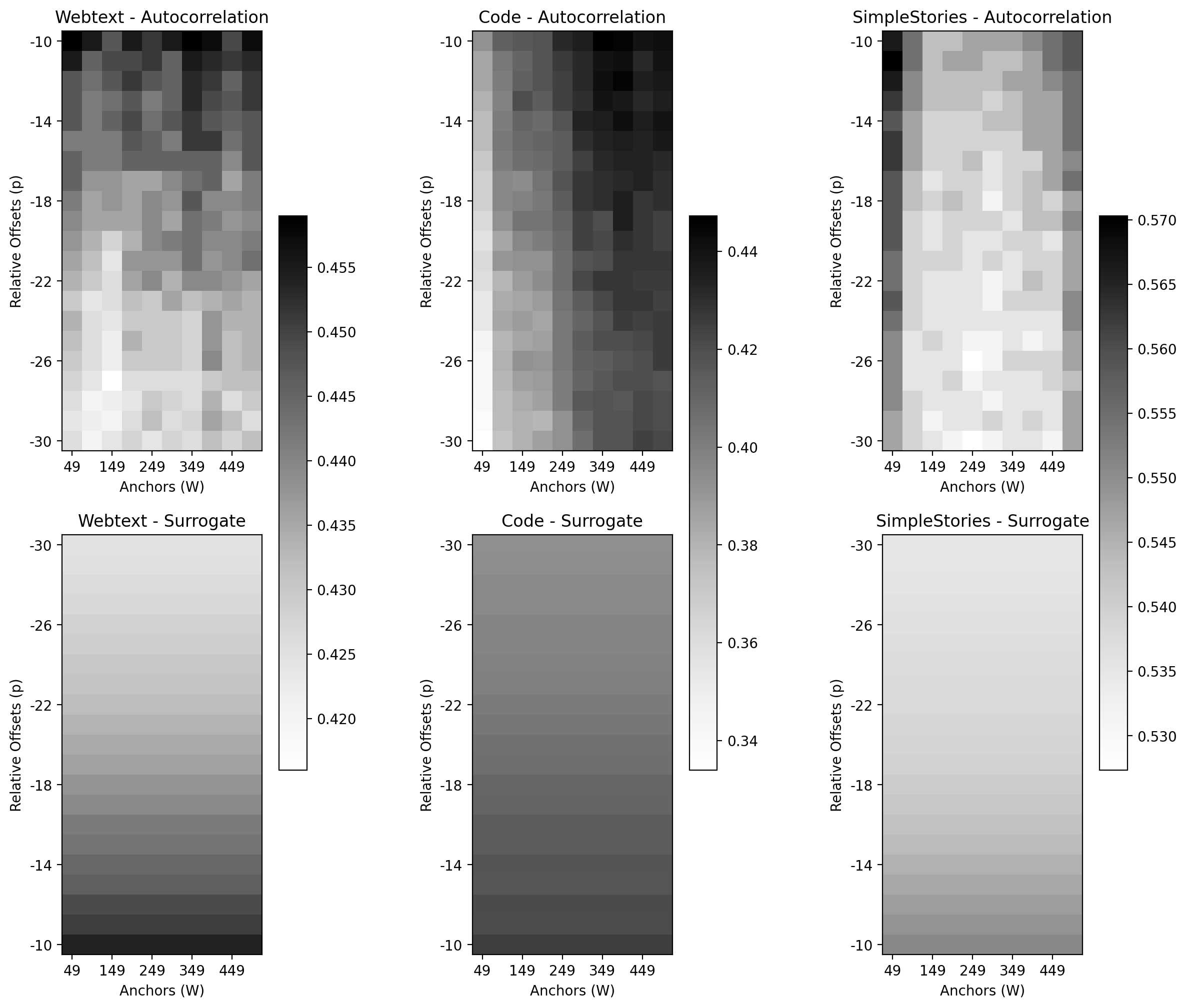}
    \caption{Autocorrelation of language model activations and a stationary surrogate across webtext, simple stories and code domains.}
    \label{fig:autocorr_gemma_alldata}
\end{figure}

\clearpage
\section{More Gemma Results}
\label{appsection:gemma_results}

\subsection{Hierarchical clustering of codes from story tokens}

\subsubsection{Story 1}
\vspace{10pt}

\begin{figure}[htb!]
    \centering
    \includegraphics[width=\linewidth]{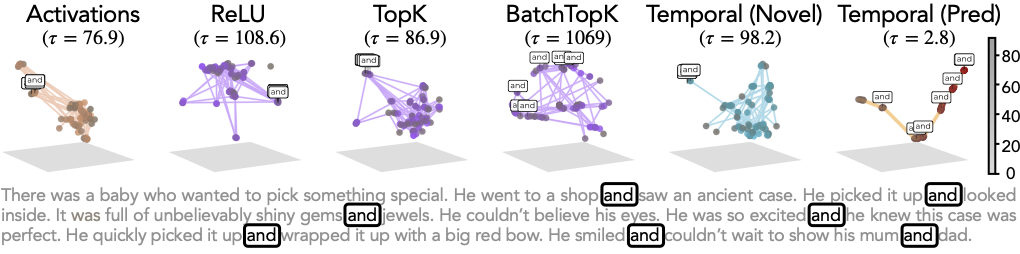}
    \caption{\textbf{Geometry.} Repeating results of Fig.~\ref{fig:ind_story_geometry}, we again find smooth trajectories in UMAP projections for Temporal Feature Analysis.
    \vspace{10pt}
    }
    \label{fig:story_geometry_1}
\end{figure}

\begin{figure}[htb!]
    \centering
    \includegraphics[width=\linewidth]{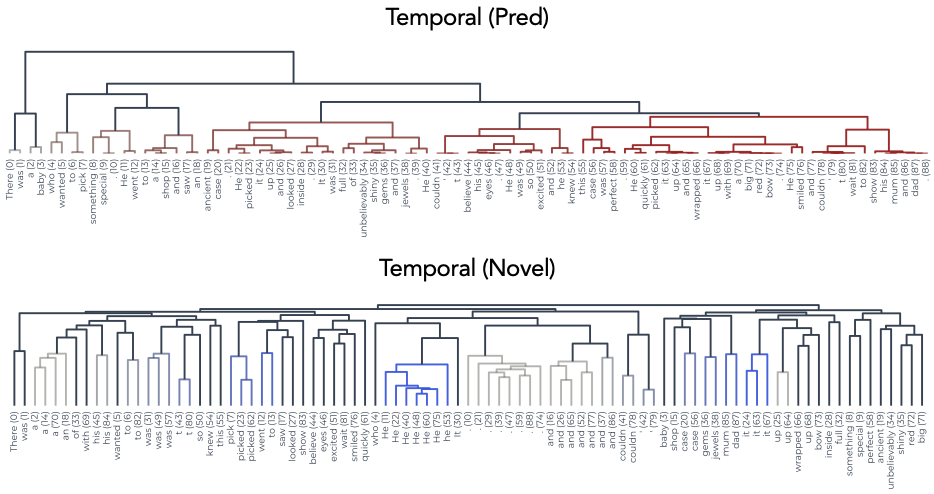}
    \caption{\textbf{Dendrograms of Predictive and Novel Components from Temporal Feature Analysis.} Reproduction of the results from Fig.~\ref{fig:ind_story_geometry}b, but including the novel component's dendrogram as well.
    \vspace{10pt}
    }
    \label{fig:story_temporal_dendro_1}
\end{figure}

\begin{figure}[htb!]
    \centering
    \includegraphics[width=\linewidth]{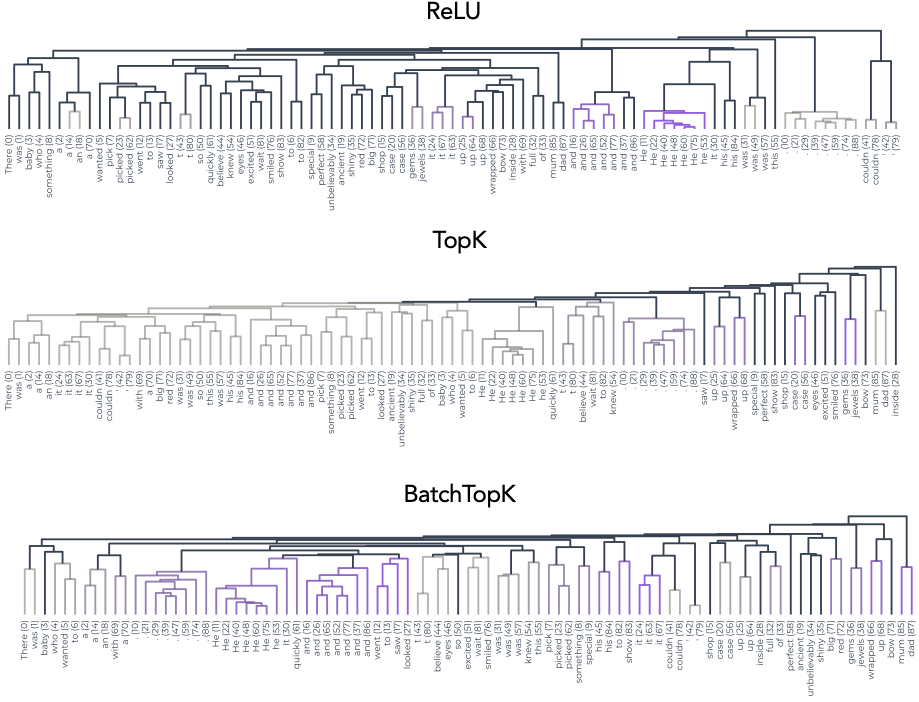}
    \caption{\textbf{Dendrograms of Standard SAEs' Latents Codes.} Reproduction of the results from Fig.~\ref{fig:ind_story_geometry}b, but for latent codes extracted using standard SAEs.
    \vspace{10pt}
    }
    \label{fig:story_standard_dendro_1}
\end{figure}

\clearpage
\subsubsection{Story 2}
\vspace{10pt}

\begin{figure}[htb!]
    \centering
    \includegraphics[width=\linewidth]{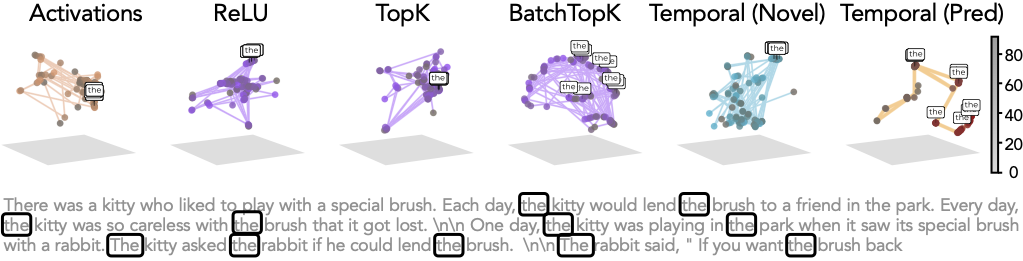}
    \caption{\textbf{Geometry.} Repeating results of Fig.~\ref{fig:ind_story_geometry} on a different story, we again find smooth trajectories in UMAP projections for Temporal Feature Analysis.
    \vspace{10pt}
    }
    \label{fig:story_geometry_2}
\end{figure}

\begin{figure}[htb!]
    \centering
    \includegraphics[width=\linewidth]{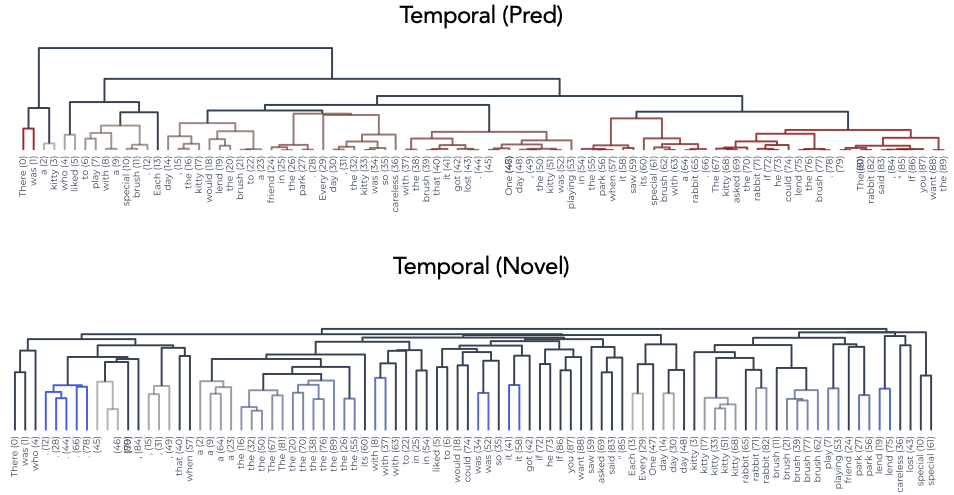}
    \caption{\textbf{Dendrograms of Predictive and Novel Components from Temporal Feature Analysis.} Reproduction of the results from Fig.~\ref{fig:ind_story_geometry}b, but on a different story and with both the predictive and novel component.
    \vspace{10pt}
    }
    \label{fig:story_temporal_dendro_2}
\end{figure}

\begin{figure}[htb!]
    \centering
    \includegraphics[width=\linewidth]{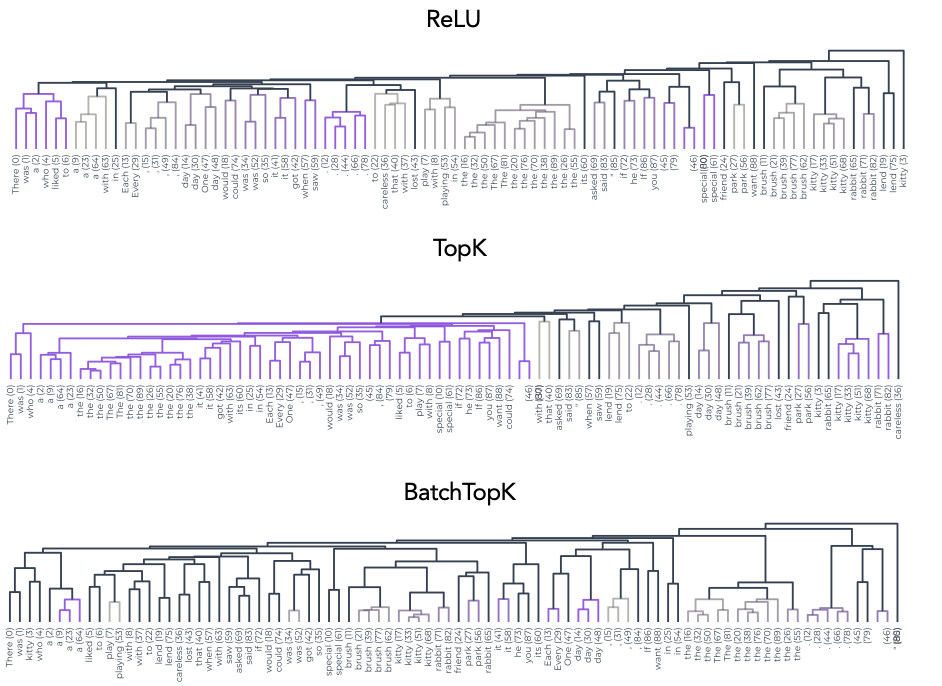}
    \caption{\textbf{Dendrograms of Standard SAEs' Latents Codes.} Reproduction of the results from Fig.~\ref{fig:ind_story_geometry}b, but for latent codes extracted using standard SAEs on a different story.
    \vspace{10pt}
    }
    \label{fig:story_standard_dendro_2}
\end{figure}

\clearpage
\subsubsection{Story 3}
\vspace{10pt}

\begin{figure}[htb!]
    \centering
    \includegraphics[width=\linewidth]{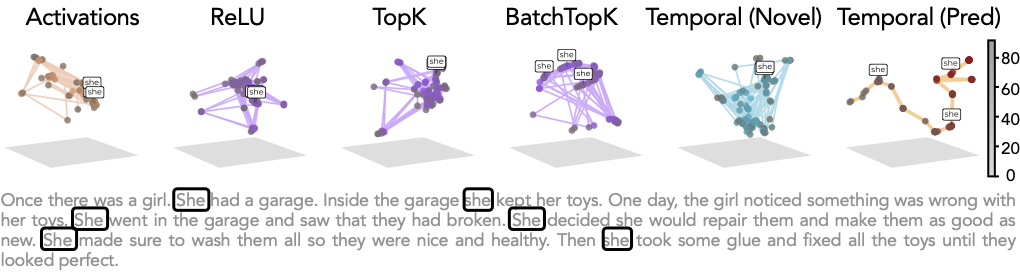}
    \caption{\textbf{Geometry.} Repeating results of Fig.~\ref{fig:ind_story_geometry} on a different story, we again find smooth trajectories in UMAP projections for Temporal Feature Analysis.
    \vspace{10pt}
    }
    \label{fig:story_geometry_3}
\end{figure}

\begin{figure}[htb!]
    \centering
    \includegraphics[width=\linewidth]{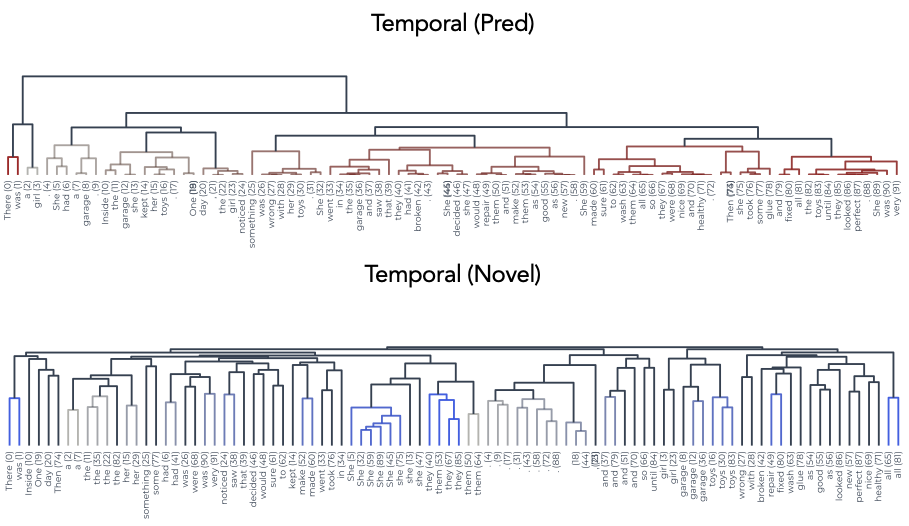}
    \caption{\textbf{Dendrograms of Predictive and Novel Components from Temporal Feature Analysis.} Reproduction of the results from Fig.~\ref{fig:ind_story_geometry}b, but on a different story and with both the predictive and novel component.
    \vspace{10pt}
    }
    \label{fig:story_temporal_dendro_3}
\end{figure}

\begin{figure}[htb!]
    \centering
    \includegraphics[width=\linewidth]{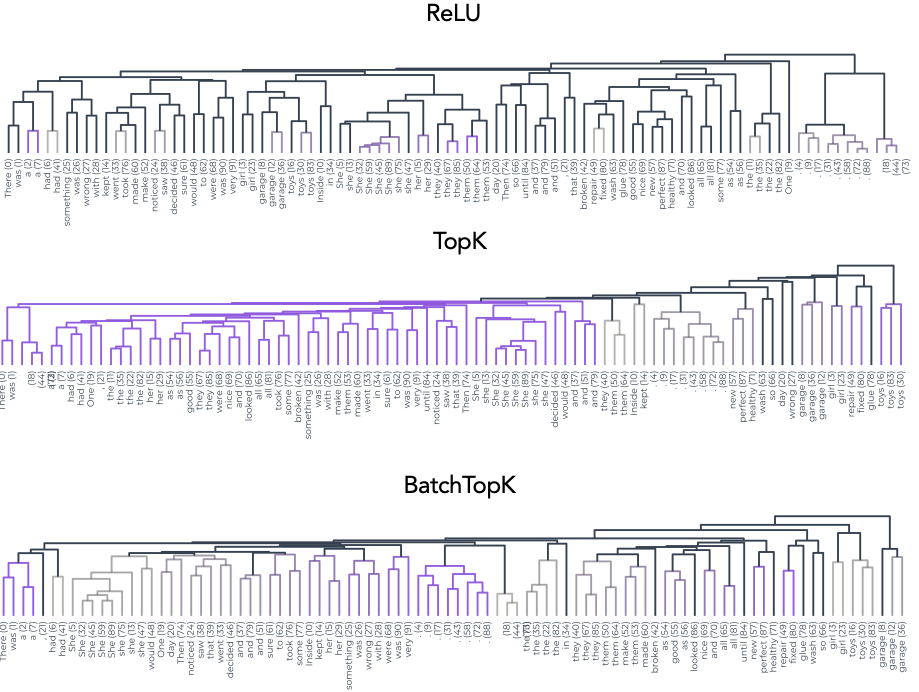}
    \caption{\textbf{Dendrograms of Standard SAEs' Latents Codes.} Reproduction of the results from Fig.~\ref{fig:ind_story_geometry}b, but for latent codes extracted using standard SAEs on a different story.
    \vspace{10pt}
    }
    \label{fig:story_standard_dendro_3}
\end{figure}

\clearpage
\subsection{Code: Analyzing Another Domain}
\vspace{10pt}

\begin{figure}[htb!]
    \centering
    \includegraphics[width=\linewidth]{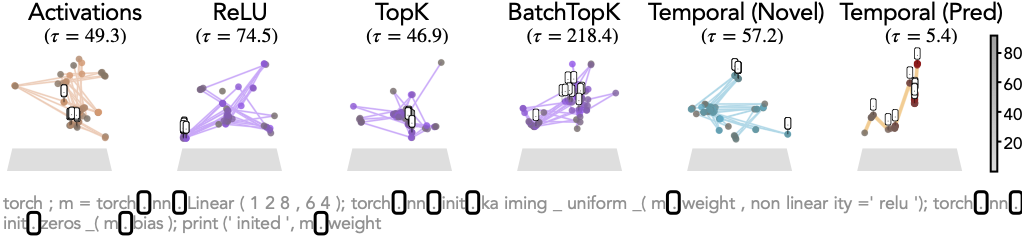}
    \caption{\textbf{Code.} Reproducing the results of Fig.~\ref{fig:ind_story_geometry}, but on a different domain, i.e., code. We again see a temporally disentangled, smoothly running trajectory for latent codes extracted using the predictive component of Temporal Feature Analysis. 
    \vspace{10pt}
    }
    \label{fig:app_code_geometry}
\end{figure}

\begin{figure}[htb!]
    \centering
    \includegraphics[width=\linewidth]{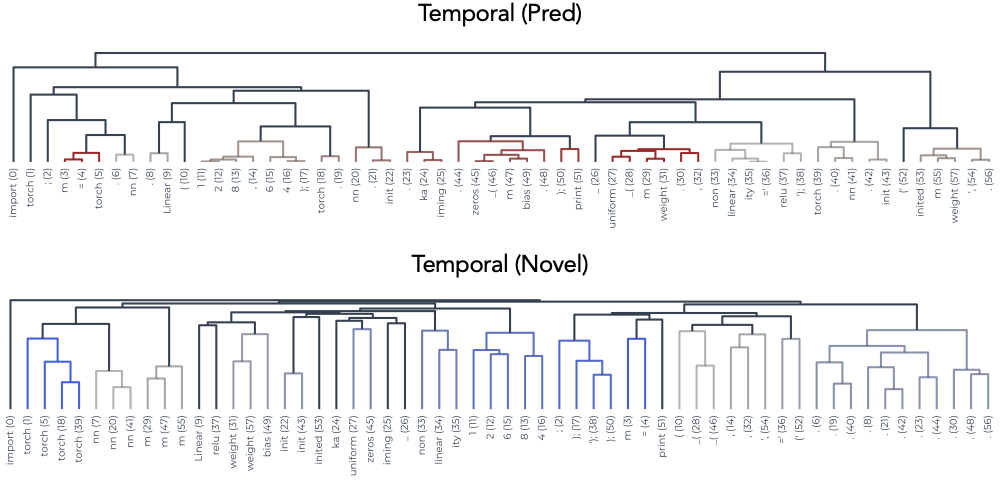}
    \caption{\textbf{Dendrograms of Predictive and Novel Components from Temporal Feature Analysis.} Reproduction of the results from Fig.~\ref{fig:ind_story_geometry}b, but on a different domain shows similar results as with narrative-driven text (stories) for both the predictive and novel component.
    \vspace{10pt}
    }
    \label{fig:app_code_temporal}
\end{figure}

\begin{figure}[htb!]
    \centering
    \includegraphics[width=\linewidth]{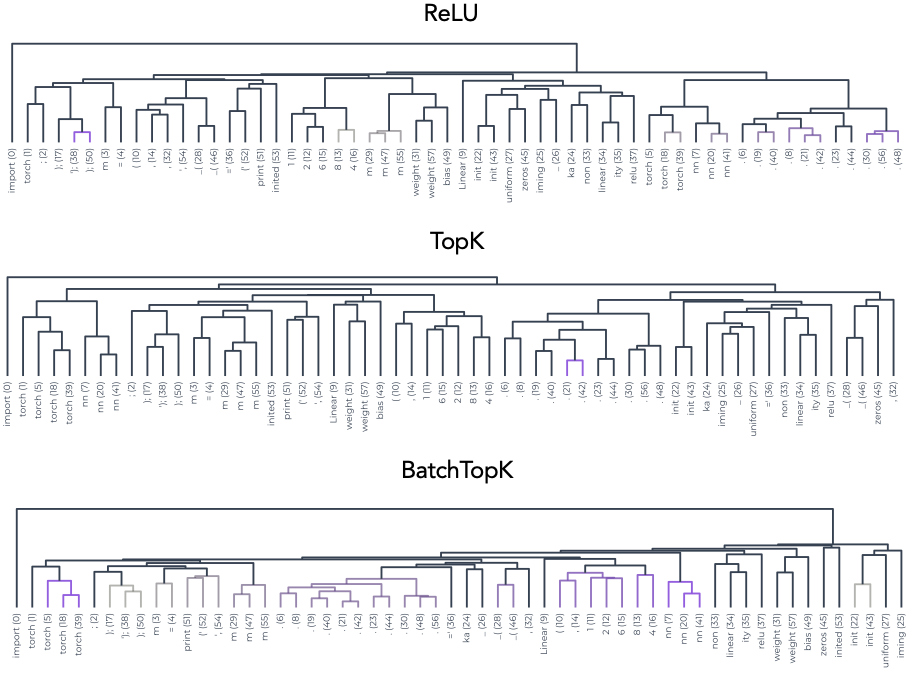}
    \caption{\textbf{Dendrograms of Standard SAE Latent Codes.} Reproduction of the results from Fig.~\ref{fig:ind_story_geometry}b, but on a different domain shows similar results as with narrative-driven text (stories) for standard SAEs.
    \vspace{10pt}
    }
    \label{fig:code_standard_story_graphs}
\end{figure}

\clearpage
\subsection{Further Results: Similarity Maps on Stories}

\begin{figure}[htb!]
    \centering
    \includegraphics[width=\linewidth]{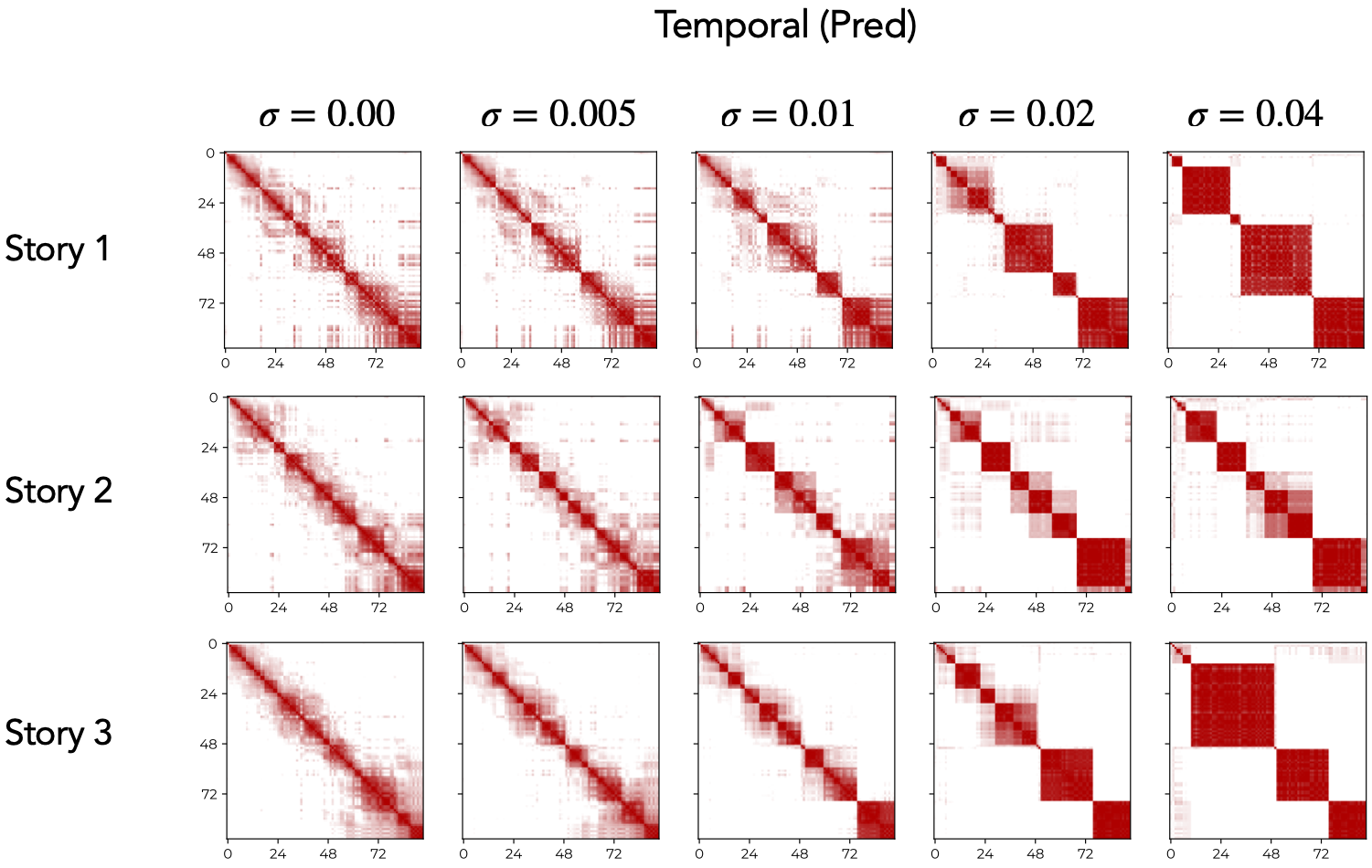}
    \caption{\textbf{Temporal (Pred) Similarity Maps Elicit Multi-Scale Structure with Noise.} Repeating the analysis shown in Fig.~\ref{fig:stories}, we find the coarsening of temporal blocks is a robust result that continues to hold for different inputs. 
    \vspace{10pt}
    }
    \label{fig:temporal_sim_and_noise}
\end{figure}

\begin{figure}[htb!]
    \centering
    \includegraphics[width=\linewidth]{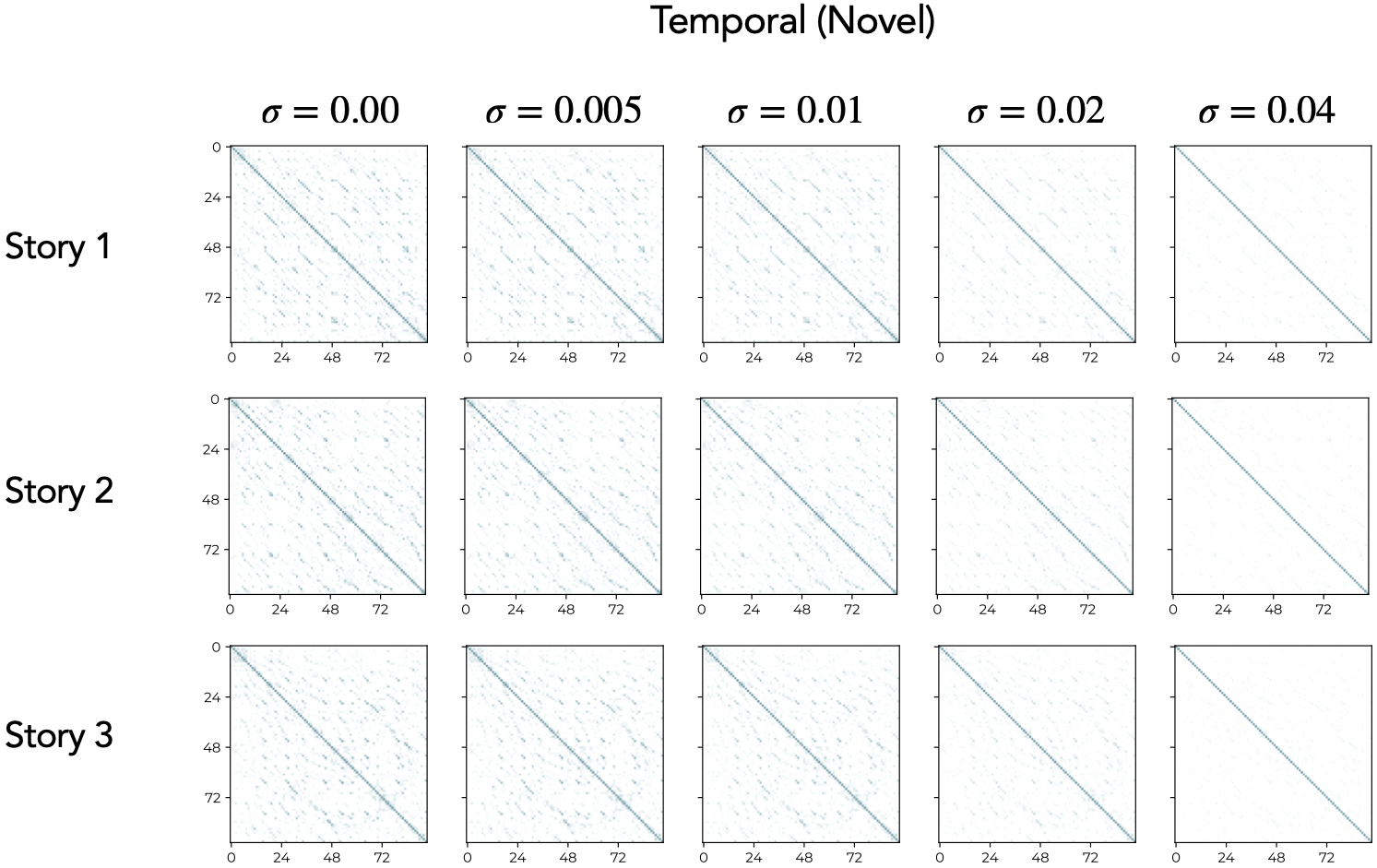}
    \caption{\textbf{Temporal (Novel) Similarity Maps under Noise.} Repeating the analysis shown in Fig.~\ref{fig:stories}, we find the novel component is only able to capture minimal local similarities, which when analyzed via dendrograms, show clustering based on lexical information.
    \vspace{10pt}
    }
    \label{fig:novel_sim_and_noise}
\end{figure}

\begin{figure}[htb!]
    \centering
    \includegraphics[width=\linewidth]{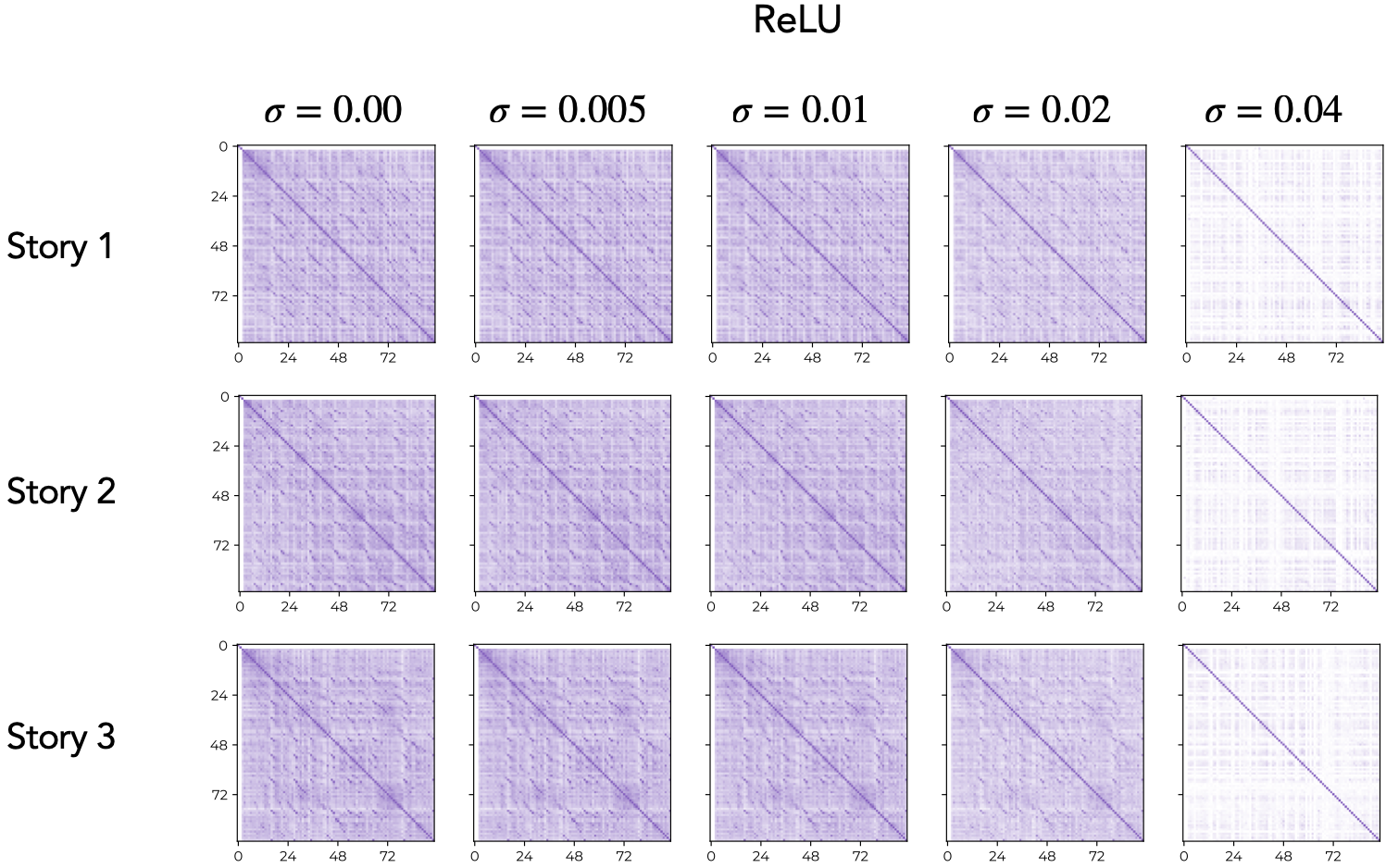}
    \caption{\textbf{ReLU Codes' Similarity Maps under Noise.} Repeating the analysis shown in Fig.~\ref{fig:stories} on ReLU SAEs, we find the ReLU latent code has high similarity across the board, suggesting lack of meaningful temporal information. This similarity is entirely removed when noise scale increases too much, which, as shown in Fig.~\ref{fig:stories}c, corresponds to the point that variance explained by ReLU SAEs drops to $\sim$0.
    \vspace{10pt}
    }
    \label{fig:relu_sim_and_noise}
\end{figure}

\begin{figure}[htb!]
    \centering
    \includegraphics[width=\linewidth]{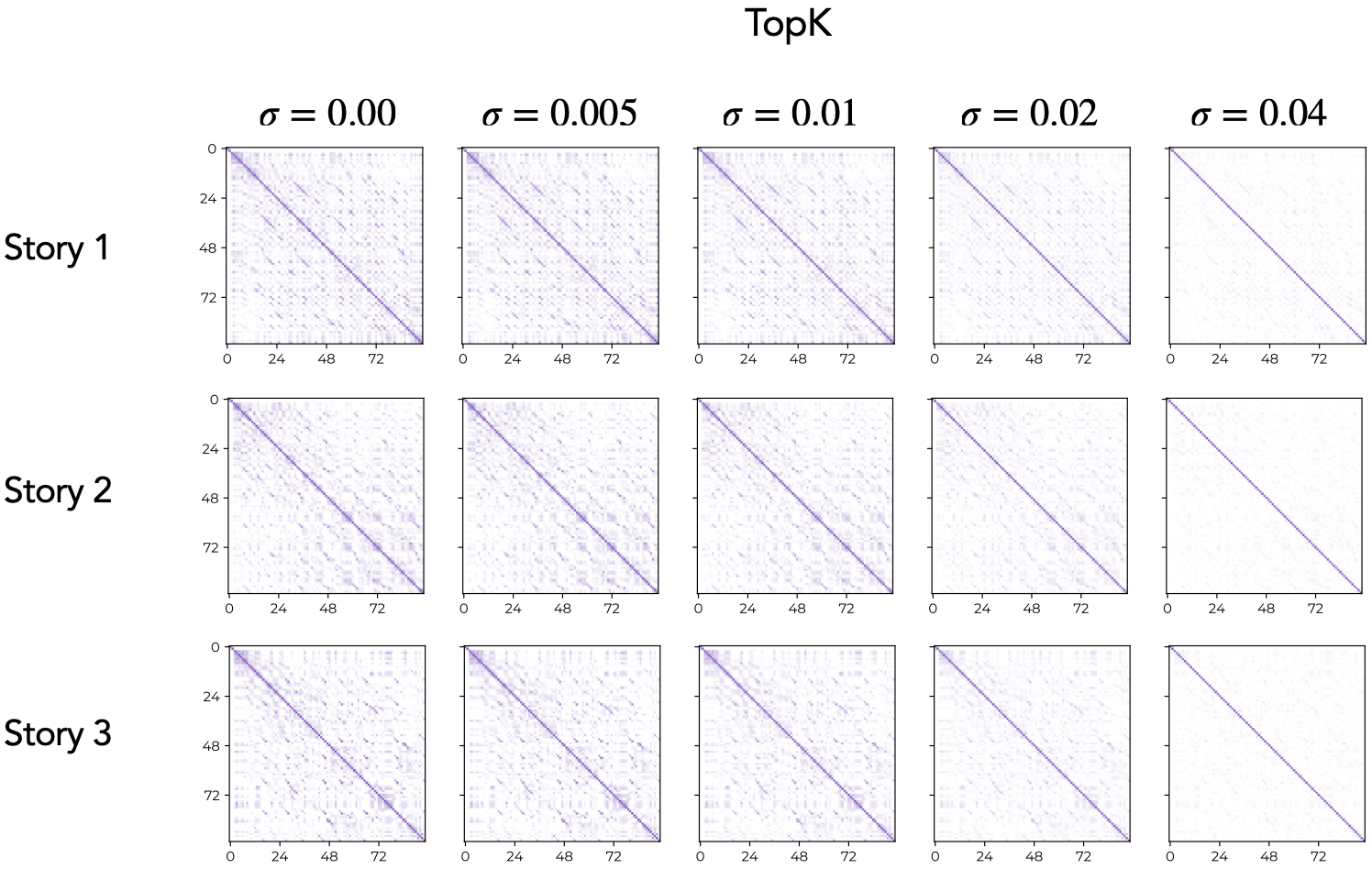}
    \caption{\textbf{TopK Codes' Similarity Maps under Noise.} Repeating the analysis shown in Fig.~\ref{fig:stories}, we find TopK SAE's latent codes are only able to capture minimal local similarities, which when analyzed via dendrograms, show clustering based on lexical information. Akin to ReLU SAEs, we see this similarity map approximately turns into an identity matrix when the noise scale increases too much, which, as shown in Fig.~\ref{fig:stories}c, corresponds to the point that variance explained by TopK SAEs drops to $\sim$0.
    \vspace{10pt}
    }
    \label{fig:topk_sim_and_noise}
\end{figure}

\begin{figure}[htb!]
    \centering
    \includegraphics[width=\linewidth]{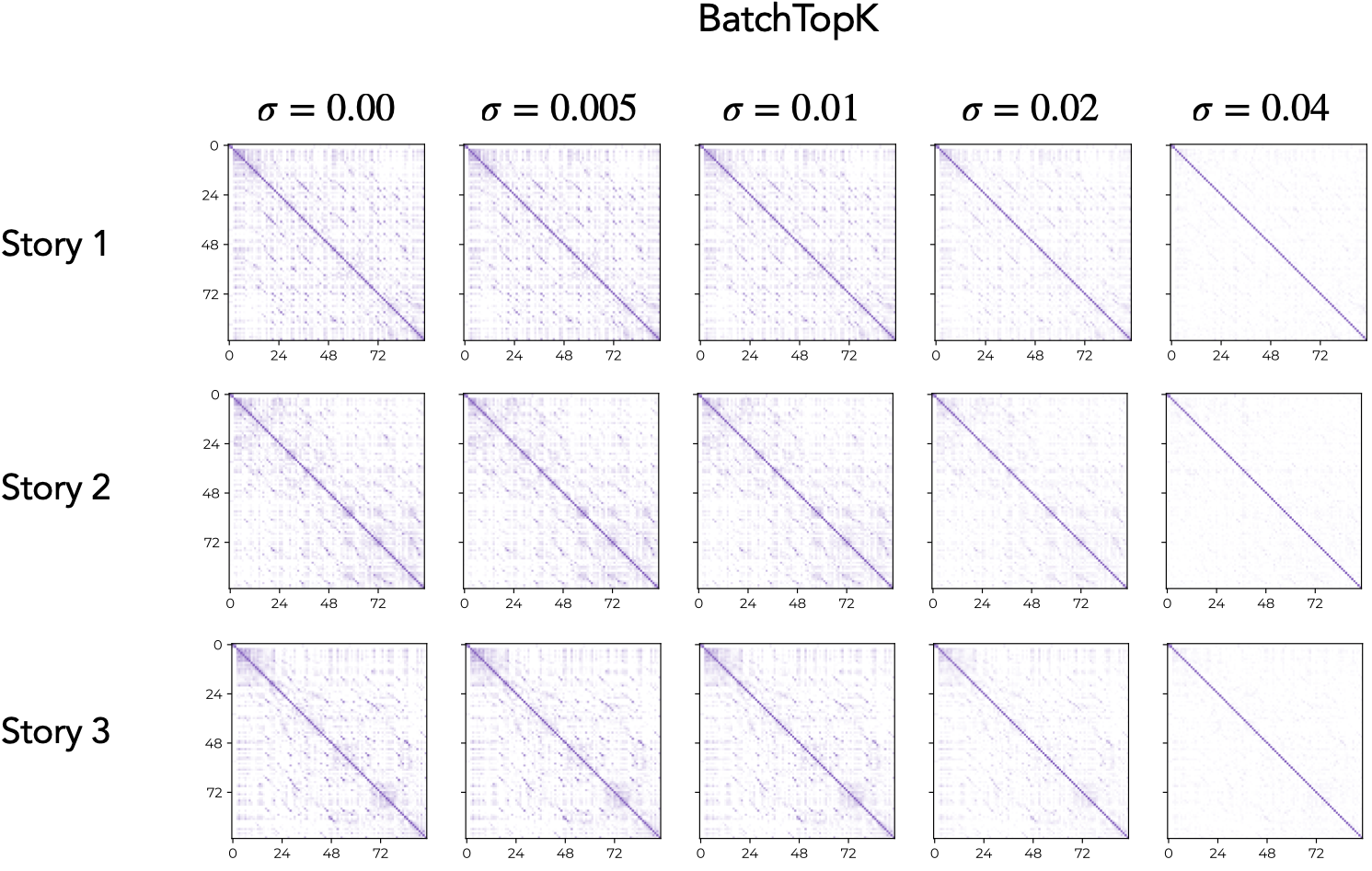}
    \caption{\textbf{BatchTopK Codes' Similarity Maps under Noise.} Repeating the analysis shown in Fig.~\ref{fig:stories}, we find BatchTopK SAE's latent codes are only able to capture minimal local similarities, which when analyzed via dendrograms, show clustering based on lexical information. Akin to ReLU SAEs, we see this similarity map approximately turns into an identity matrix when the noise scale increases too much, which, as shown in Fig.~\ref{fig:stories}c, corresponds to the point that variance explained by BatchTopK SAEs drops to $\sim$0.
    \vspace{10pt}
    }
    \label{fig:batchtop_sim_and_noise}
\end{figure}

\clearpage
\subsection{Further Results: Dendrograms and Evaluations on Garden Path Sentences}
\label{app:gemma_gp}

\subsubsection{Dendrograms}

\begin{figure}[htb!]
    \centering
    \includegraphics[width=\linewidth]{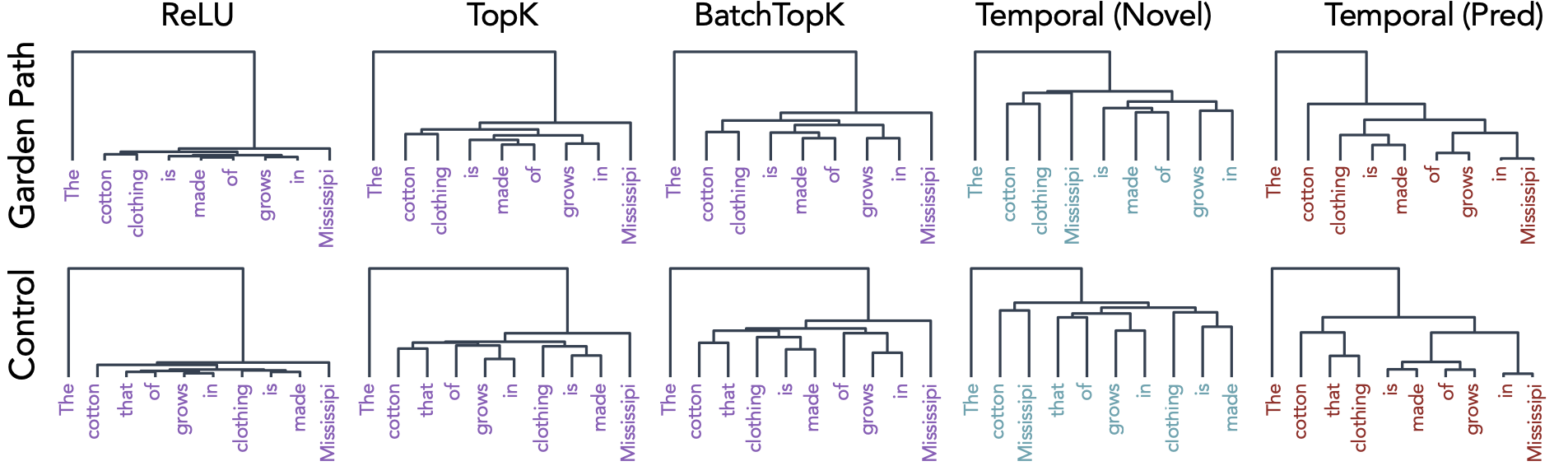}
    \caption{\textbf{Example Sentence 1.} Repeating the results of Fig.~\ref{fig:garden_path}a on a different garden path sentence.
    \vspace{10pt}
    }
    \label{fig:gemma_gp_dendrograms_base}
\end{figure}

\begin{figure}[htb!]
    \centering
    \includegraphics[width=\linewidth]{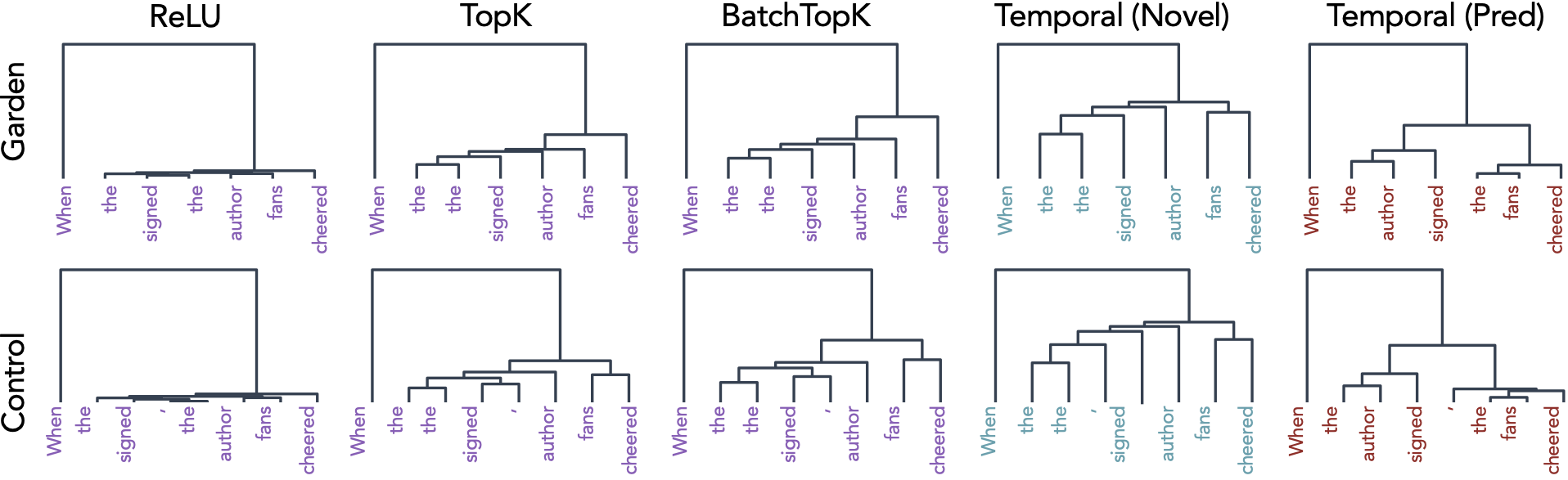}
    \caption{\textbf{Example Sentence 2.} Repeating the results of Fig.~\ref{fig:garden_path}a on a different garden path sentence.
    \vspace{10pt}
    }
    \label{fig:gemma_gp_dendrograms_comma}
\end{figure}

\begin{figure}[htb!]
    \centering
    \includegraphics[width=\linewidth]{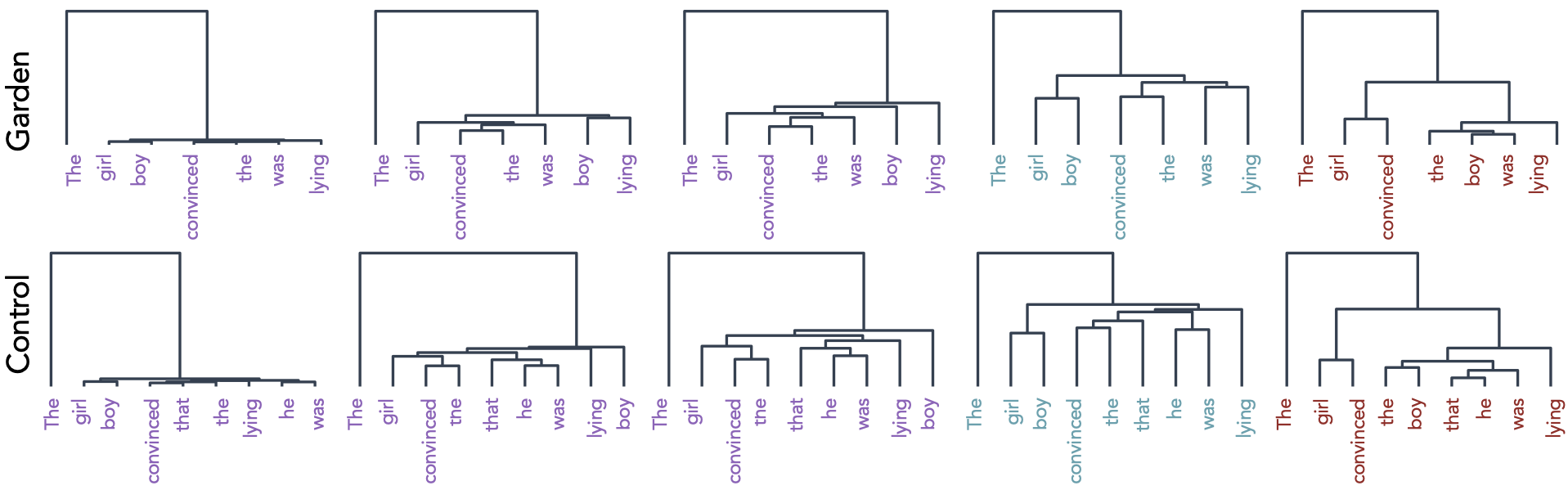}
    \caption{\textbf{Example Sentence 3.} Repeating the results of Fig.~\ref{fig:garden_path}a on a different garden path sentence.
    \vspace{10pt}
    }
    \label{fig:gemma_gp_dendrograms_that}
\end{figure}

\clearpage
\subsection{Further Results: Counterfactual Generation with In-Context Representations}

\begin{figure}[htb!]
    \centering
    \includegraphics[width=\linewidth]{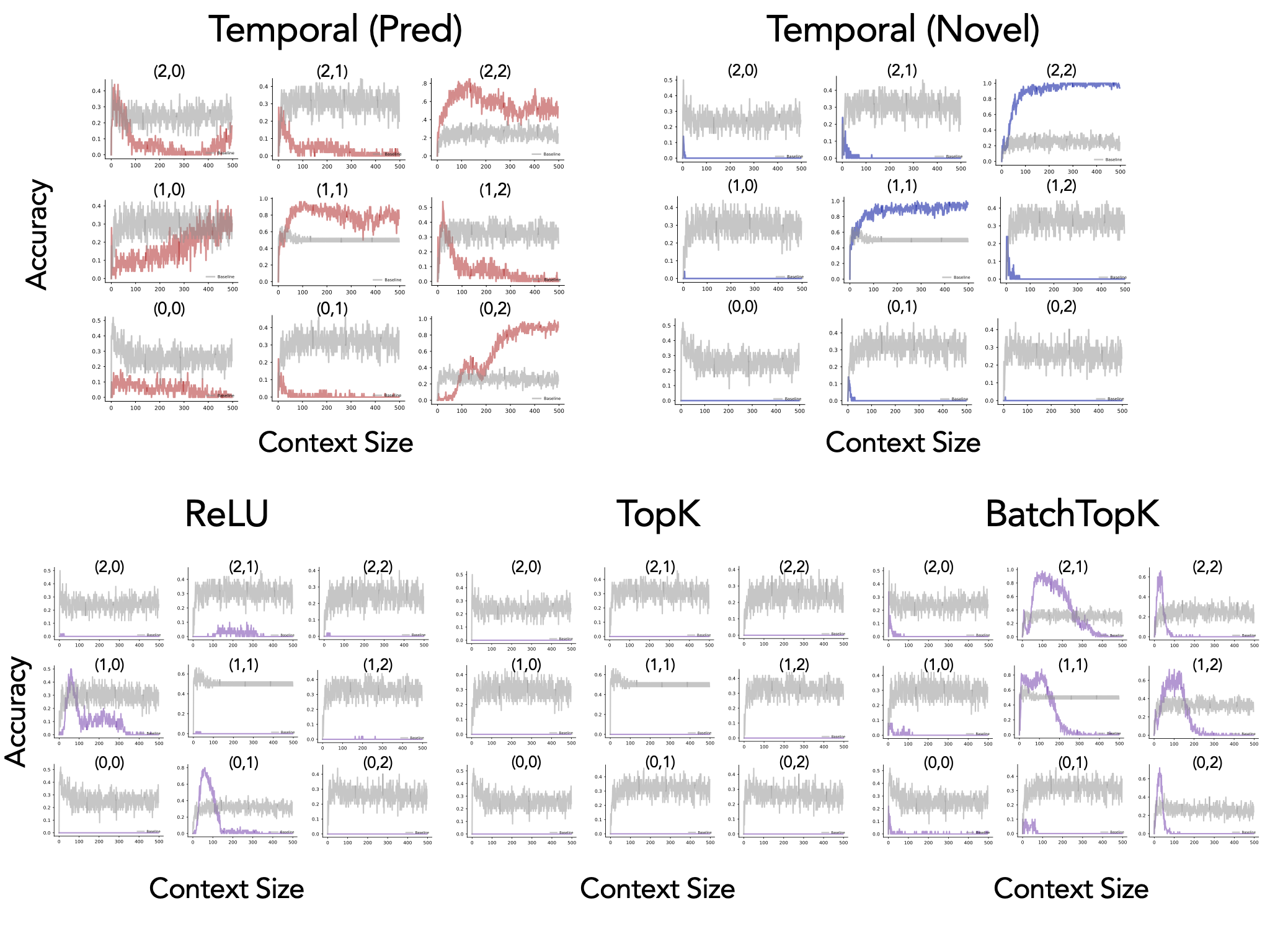}
    \caption{\textbf{Counterfactuals.} Counterfactual analysis of performing belief update intervention for all latent codes. We see essentially no SAE offers the ability to alter model beliefs, i.e., altering where on the graph it currently believes it is and hence what next tokens it should produce. Temporal Feature Analysis shows success in this case: we find $\nicefrac{4}{9}$ nodes can be intervened upon using the predictive component, out of which, interestingly, 2 also respond to an intervention on the novel component.
    \vspace{10pt}
    }
    \label{fig:all_pred_counterfactuals}
\end{figure}

\clearpage
\section{Further Results: Temporal Feature Analysis on LLama-3.1-8B}

We replicate a subset of the experiments evaluating Temporal Feature Analysis on Llama-3.1-8B. The training protocol remains the same as that of Gemma-2-2B model: train on 1B token activations with similar normalization schema, but activations are harvested now from Layer 15 (i.e., at $\sim$50\% of model depth). We note that the trained SAEs are not finished training, and hence the following results are solely meant to be an impression of whether the qualitative trends observed with Gemma models generalize to another model class. 

\subsection{Geometry, Dendrograms, and Spectra}
\subsubsection{Story 1}
\vspace{10pt}

\begin{figure}[htb!]
    \centering
    \includegraphics[width=\linewidth]{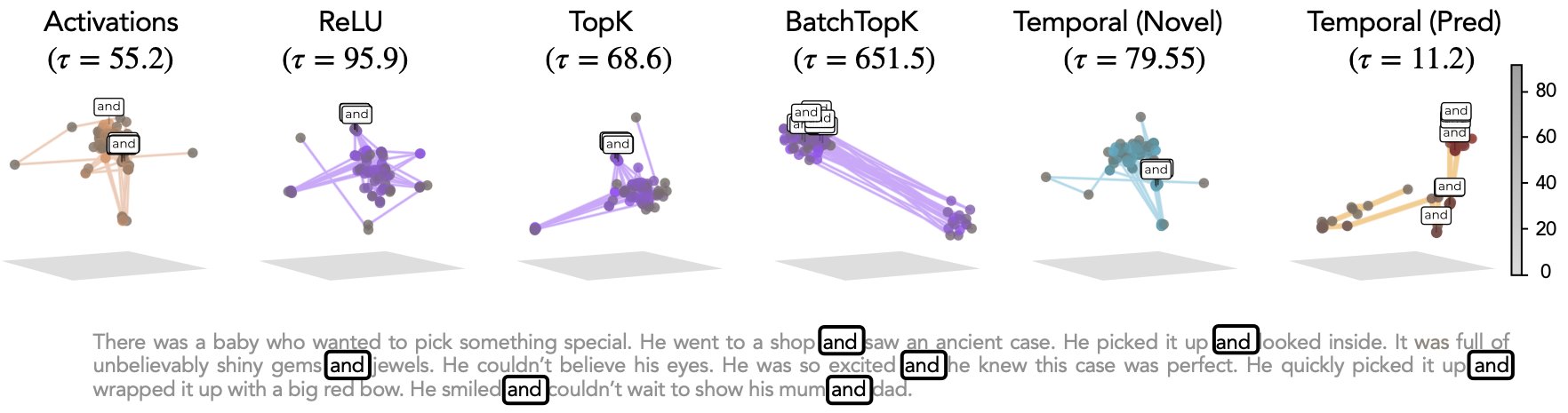}
    \caption{\textbf{Geometry.} Repeating results of Fig.~\ref{fig:ind_story_geometry}, we again find smooth trajectories in UMAP projections for Temporal Feature Analysis.
    \vspace{10pt}
    }
    \label{fig:llama_story_geometry_1}
\end{figure}

\begin{figure}[htb!]
    \centering
    \includegraphics[width=\linewidth]{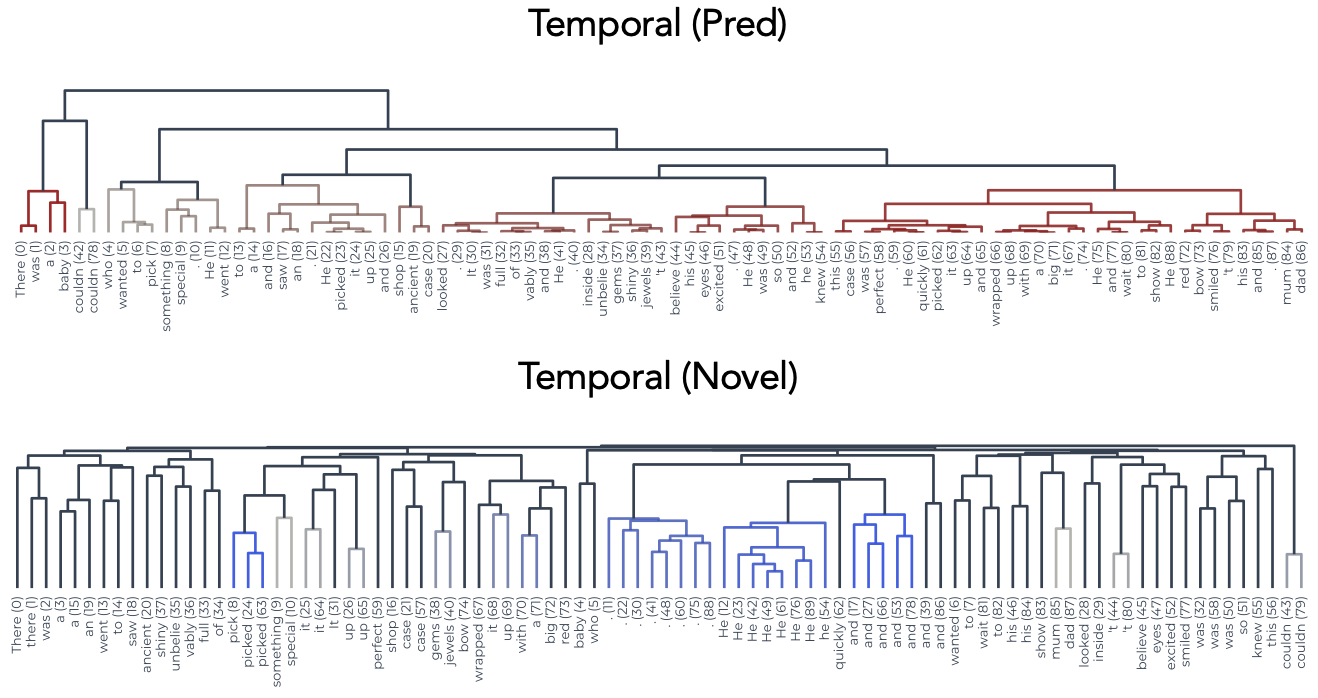}
    \caption{\textbf{Dendrograms of Predictive and Novel Components from Temporal Feature Analysis.} Reproduction of the results from Fig.~\ref{fig:ind_story_geometry}b, but including the novel component's dendrogram as well.
    \vspace{10pt}
    }
    \label{fig:llama_story_temporal_dendro_1}
\end{figure}

\begin{figure}[htb!]
    \centering
    \includegraphics[width=\linewidth]{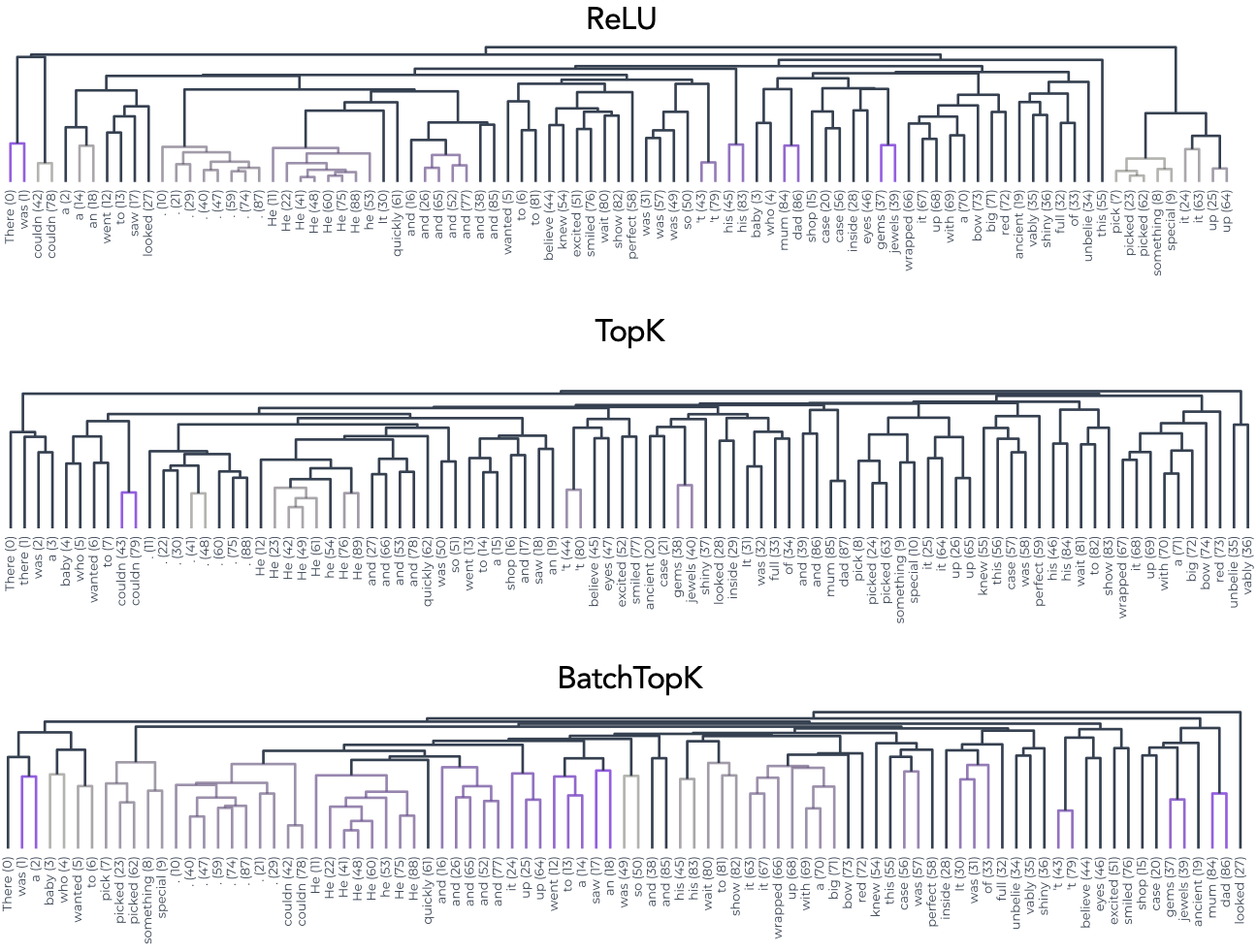}
    \caption{\textbf{Dendrograms of Standard SAEs' Latents Codes.} Reproduction of the results from Fig.~\ref{fig:ind_story_geometry}b, but for latent codes extracted using standard SAEs.
    \vspace{10pt}
    }
    \label{fig:llama_story_standard_dendro_1}
\end{figure}

\clearpage
\subsubsection{Story 2}
\vspace{10pt}

\begin{figure}[htb!]
    \centering
    \includegraphics[width=\linewidth]{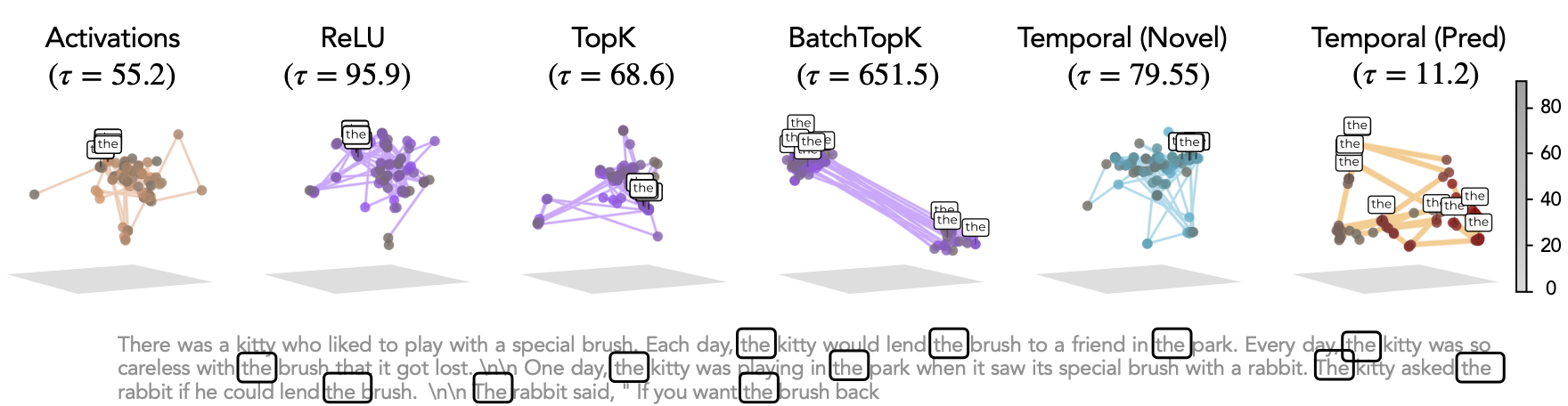}
    \caption{\textbf{Geometry.} Repeating results of Fig.~\ref{fig:ind_story_geometry}, we again find smooth trajectories in UMAP projections for Temporal Feature Analysis.
    \vspace{10pt}
    }
    \label{fig:llama_story_geometry_2}
\end{figure}

\begin{figure}[htb!]
    \centering
    \includegraphics[width=\linewidth]{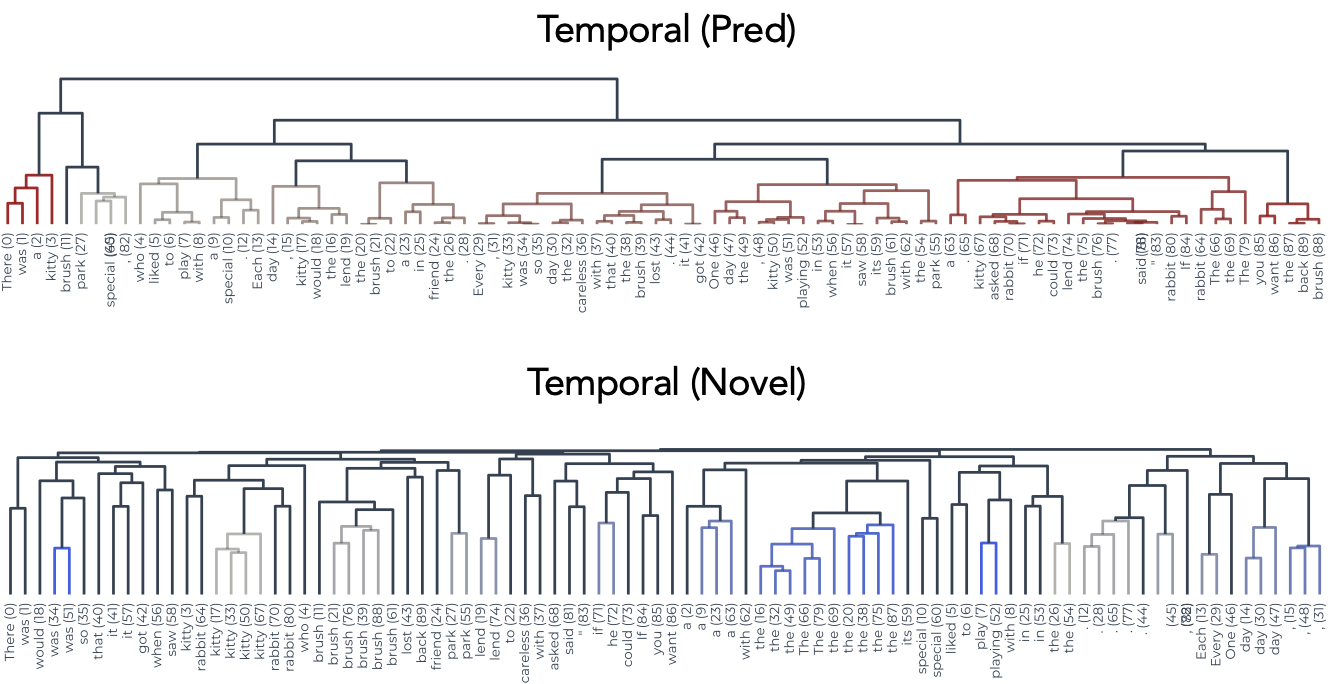}
    \caption{\textbf{Dendrograms of Predictive and Novel Components from Temporal Feature Analysis.} Reproduction of the results from Fig.~\ref{fig:ind_story_geometry}b, but including the novel component's dendrogram as well.
    \vspace{10pt}
    }
    \label{fig:llama_story_temporal_dendro_2}
\end{figure}

\begin{figure}[htb!]
    \centering
    \includegraphics[width=\linewidth]{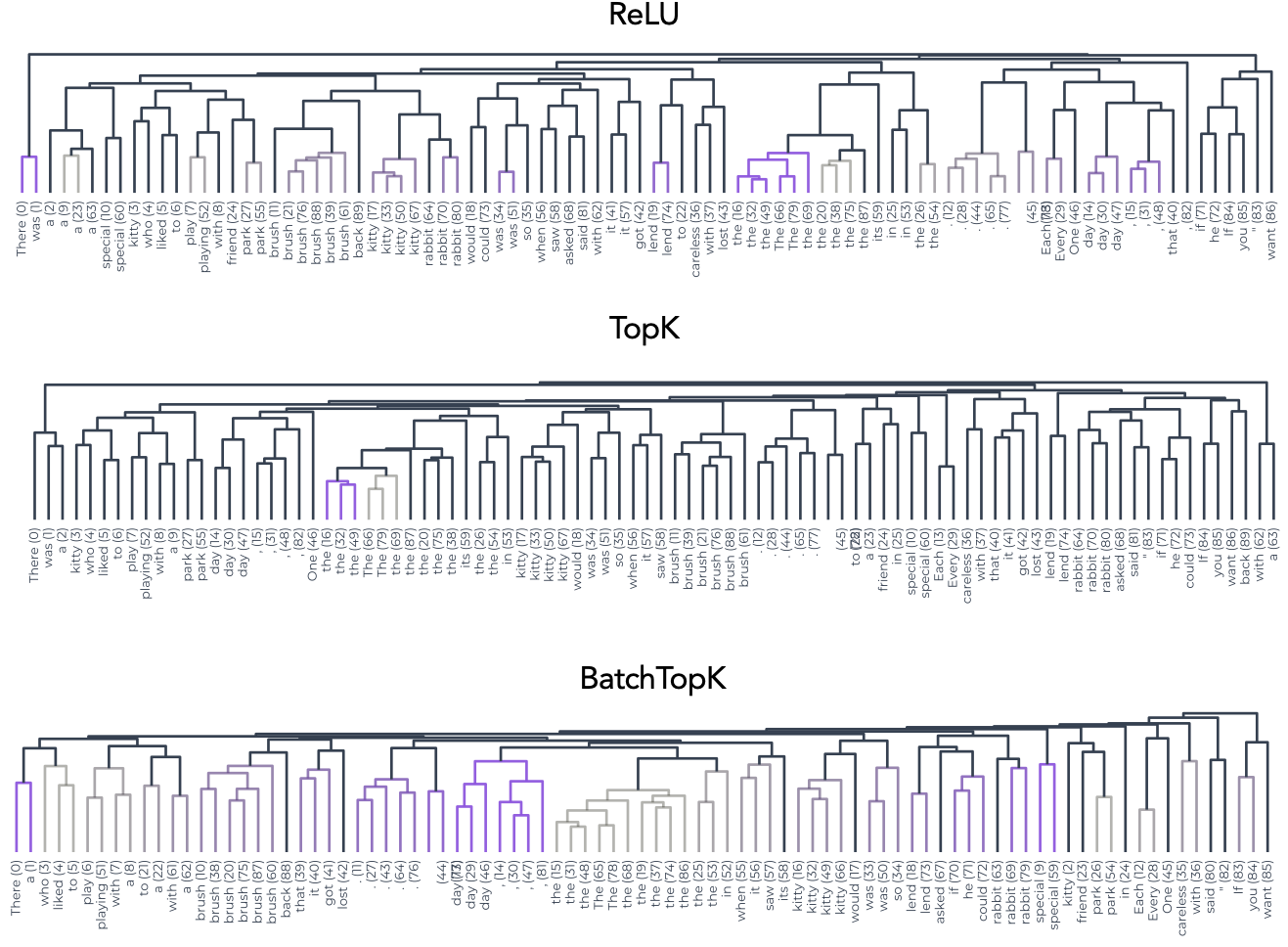}
    \caption{\textbf{Dendrograms of Standard SAEs' Latents Codes.} Reproduction of the results from Fig.~\ref{fig:ind_story_geometry}b, but for latent codes extracted using standard SAEs.
    \vspace{10pt}
    }
    \label{fig:llama_story_standard_dendro_2}
\end{figure}

\clearpage
\subsubsection{Story 3}
\vspace{10pt}

\begin{figure}[htb!]
    \centering
    \includegraphics[width=\linewidth]{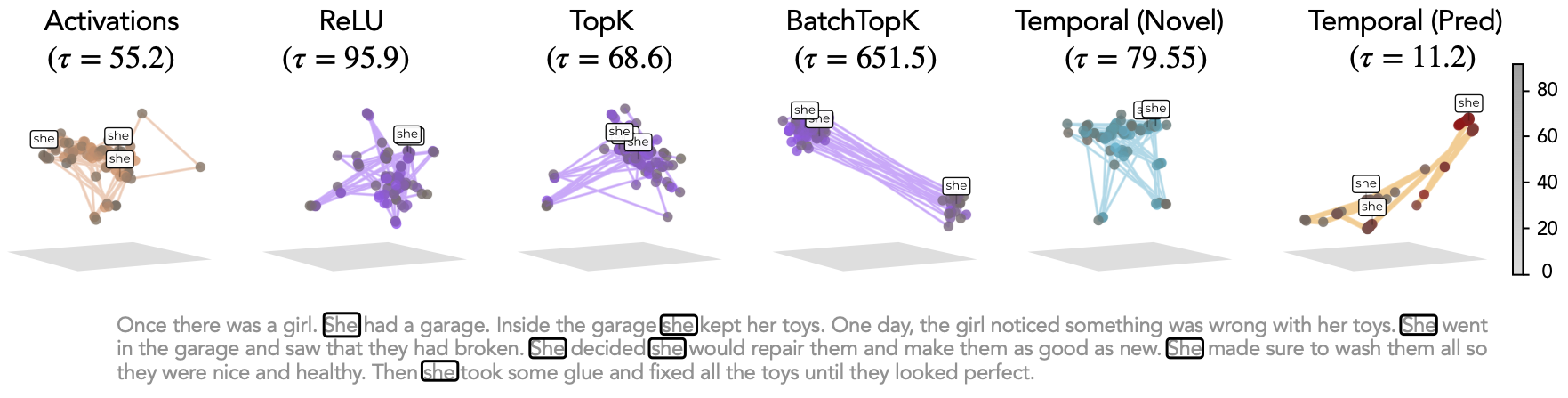}
    \caption{\textbf{Geometry.} Repeating results of Fig.~\ref{fig:ind_story_geometry}, we again find smooth trajectories in UMAP projections for Temporal Feature Analysis.
    \vspace{10pt}
    }
    \label{fig:llama_story_geometry_3}
\end{figure}

\begin{figure}[htb!]
    \centering
    \includegraphics[width=\linewidth]{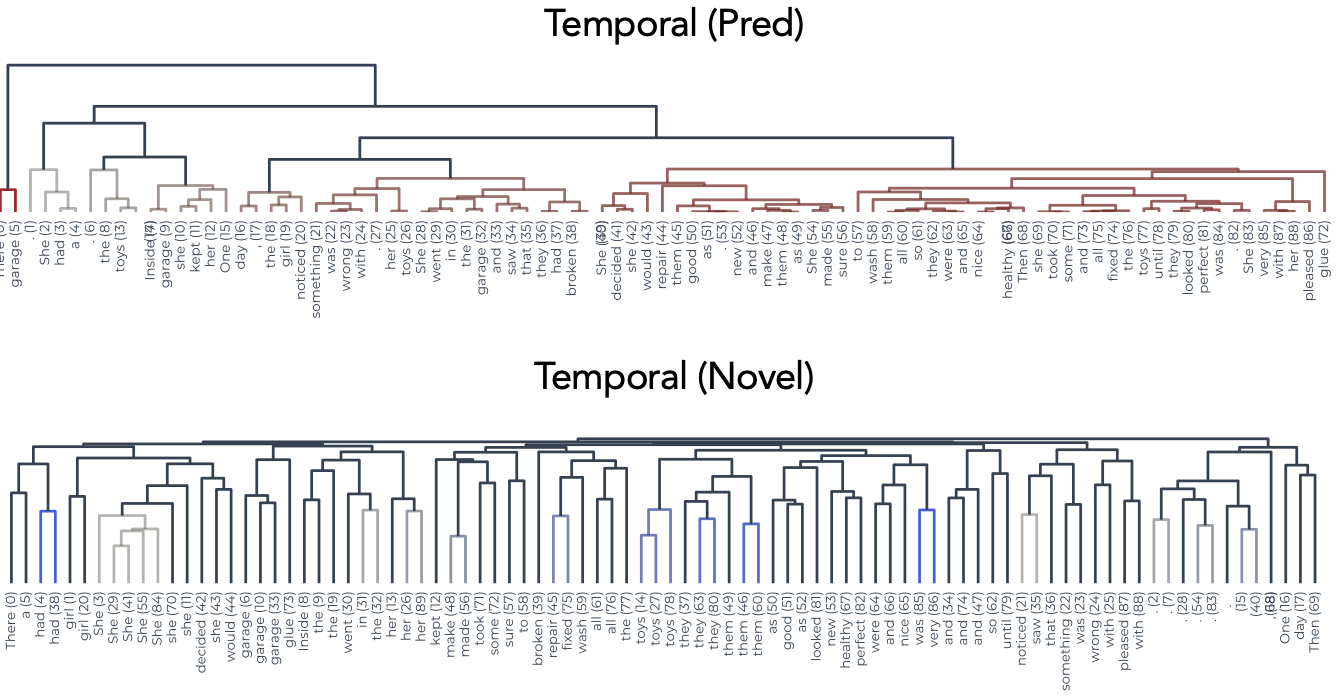}
    \caption{\textbf{Dendrograms of Predictive and Novel Components from Temporal Feature Analysis.} Reproduction of the results from Fig.~\ref{fig:ind_story_geometry}b, but including the novel component's dendrogram as well.
    \vspace{10pt}
    }
    \label{fig:llama_story_temporal_dendro_3}
\end{figure}

\begin{figure}[htb!]
    \centering
    \includegraphics[width=\linewidth]{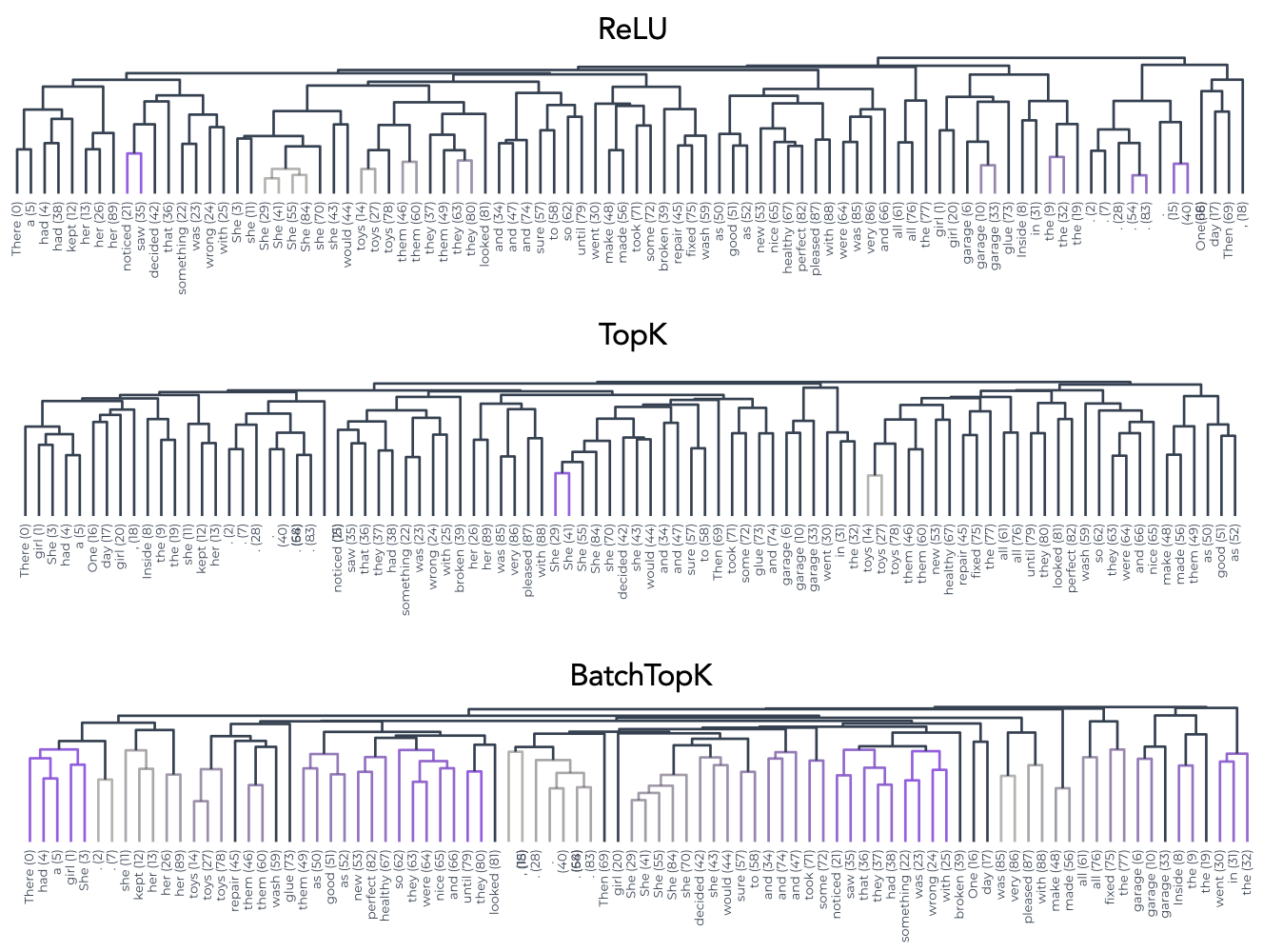}
    \caption{\textbf{Dendrograms of Standard SAEs' Latents Codes.} Reproduction of the results from Fig.~\ref{fig:ind_story_geometry}b, but for latent codes extracted using standard SAEs.
    \vspace{10pt}
    }
    \label{fig:llama_story_standard_dendro_3}
\end{figure}

\clearpage
\subsubsection{Comparing Eigenspectrum with Slow vs.\ Fast Features}

\begin{figure}[htb!]
    \centering
    \includegraphics[width=\linewidth]{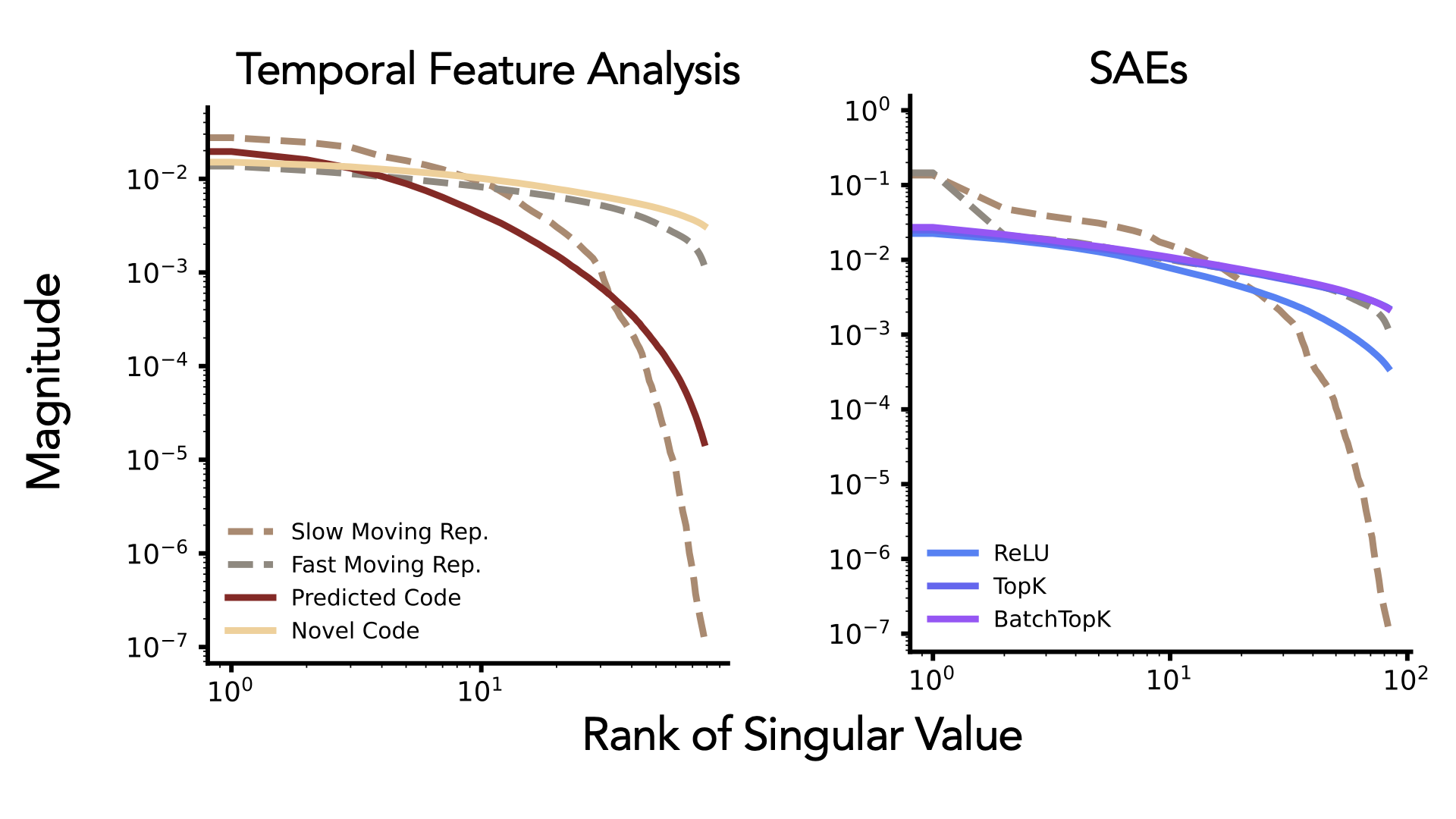}
    \caption{\textbf{Kernel spectrum for latent codes and model representations}. Kernels defined using novel code from Temporal Feature Analysis and standard SAEs both align well with the fast-changing part of model representations; meanwhile, only the predictive code shows strong similarity to the slow changing part.
    \vspace{10pt}
    }
    \label{fig:llama_spectra}
\end{figure}

\clearpage
\subsection{Event Boundaries and Noise Stability}

\subsubsection{Event Boundaries}

\begin{figure}[htb!]
    \centering
    \vspace{20pt}
    \includegraphics[width=\linewidth]{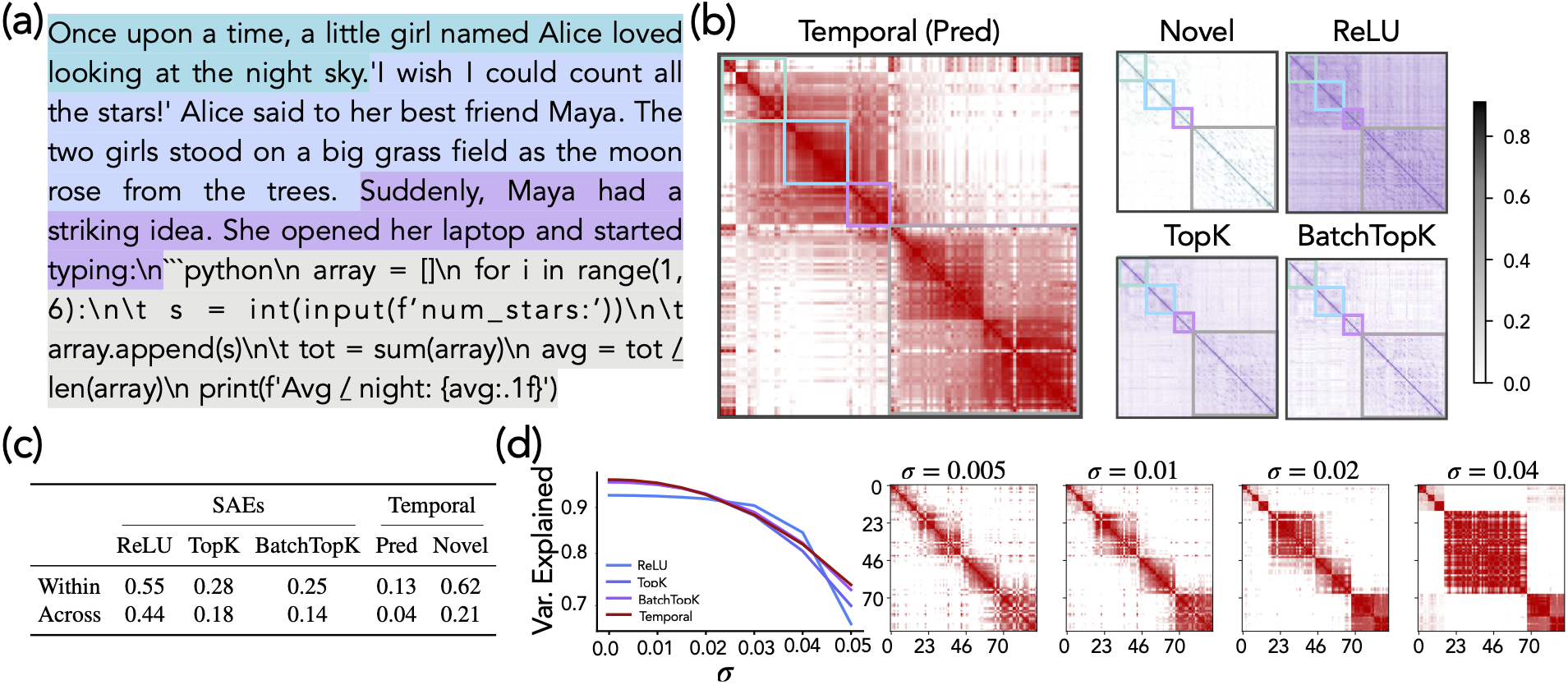}
    \caption{\textbf{Analysis of Event Boundaries.} Reproducing results from Fig.~\ref{fig:stories} on Llama-3.1-8B, we see similar behavior as the Gemma analysis done previously.
    \vspace{20pt}
    }
    \label{fig:llama_event_boundaries}
\end{figure}

\begin{figure}[htb!]
    \centering
    \includegraphics[width=\linewidth]{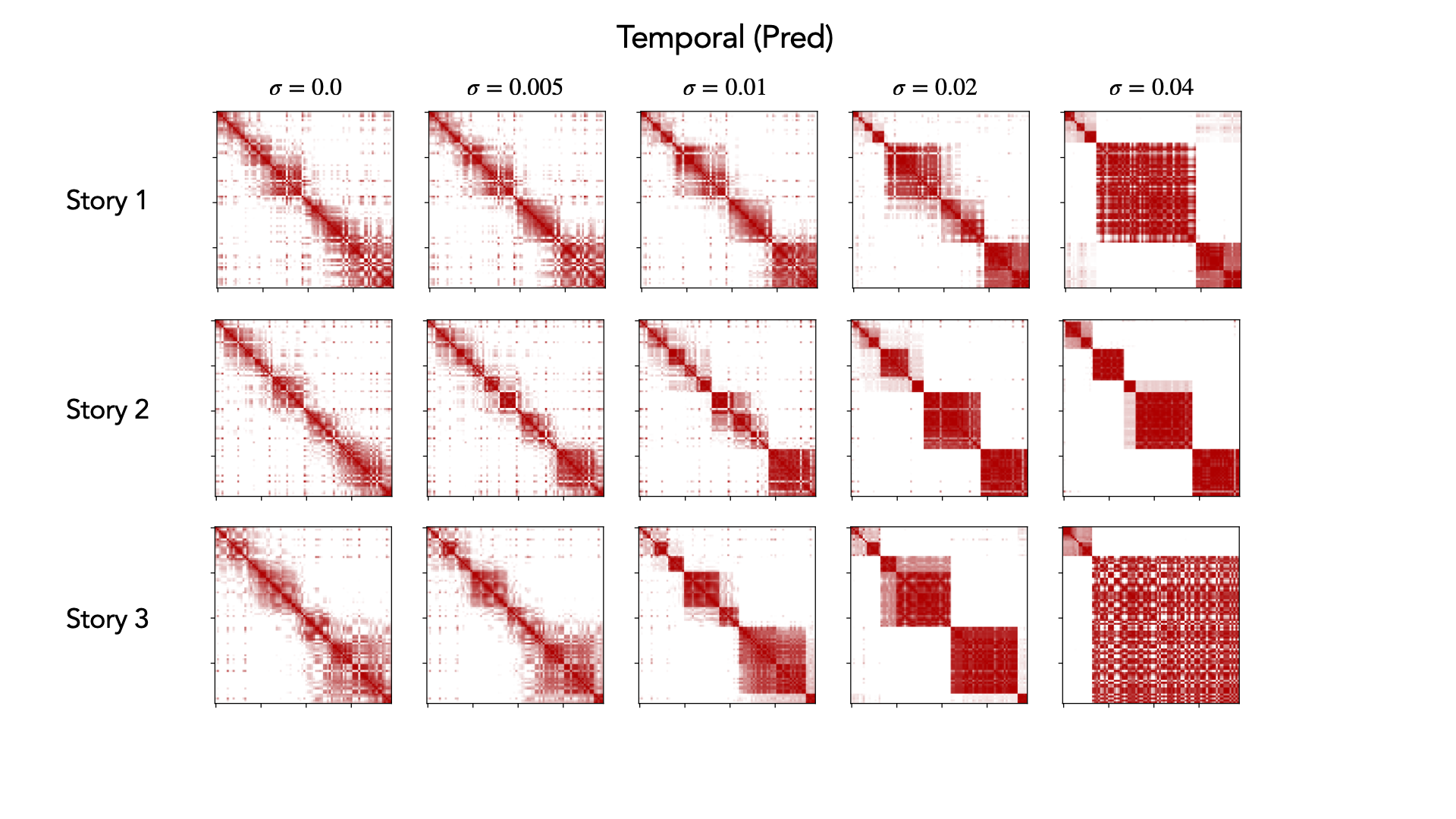}
    \caption{\textbf{Temporal (Pred) Similarity Maps Elicit Multi-Scale Structure with Noise.} Repeating the analysis shown in Fig.~\ref{fig:stories}, we find the coarsening of temporal blocks is a robust result that continues to hold for different inputs. 
    \vspace{10pt}
    }
    \label{fig:llama_pred_sim_and_noise}
\end{figure}

\begin{figure}[htb!]
    \centering
    \includegraphics[width=\linewidth]{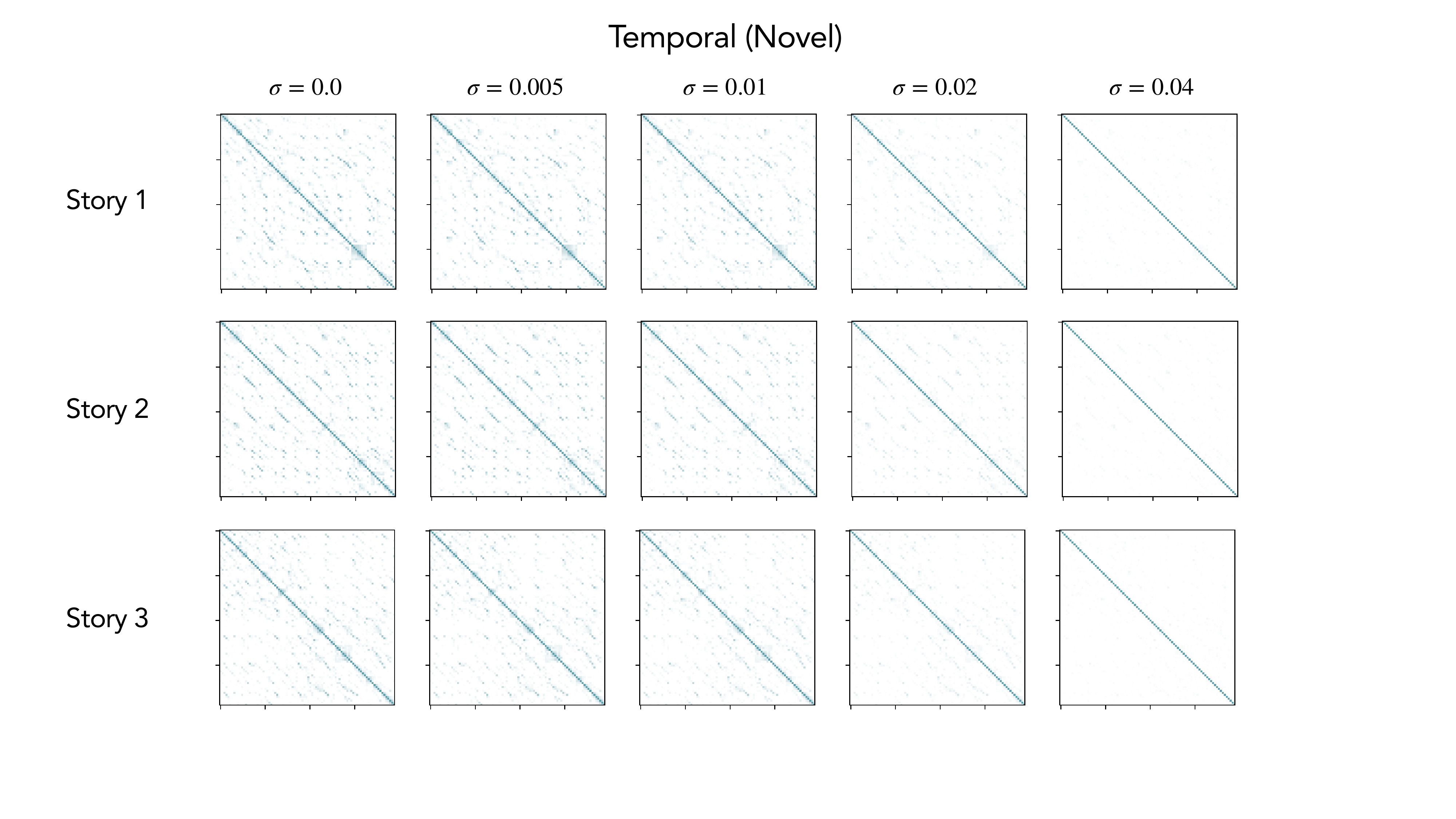}
    \caption{\textbf{Temporal (Novel) Similarity Maps under Noise.} Repeating the analysis shown in Fig.~\ref{fig:stories}, we find the novel component is only able to capture minimal local similarities, which when analyzed via dendrograms, show clustering based on lexical information.
    \vspace{10pt}
    }
    \label{fig:llama_novel_sim_and_noise}
\end{figure}

\begin{figure}[htb!]
    \centering
    \includegraphics[width=\linewidth]{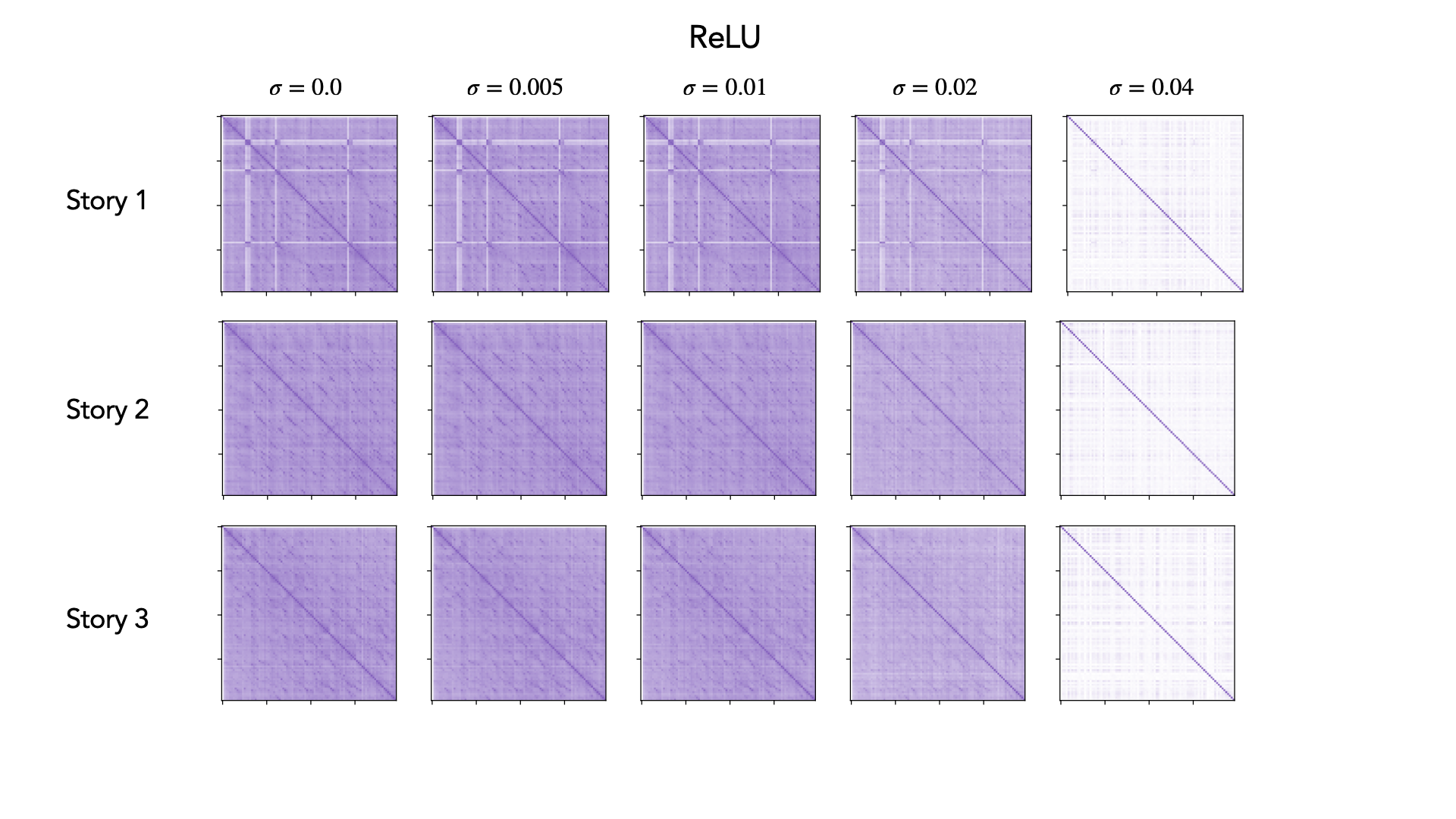}
    \caption{\textbf{ReLU Codes' Similarity Maps under Noise.} Repeating the analysis shown in Fig.~\ref{fig:stories} on ReLU SAEs, we find the ReLU latent code has high similarity across the board, suggesting lack of meaningful temporal information.
    \vspace{10pt}
    }
    \label{fig:llama_relu_sim_and_noise}
\end{figure}

\begin{figure}[htb!]
    \centering
    \includegraphics[width=\linewidth]{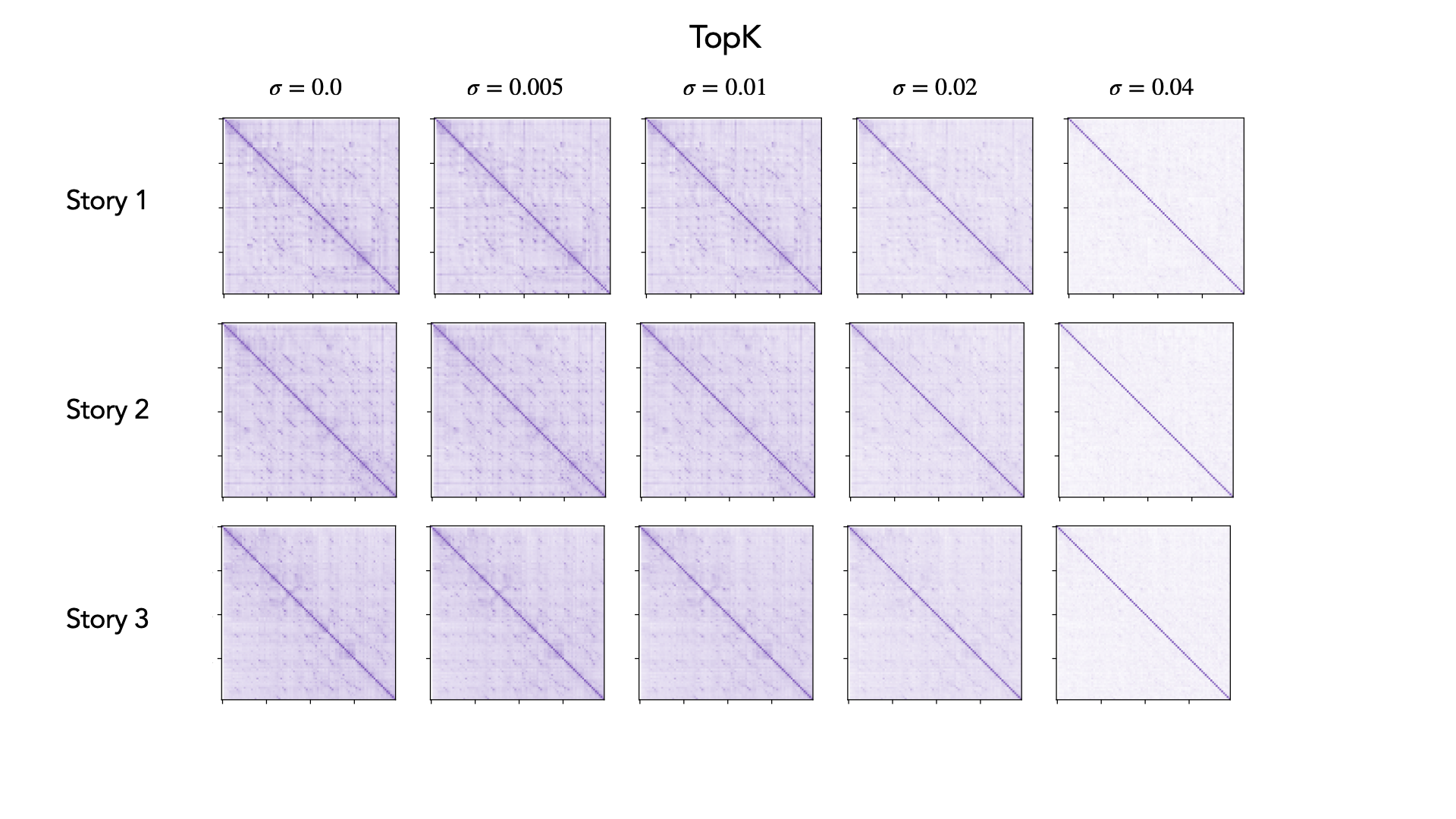}
    \caption{\textbf{TopK Codes' Similarity Maps under Noise.} Repeating the analysis shown in Fig.~\ref{fig:stories}, we find TopK SAE's latent codes are only able to capture minimal local similarities, which when analyzed via dendrograms, show clustering based on lexical information. 
    \vspace{10pt}
    }
    \label{fig:llama_topk_sim_and_noise}
\end{figure}

\begin{figure}[htb!]
    \centering
    \includegraphics[width=\linewidth]{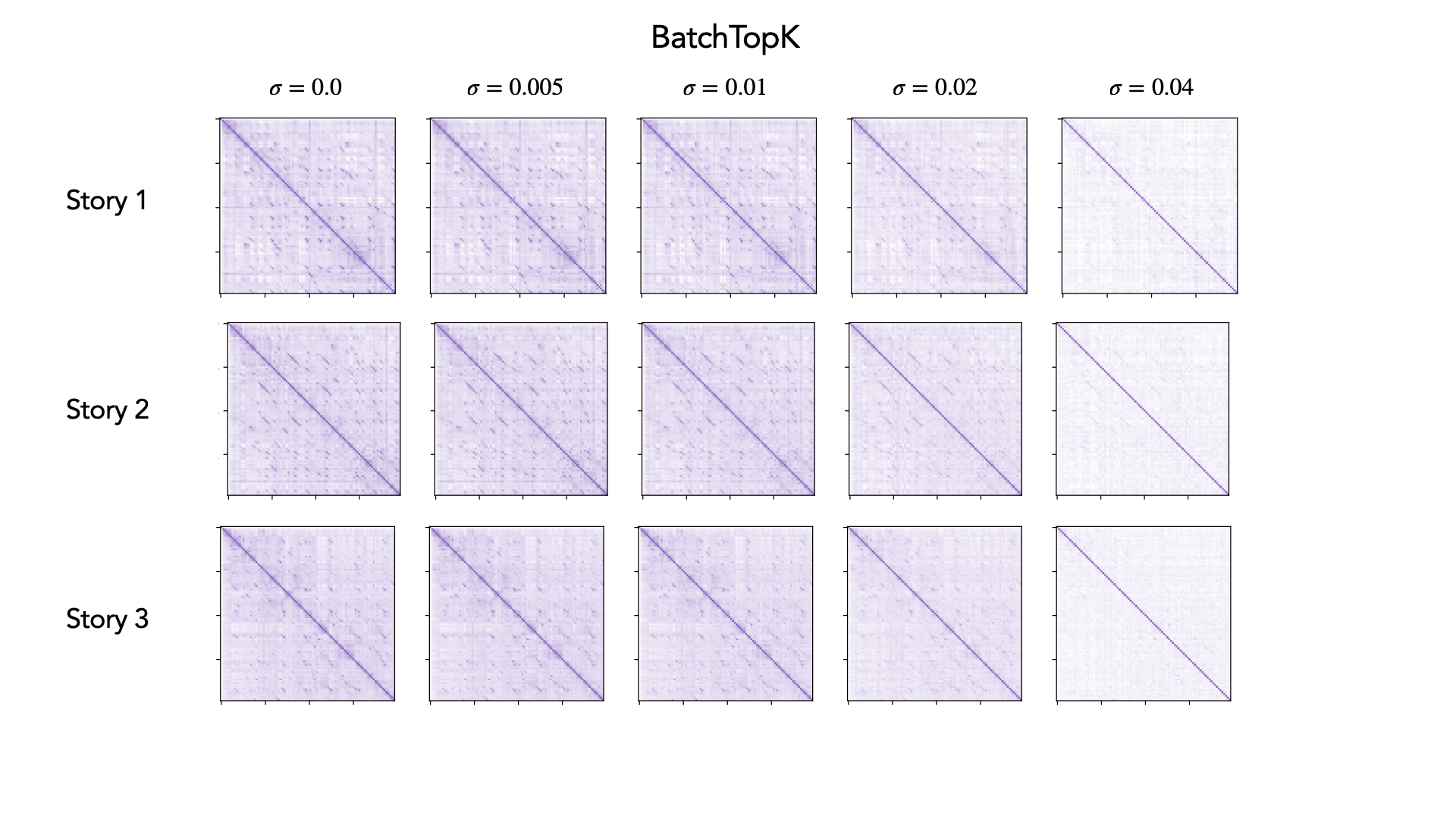}
    \caption{\textbf{BatchTopK Codes' Similarity Maps under Noise.} Repeating the analysis shown in Fig.~\ref{fig:stories}, we find BatchTopK SAE's latent codes are only able to capture minimal local similarities, which when analyzed via dendrograms, show clustering based on lexical information. 
    \vspace{10pt}
    }
    \label{fig:llama_batchtopk_sim_and_noise}
\end{figure}

\clearpage
\subsection{In-Context Representations}

\begin{figure}[htb!]
    \centering
    \includegraphics[width=\linewidth]{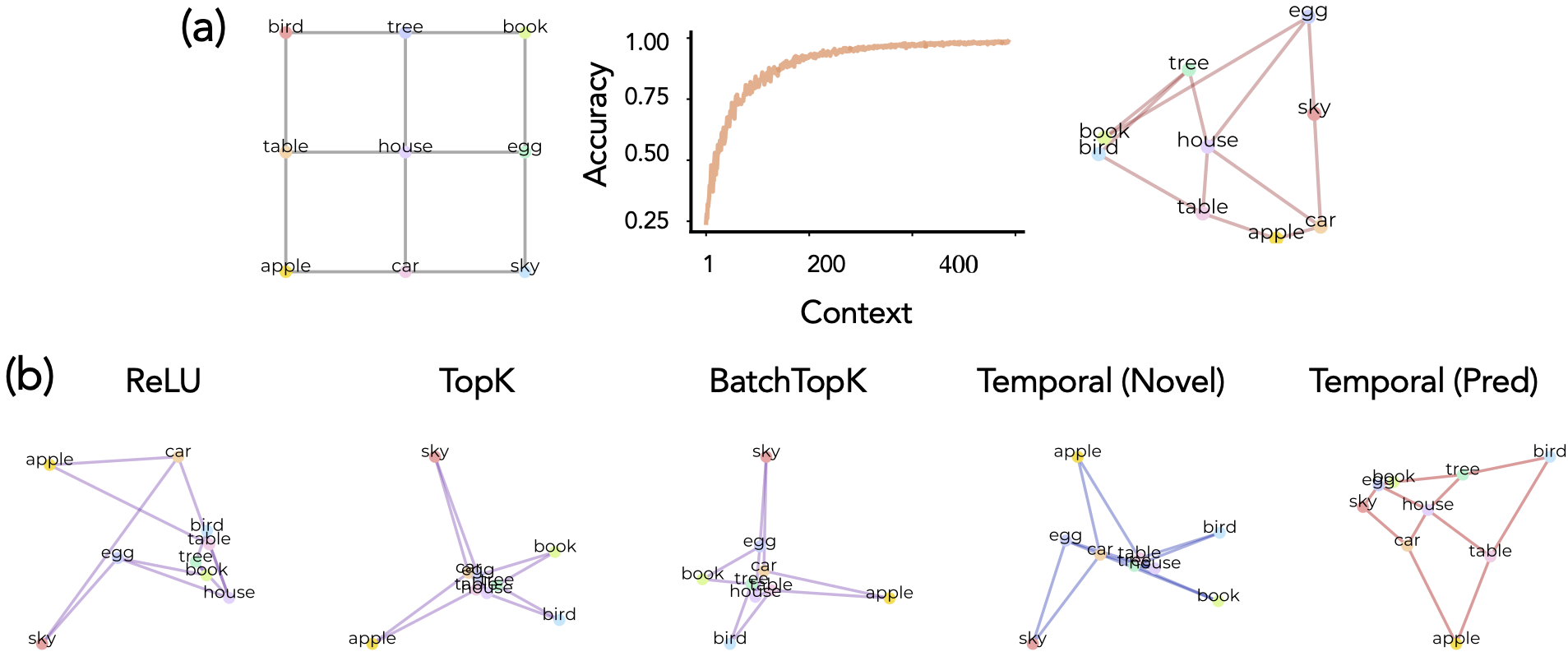}
    \caption{\textbf{In-Context Representations.} Reproducing the results of Fig.~\ref{fig:iclr_iclr} on Llama-3.1-8B, we see similar results as the Gemma analysis.
    \vspace{10pt}
    }
    \label{fig:llama_iclr_iclr}
\end{figure}

\clearpage
\section{Further Results: Temporal Feature Analysis with Only Predictive Component}
\label{appsection:pred_only}

In this section, we reproduce a subset of the experiments on an ablated version of Temporal Feature Analysis wherein only the predictive component is used for reconstructing the signal, i.e., the training objective is entirely based on prediction of the next-token representation. As we showed in Tab.~\ref{tab:comparing_saes_nmse}, this ablation does not achieve low-enough error. Correspondingly, the results in this section are a mere qualtitative artifact to show that Temporal Feature Analysis is not over-powered by its expressivity.

\begin{figure}[htb!]
    \centering
    \includegraphics[width=\linewidth]{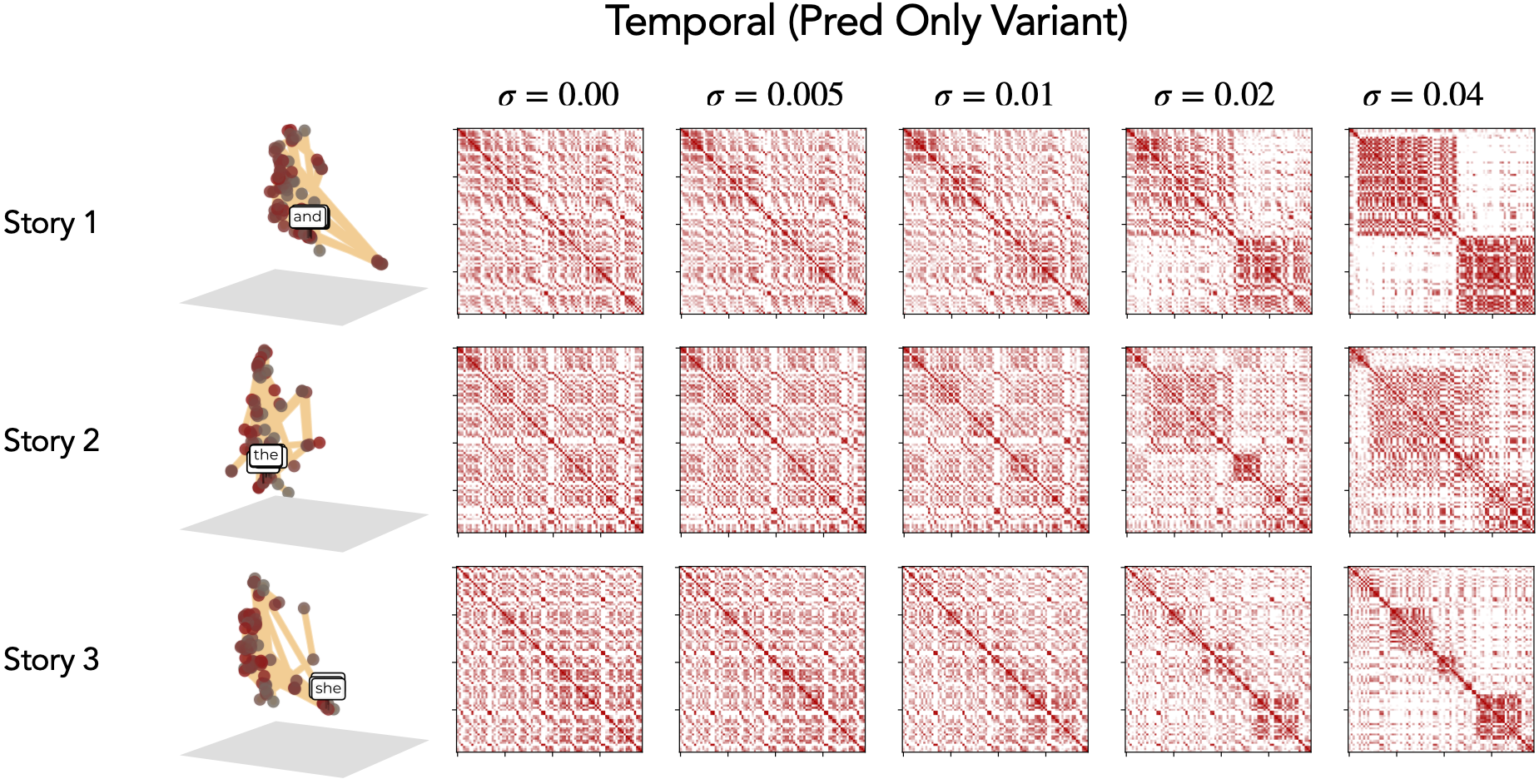}
    \caption{\textbf{Geometry and Heatmaps from the Prediction-only Ablation.} We see block and event structures do emerge, but they are substantially more wound than the smooth trajectories seen in Fig.~\ref{fig:ind_story_geometry}.
    \vspace{20pt}
    }
    \label{fig:pred_only_heatmaps}
\end{figure}

\begin{figure}[htb!]
    \centering
    \includegraphics[width=\linewidth]{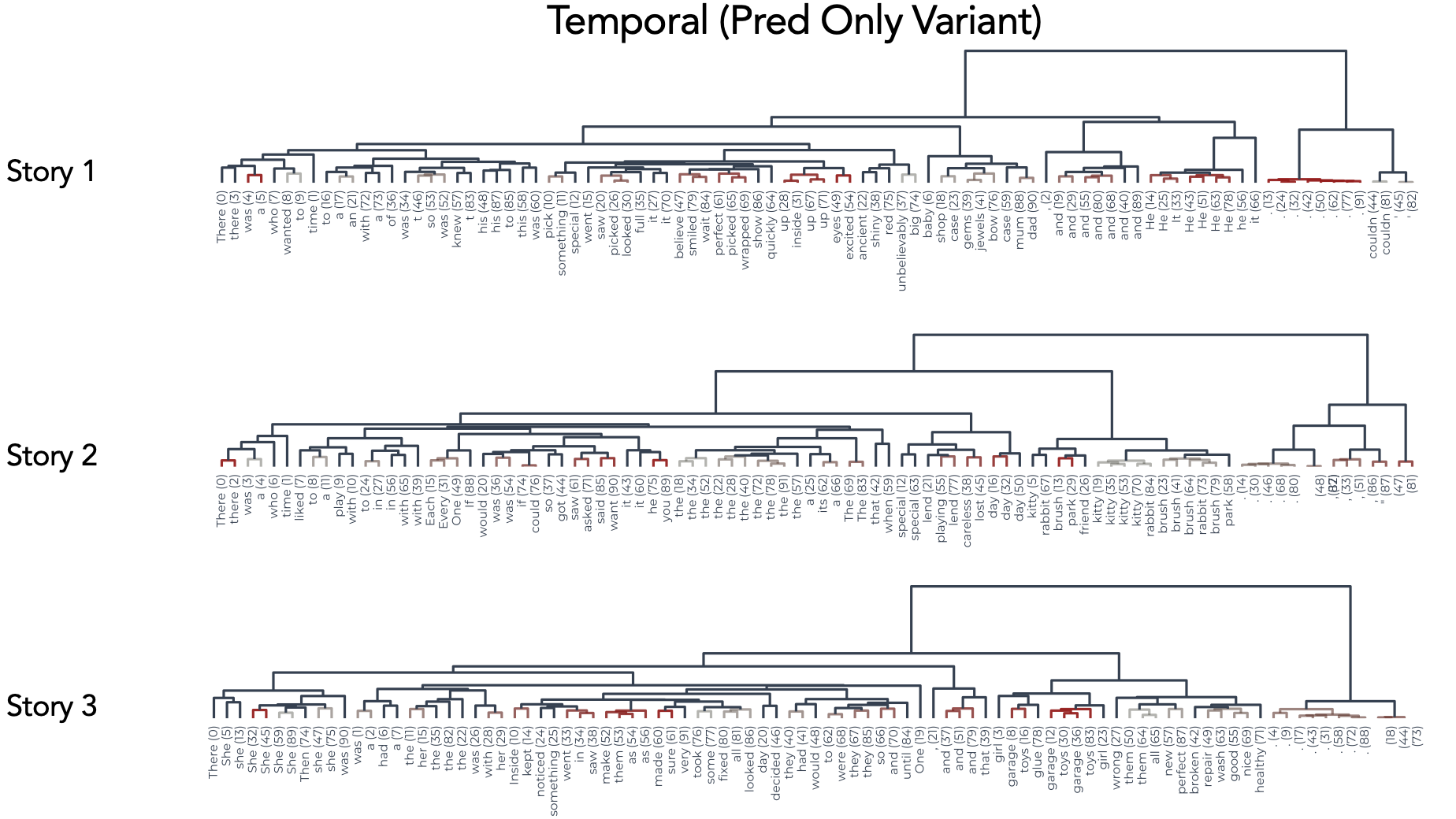}
    \caption{\textbf{Dendrograms from the Prediction-only Ablation.} Dendrograms reflect the event structures to some extent, but the results is Fig.~\ref{fig:ind_story_geometry} are substantially more crisp.
    \vspace{10pt}
    }
    \label{fig:pred_only_dendro}
\end{figure}

\begin{figure}
    \centering
    \includegraphics[width=1\linewidth]{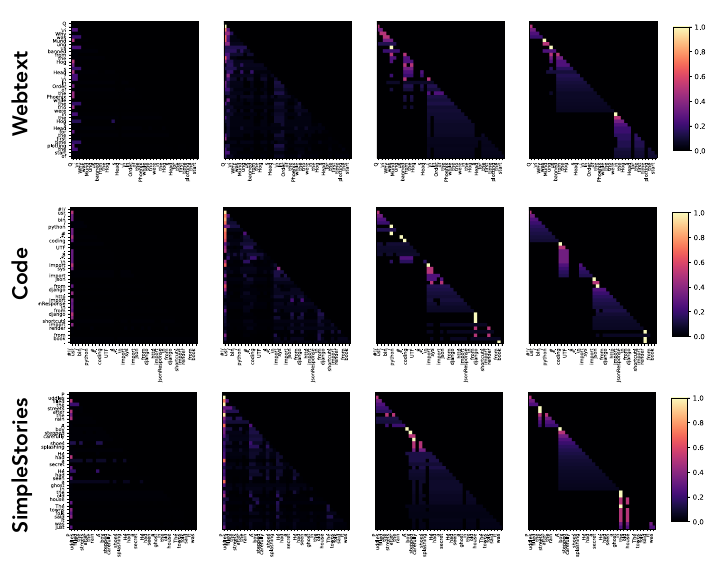}
    \caption{\textbf{Attention Patterns from the Prediction-only Ablation.} Attention weights for single sequences from three domains. For reference see the more coherent attention patterns of TemporalSAE, where the attention layer is trained jointly with a subsequent SAE in Figure~\ref{fig:attn_pattern_pred}.}
    \label{fig:pred_only_attn_pattern_viz}
\end{figure}

\clearpage
\section{Further Results: Emergent Separation of Predictive and Novel Code in Temporal Feature Analysis}
\vspace{10pt}

\begin{figure}[H]
    \centering
    \includegraphics[width=1\linewidth]{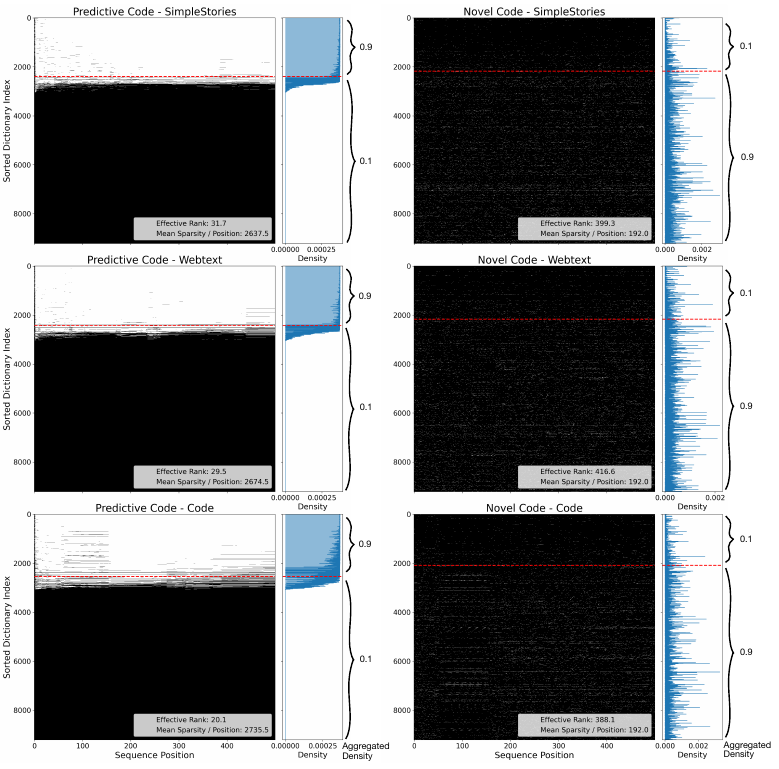}
    \caption{\textbf{Separation of Predictive and Novel Dictionaries.}
    Predictive and novel codes for sequences drawn from three datasets---SimpleStories (top), Webtext (middle), and Code (bottom)---are binarized to zero (black) and non-zero (white). The dictionary elements are sorted by the sum of non-zero activations in the predictive code across the sequence. The ordering obtained from the predictive code is also applied to the novel code. The red dashed line marks the separation of 90\% of non-zero activation counts; only 10\% are overlapping into the other subset. Additionally, the effective rank of the predictive codes is two orders of magnitudes lower than the average number of non-zero dictionary elements (Mean L0).
    }
    \label{fig:separation_pred_novel}
\end{figure}

The Temporal Feature Analyzers studied in the main paper share a dictionary for both predictive and novel codes. 
We now investigate if there is any shared structure in the two dictionaries. 
In particular, given the predictive and novel codes play a different role, it is plausible the SAE learns to approximately split the dictionary into two parts: one responsible for computing the predictive code and the other for novel one.
Specifically, we investigate a Temporal Feature Analyzer trained on Gemma-2-2B Layer 12 residual stream activations. 
Results are shown in Fig.~\ref{fig:separation_pred_novel} and show a separation of dictionary elements for sequences drawn from three datasets: SimpleStories, Webtext, and Code. 
Specifically, we find $\sim$2K dictionary elements participate in defining the predictive code, while the remaining $\sim$8K primarily participate in defining the novel code.
However, the absolute count is merely indicative of the L0 sparsity of a code, which may not reflect how many \textit{directions} are actually used to define the code---estimating latter requires computing the rank of the code (in fact, we note that rank-sparsity is a well-known alternative to L0 sparsity in dictionary learning literature~\citep{elad2010sparse}).
As we show in Fig.~\ref{fig:separation_pred_novel}, the effective rank of the predictive codes ($\sim$20--30) is substantially lower than both the effective rank ($\sim$390--410) and absolute L0 sparsity of novel codes.

Overall, the posthoc analysis above shows that predictive and novel codes largely use separate subsets of dictionary elements. 
This emergent disentanglement of the two components motivates one to preemptively split the dictionary into two components---one responsible for the predictive code and other for novel one. 
Our preliminary experiments show a Temporal SAE trained with this split-dictionary architecture is similarly performant as the tied dictionary one, resulting in predictive / novel codes with effective ranks $\sim$20--30 / $\sim$390--420, and a slightly better overall loss.
It is possible optimizing hyperparameters (e.g., having different expansion factors for the two components) can make this architecture more performant than the tied dictionary one, but we leave a further characterization to future work.

\clearpage
\section{Proofs}

\subsection{Priors on the Sparse Code for various SAEs}
\label{appsec:saepriors-derived}
We restate and prove the proposition on independence priors of SAEs over time (Proposition \ref{thm:priors-over-time}) below.
\begin{proposition}[Independence priors over time] Consider the SAE maximum aposteriori (MAP) objective for ReLU, JumpReLU, TopK and BatchTopK SAEs. The sparsity constraints for these SAEs are additive over time, resulting in:
\begin{equation}
\begin{split}
    \argmin_{\D, \z} \frac{1}{T} &\sum_{i=1}^{T} \|\x_{i} - \D \z_{i}\|_2^2 + \lambda \regsparse(\z_{i}),\\
    \text{s.t. } \z_{k} &= f_{\mathtt{SAE}}(\x_{k})\; \forall k, \; \Tilde{g}(\z_1, \dots, \z_T)=\sum_i \Tilde{g}(\z_i)=0. \; 
\end{split}
\end{equation}
This MAP objective has an independent and identically distributed (i.i.d.) prior over time i.e., 
\begin{align*}
    P(\z_1, \dots, \z_T) &\propto \prod_{t=1}^T \exp \left( - \lambda \regsparse(\z_i) -\Tilde{\lambda}\Tilde{g}(\z_i) \right) =  \prod_i P(\z_i),
\end{align*}
\end{proposition}
\begin{proof}
    The sparsity constraints and sparsity-promoting regularizers for the SAEs under study are specified in the table below.
\begin{table}[ht]
    \centering
    \begin{tabular}{c c c}
    \toprule
     SAE  &  \makecell{Regularizer \\$\regsparse (\z_i)$} & \makecell{Sparsity Constraint \\$\Tilde{g} (\z_1, \dots, \z_T)=0$} \\
    \midrule 
      ReLU & $\|\z_i\|_1$ & 0 \\
      JumpReLU & $\|\z_i\|_0$ & 0 \\
      TopK & 0 & $\sum_{i=1}^T(\|\z_i\|_0-K)^2$ \\
      BatchTopK & 0 & $\frac{1}{T}\sum_{i=1}^T \|\z_i\|_0 - K$
    \end{tabular}
    \caption{Sparsity constraints and regularizers for SAEs}
    \label{tab:sae-reg-sparsityconstraints}
\end{table}

Note that TopK imposes a pointwise hard sparsity constraint, which has been restated using sum-of-squares above for convenience. While BatchTopK imposes the fixed mean sparsity for each mini batch, we take the batch to capture the entire timeseries in the above formulation. The above table shows us that the sparsity constraint is additive over time in all cases:
\begin{align}
    \Tilde{g} (\z_1, \dots, \z_T) &= \sum_{i=1}^T \Tilde{g}(\z_i) = 0, \\
    \text{where } \Tilde{g}(\z_i) &= \begin{cases}
        (\|\z_i\|_0 - K)^2  & \text{TopK} \\
        \frac{1}{T}(\|\z_i\|_0 - K) & \text{BatchTopK} \\
        0 & \text{ReLU, JumpReLU}
    \end{cases}
\end{align}
Recall that SAEs solve the following constrained optimization problem (restated from Eq.~\ref{eq:sae-sparsecoding}).
\begin{equation}
\begin{split}
    \argmin_{\D, \z} \frac{1}{N} &\sum_{i=1}^{T} \|\x_{i} - \D \z_{i}\|_2^2 + \lambda \regsparse(\z_{i}),\\
    \text{s.t. } \z_{k} &= f_{\mathtt{SAE}}(\x_{k})\; \forall k, \; \Tilde{g}(\z_1, \dots, \z_T)=0 \;.
\end{split}
\end{equation}
We rewrite the above problem using Lagrange multipliers on the sparsity constraints, and further simplify using the above result on constraints being additive over time, as:
\begin{equation}
\begin{split}
    \argmin_{\D, \z} \frac{1}{N} &\sum_{i=1}^{T} \|\x_{i} - \D \z_{i}\|_2^2 + \lambda \regsparse(\z_{i}) + \Tilde{\lambda} \Tilde{g}(\|\z_i\|_0),\\
    \text{s.t. } \z_{k} &= f_{\mathtt{SAE}}(\x_{k})\; \forall k.
\end{split}
\end{equation}
Note that we don't use Lagrange multipliers on the SAE architecture constraint $\z=f_{\mathtt{SAE}}(\x)$ since we only care about $\z$-specific constraints (which don't include the data $\x$) for the prior. 

\paragraph{Bayesian Interpretation.} The objective function above (sans the SAE architecture constraint) can be thought of as minimizing the negative log posterior, which is proportional to log prior added to log likelihood:
\begin{align}
    -\log P(\z_1, \dots, \z_T \mid \x_1, \dots, \x_T) &\propto \underbrace{\frac{1}{N} \sum_{i=1}^{T} \|\x_{i} - \D \z_{i}\|_2^2}_{-\log P(\x_1, \dots, \x_T \mid \z_1, \dots, \z_T)} + \underbrace{\frac{1}{N} \sum_{i=1}^{T}\lambda \regsparse(\z_{i}) + \Tilde{\lambda} \Tilde{g}(\|\z_i\|_0)}_{-\log P(\z_1, \dots, \z_T)}
\end{align}

The prior over latents $\z$ is:
\begin{align}
    \log P(\z_1, \dots, \z_T) &= - \frac{1}{N} \sum_{i=1}^{T}\lambda \regsparse(\z_{i}) + \Tilde{\lambda} \Tilde{g}(\|\z_i\|_0), \\
    \implies P(\z_1, \dots, \z_T) &= \prod_{i=1}^T \exp \left( - \lambda \regsparse(\z_{i}) + \Tilde{\lambda} \Tilde{g}(\|\z_i\|_0) \right) = \prod_i P(\z_i).
\end{align}
Therefore, the prior is multiplicative over time, implying independence, and the distribution at each time $t$ has the same form, implying that the prior is independent and identically distributed (i.i.d.). This completes the proof.

\end{proof}

\subsection{Sparsity over time and implications}
\label{app:proofs-sparsity-supportswitch}
The i.i.d. prior over time leads to similar i.i.d. assumptions about the sparsity of representations over time for SAEs. The following corollary of Prop. \ref{thm:priors-over-time} states this assumption.
\begin{corollary}[Assumptions of time-invariant sparsity] As a consequence of the i.i.d. priors over time from Prop. \ref{thm:priors-over-time}, standard SAEs assume that sparsity of representations emerges from a fixed distribution independently over time (i.i.d.), i.e.,
\begin{align}
    P(\|\z_1\|_0, \dots, \|\z_T\|_0) &= \prod_t P(\|\z_t\|_0), \\
    \implies P(\|\z_t\|_0) &= P(\|\z_{t'}\|_0) \;\; t\neq t'
\end{align}
\end{corollary}
\begin{proof}
    This result essentially follows from Prop. \ref{thm:priors-over-time}. Since the prior over latents $\z$ is i.i.d., and sparsity is a function of the latents, sparsity must also be i.i.d. over time. As a consequence, the distributions at different times are identical.
\end{proof}

Due to the observation of increase in number of concepts over time in language (Fig. \ref{fig:intro}), in the following proposition, we show that time invariant sparsity priors in SAEs lead to SAEs losing out on important structural information in model activations.

\begin{proposition}[Restrictive Sparsity Budget Leads to Support Switching in SAEs] Suppose data $\x$ lies on a $C^1$ manifold $\mathcal{M}\subset \mathbb{R}^d$. For $\x \in \mathcal{M}$, let $T_{\x}\mathcal{M}$ denote the tangent space at $\x$, and $m(\x)=\dim T_{\x}\mathcal{M}$ its dimension. Suppose a Sparse Autoencoder (SAE) has a sparsity budget $|S(\x)|=K$. . If $K < m(\x)$, under the assumption of low mean-squared error (MSE) of the SAE, in some neighborhood $\mathcal{N}_{\x}$ of $\x$, $\exists \x_1, \x_2 \in \mathcal{N}_{\x}$ s.t. $S(\x_1)\neq S(\x_2)$, i.e., support switching occurs in the SAE latents. 
\end{proposition}
\begin{proof}
    Suppose not: this implies that for every neighborhood $\mathcal{N}_{\x}$ of $\x$, the SAE has the same support at all points: $\x_1, \x_2 \in \mathcal{N}_{\x} \implies S(\x_1)=S(\x_2)$. Since the SAE decoder is linear (from latents $\z$ to output $\hat{\x}$ conditioned on the support $S(\x)$), fixed support with size $|S(\x)|=K$ in the neighborhood $\mathcal{N}_{\x}$ corresponds to a $K$-dimensional region $\mathcal{R}_{S(\x)}$ in $\mathbb{R}^d$ in the SAE outputs (which aim to reconstruct the input $\x$). Since $m(x)=\dim T_{\x}$, $\mathcal{N}_{\x}$ spans an $m(x)$-dimensional region in Euclidean space (using a local chart of the manifold $\mathcal{M}$ at $\x$). Since $m(\x)=\dim \mathcal{N}_{\x}>K=\dim \mathcal{R}_{S(\x)}$, the data (inputs to SAE) and SAE reconstructions belong to regions with different dimensions---with the reconstructions $\hat{\x}$ belonging to a lower dimensional region than inputs $\x$---which implies that the reconstruction error of the SAE (MSE) is high. This contradicts the assumption that SAE reconstruction error is low, which implies that our assumption about same support within a neighborhood of $\x$ is false. This implies that there is a neighborhood of $\x$ where the support of the SAE switches. This completes the proof. 
\end{proof}

The implications of support switching have been discussed in the main text: the local geometry in data $\x$ is not preserved in SAE latents $\z$, so concepts that correspond to local geometry (over time, which indicates contextual information) may not be observed in the SAE latents.

\clearpage
\section{Further Theoretical Results: SAE-wise Characterization of Priors over concepts and time}
We can further think of the priors of each SAE over concepts as well as over time in a generative fashion. In some cases, this mainfests as a hierarchical latent variable model $n\rightarrow S\rightarrow \z$, where $n=\|\z\|_0$ is the sparsity, $S=\mathrm{supp}(\z)=\{k: z^k>0\}$ is the support, and $\z$ is the SAE latent code.

\begin{proposition}[SAE Priors on Sparse Code]
\label{thm:saepriors}
Let $S_t = \mathrm{supp}(\z_t)=\{k: z^k_t>0\}$ be the set of active latents in the sparse code $\z$ at time $t$, and $n_t = |S_t|$ be the cardinality of $S_t$ (the number of active latents). Each SAE imposes a prior distribution on the sparse code $\z$, arising from its sparsity penalty $\regsparse(\z)$ or implicit conditions imposed on the sparse code. These conditions are highlighted in Table \ref{tab:saepriors}.
\begin{table}[ht]
\centering
    \caption{Priors over concept interactions and dynamics for various SAEs}
    \label{tab:saepriors}
    \footnotesize
    \begin{tabular}{c c c} 
    \toprule
        $f_{\mathtt{SAE}}, \regsparse(\z) $ & Across-Concept Prior (interaction) & Across-time Prior (dynamics) \\ 
        \midrule 
        \makecell{ReLU, \\ $L_1$-norm} & 
        $z^1_t, \dots, z^M_t \overset{\text{i.i.d.}}{\sim} \text{Laplace}(0, \cdot)$ & 
        $\z_1, \dots, \z_t \overset{\text{i.i.d.}}{\sim} P_{\z}$\\
        TopK & 
        \makecell{$z^{i_1}_t, \ldots, z^{i_K}_t \mid S_t \;\overset{\text{i.i.d.}}{\sim} U(0, \cdot)\;\forall i_\cdot\in S_t,$ \\ $S_t \sim U([M]^{K})$} & 
         \makecell{$(\z_1, S_1), \dots, (\z_t,S_t) \overset{\text{i.i.d.}}{\sim} P_{S}P_{\z \mid S},$ \\ $S_1, \dots, S_t \overset{\text{i.i.d.}}{\sim} U([M]^K)$} \\
        \makecell{JumpReLU,\\ $L_0$-norm} &
        \makecell{$z^{i_1}_t, \ldots, z^{i_{n_t}}_t \mid S_t \;\overset{\text{i.i.d.}}{\sim} U(0, \cdot)\;\forall i_\cdot\in S_t,$ \\ $S_t \mid n_t \sim U([M]^{n_t})$}
        & 
        \makecell{$(\z_1, S_1, n_1), \dots, (\z_t, S_t, n_t) \overset{\text{i.i.d.}}{\sim} P_{n}P_{S \mid n} P(\z \mid S),$ \\ $n_1, \dots, n_t \overset{\text{i.i.d.}}{\sim} P_n$} \\
        BatchTopK & 
        \makecell{$z^{i_1}_t, \ldots, z^{i_{n_t}}_t \mid S_t \;\overset{\text{i.i.d.}}{\sim} U(0, \cdot)\;\forall i_\cdot\in S_t,$ \\ $S_t  \mid n_t \sim U([M]^{n_t})$}
         & 
         \makecell{$(\z_1, S_1, n_1), \dots, (\z_t,S_t, n_t) \overset{\text{i.i.d.}}{\sim} P_{n}P_{S \mid n} P(\z \mid S),$ \\ $n_{1}, \cdots, n_{t} \overset{\text{i.i.d.}}{\sim} P_{n}, \; \mathbb{E}[n_t]=K$}
    \end{tabular}

\end{table}

\end{proposition}

\subsubsection{ReLU SAE}

The vanilla ReLU SAE (\cite{bricken2023monosemanticity}, \cite{cunningham2023sparse}) is trained with the $L_1$-norm penalty:
\begin{align}
    \regsparse(\z) &= \|\z\|_1.
\end{align}
The prior over $\z$ for the above case is:
\begin{align}    
    \log P(\z_1, \dots, \z_N) &\propto -\sum_{i=1}^N \sum_{k=1}^M |z_i^k|, \\
    \implies P(\z_1, \dots, \z_N) &\propto \prod_{i=1}^{N} \Bigg( \prod_{k=1}^M \exp -\nu |z_i^k| \Bigg).
\end{align}
This joint distribution implies that for each sample $i$, different indices $k$ are sampled i.i.d. from the same distribution:
\begin{align}
    z_i^1, \dots, z_i^M \overset{\text{i.i.d.}}{\sim} \text{Laplace}(0, 1/\nu),
\end{align}
and different samples are all independently sampled from the same product Laplace distribution:
\begin{align}
    \z_1, \dots \z_N \overset{\text{i.i.d.}}{\sim} \text{Laplace}^{M}(0, 1/\nu).
\end{align}
This concludes the proof for priors of ReLU SAE trained with $L_1$ norm sparsity penalty. \hfill $\square$

\subsubsection{TopK SAE}
\label{appsection:topkprior}

The TopK SAE (\cite{makhzani2013k}, \cite{gao2024scaling}) directly controls the sparsity of the representation $\z$ by fixing it at $\|\z\|_0=K$, instead of imposing an explicit sparsity penalty $\regsparse(\z)$ in the loss function. The objective function for TopK SAE is:
\begin{align}
    &\argmin_{\D, \z} \sum_{i=1}^{N} \frac{1}{N} \|\x_{i} - \D \z_{i}\|_2^2,  \\
    &\text{s.t. } \forall j,  \z_{j} = f_{TopK}(\x_{j}),\; \|\z_j\|_0=K .
\end{align}
Since the fixed sparsity is a hard constraint that depends on $\z$ alone (and not the data $\x$), it can further be simplified as a sum-of-squares constraint: $\sum_j (\|\z_j\|_0-K)^2 = 0$. We can use Lagrange multipliers to reformulate it as an effective prior:

\begin{align}
    &\argmin_{\D, \z, \{\lambda_i\}} \sum_{i=1}^{N} \frac{1}{N} \Bigg( \|\x_{i} - \D \z_{i}\|_2^2 + \lambda \left( \big| \|\z_i\|_0-K \big| \right)^2 \Bigg),  \\
    &\text{s.t. } \forall j,  \z_{j} = f_{TopK}(\x_{j}). 
\end{align}
The prior over $\z$ for the above (effective) regularizer is:
\begin{align}
\label{eq:topkprior1}
    \log P(\z_1, \dots, \z_N) &\propto - \sum_{i=1}^N \lambda \left( \big( \|\z_i\|_0-K \big)^2 \right) \\
    \implies P(\z_1, \dots, \z_N) &\propto \prod_{i=1}^N \exp \big( -\lambda \big( \|\z_i\|_0 - K \big)^2 \big)
\end{align}

Note that the above prior is finite for finite values of $\lambda$, but the overall objective optimizes over $\lambda$, resulting in a \textit{hard} prior peaked at $\|\z_i\|_0=K$ for each sample $i$.  

The factorization over samples $i$ implies mutual independence of $\z_1, \dots, \z_n$: $P(\z_1, \dots, \z_n) = \prod_{i=1}^N P(\z_i)$.

As defined in Theorem \ref{thm:saepriors} (and restated here for convenience), let $S_i = \mathrm{supp}(\z_i)=\{k: z_i^k>0\}, n_i = |S_i|=\|\z_i\|_0$ denote the active indices and their number (sparsity) respectively. 

For individual samples $\z_i$, if we condition on the set of active indices $S_i$, the sparsity gets fixed since $\|\z_i\|_0 = |S_i| = n_i$, and the distribution becomes constant:
\begin{align}
    P(\z_i \mid S_i) &= C \\
    \implies z_i^\mu \mid S_i \sim & \begin{cases}
        U(0, \kappa) \; &\mu\in S_i \\
        \delta_0 \; &\mu \notin S_i
    \end{cases} , \text{and} \\
    z_i^{\mu_1}, \dots, z_i^{\mu_{|S_i|}} \mid S_i &\overset{\text{i.i.d.}}{\sim} U(0,\kappa)\; \text{ for } \mu_\cdot \in S_i
\end{align}
where $C, \kappa$ are appropriate constants.

Since $\{\z_i\}_i$s are mutually independent, any measurable function of each is also independent. The indices of nonzero entries of $\z_j$, i.e., $S_j$ is a measurable function since it is a map $S: \mathbb{R}_+^M \rightarrow 2^M$ which is discrete valued, and pre images of each value---a set of nonzero indices---are measurable since they equal the cartesian products of the measurable sets $\{z=0\}, \{z>0\}$ over all indices. Hence, $S_1, \dots, S_n$ are also independent. 

Since $S_i = g(\z_i)$ and the distribution of $\z_i$ depends only on $n_i=\|\z_i\|_0$ (Eq.~\ref{eq:topkprior1}), the distribution of $S_i$ will also depend only on $n_i$, becoming uniform when conditioned on $n_i$. In TopK SAE, $n_i=K$ is a constant. Therefore, each $S_i \sim U([M]^{K})$, and together with independence argued above,
\begin{align}
    S_1, \dots, S_N \overset{\text{i.i.d.}}{\sim} U([M]^{K})
\end{align}

This completes the proof for the priors of TopK SAE. \hfill $\square$

\subsubsection{BatchTopK SAE}

BatchTopK SAE (\cite{bussmann2024batchtopksparseautoencoders}) is a modification of the TopK SAE. Instead of fixing sparsity like TopK, BatchTopK allows variable sparsity per input while fixing the mean sparsity over a batch at $K$.  The objective function for BatchTopK SAE can equivalently be written as:
\begin{align}
    &\argmin_{\D, \z} \sum_{i=1}^{N} \frac{1}{N} \|\x_{i} - \D \z_{i}\|_2^2,  \\
    &\text{s.t. } \forall j,  \z_{j} = f_{TopK}(\x_{j}),\; \frac{1}{N}\sum_{j=1}^{N}\|\z_j\|_0=K .
\end{align}
While BatchTopK imposes a mean sparsity per batch, for simplicity, we use the batch size to match the size of the entire dataset (WLOG). Smaller batch sizes can easily be incorporated by adding separate constraints, each over the entire batch (only leads to a change in constants---lagrange multipliers---in the analysis).

Following similar analysis as for TopK SAE (App. \ref{appsection:topkprior}), we can derive an equivalent prior over $\z$ for BatchTopK SAE:
\begin{align}
    P(\z_1, \dots, \z_N) \propto \prod_{i=1}^{N}  \exp \big( -\lambda \big| \|\z_i\|_0 - K \big| \big)
\end{align}
The sparse codes for different samples $\{\z_i\}_i$ are thus sampled i.i.d. from a distribution that only depends on the sparsity penalty.
While this prior looks very similar to the prior of TopK SAE, the difference is that in TopK, the fixed sparsity constraint is imposed per sample, leading to a different Lagrange multiplier $\lambda_i$ per sample to optimize over, while in BatchTopK, we have a common multiplier $\lambda$ over all examples in a batch (with multiple batches, we will have one multiplier per batch), which is then optimized over to ensure that average sparsity per batch constraint is met.

Similar to the analysis for the TopK SAE, we get the following prior over different latents per sample:
\begin{align}
 z_i^\mu \mid S_i \sim & \begin{cases}
        U(0, \kappa) \; &\mu\in S_i \\
        \delta_0 \; &\mu \notin S_i
    \end{cases} , \text{and} \\
    z_i^{\mu_1}, \dots, z_i^{\mu_{|S_i|}} \mid S_i &\overset{\text{i.i.d.}}{\sim} U(0,\kappa)\; \text{ for } \mu_\cdot \in S_i
\end{align}
The active indices $S_i$ are sampled uniformly conditioned on the number of active indices $n_i$:
\begin{align}
    S_i \mid n_i \sim U([M]^{n_i})
\end{align}
The number of active latents $n_i$ are themselves sampled i.i.d. (since $n_i = \Tilde{g}(\z_i)$ and $\{\z_i\}_i$ are i.i.d.) from a distribution whose mean is fixed:
\begin{align}
    n_1, \dots, n_N \overset{\text{i.i.d.}}{\sim} P, \; \text{s.t.} \; \mathbb{E}[n_{\cdot}]=K
\end{align}
This completes the derivation for the BatchTopK prior. \hfill $\square$

\subsubsection{JumpReLU SAE}

JumpReLU SAE (\cite{rajamanoharan2024jumpingaheadimprovingreconstruction}) is trained with the $L_0$ (pseudo-)norm regularizer. This leads to the following optimization problem:
\begin{align*}
    &\argmin_{\D, \z} \sum_{i=1}^{N} \frac{1}{N} \left(\|\x_{i} - \D \z_{i}\|_2^2 + \lambda \|\z_i\|_0 \right)\\
    &\text{s.t. } \forall k,  \z_{k} = f_{JumpReLU}(\x_{k}) 
\end{align*}

This objective is equivalent to the following prior over $\z$:
\begin{align}
    P(\z_1, \dots, \z_N) &\propto \prod_{i=1}^N\exp \Big( -\eta \|\z_i\|_0 \Big)
\end{align}

Noting the similarity with the TopK/ BatchTopK cases, we use the same analysis to derive the following conditions:
\begin{align}
    z_i^\mu \mid S_i \sim & \begin{cases}
        U(0, \kappa) \; &\mu\in S_i \\
        \delta_0 \; &\mu \notin S_i
    \end{cases} , \text{and} \\
    z_i^{\mu_1}, \dots, z_i^{\mu_{|S_i|}} \mid S_i &\overset{\text{i.i.d.}}{\sim} U(0,\kappa)\; \text{ for } \mu_\cdot \in S_i \\
    S_i \mid n_i &\sim U([M]^{n_i})
\end{align}
The number of active latents $n_i$ are again i.i.d., but there is no constraint on the mean of the distribution (unlike BatchTopK which constrained the mean of $n_i$ to equal $K$):
\begin{align}
    n_1, \dots, n_N \overset{\text{i.i.d.}}{\sim} P,
\end{align}
which completes the analysis for JumpReLU SAE. \hfill $\square$

\end{document}